\documentclass[10pt, a4paper, oneside, reqno]{amsart}

\makeatletter
\def\@settitle{\begin{center}%
		\baselineskip14\p@\relax
		\normalfont\LARGE\scshape\bfseries
		\@title
	\end{center}%
}
\makeatother

\makeatletter

\def\subsection{\@startsection{subsection}{2}%
	\z@{.5\linespacing\@plus.7\linespacing}{.5\linespacing}%
	{\normalfont\bfseries}}

\def\subsubsection{\@startsection{subsubsection}{3}%
	\z@{.5\linespacing\@plus.7\linespacing}{.5\linespacing}%
	{\normalfont\itshape}}

\usepackage[usenames, dvipsnames]{color}
\definecolor{darkblue}{rgb}{0.0, 0.0, 0.45}

\usepackage[colorlinks	= true,
			raiselinks	= true,
			linkcolor	= darkblue, 
			citecolor	= Mahogany,
			urlcolor	= ForestGreen,
			pdfauthor	= {},
			pdftitle	= {},
			pdfkeywords	= {},
			pdfsubject	= {},
			plainpages	= false]{hyperref}

\usepackage{dsfont,amssymb,amsmath,graphicx} 
\usepackage{amsfonts,dsfont,mathtools,mathrsfs,amsthm} 
\mathtoolsset{showonlyrefs=true}
\usepackage[amssymb, thickqspace]{SIunits}
\usepackage{fancyhdr,setspace}

\usepackage{nicefrac}       
\usepackage{microtype}      
\usepackage{tabto}

\usepackage{enumerate,enumitem}

\usepackage{bbm}

\allowdisplaybreaks
\date{\today}
\theoremstyle{theorem}
\newtheorem{Thm}{Theorem}[section]

\newtheorem{Lem}[Thm]{Lemma}
\newtheorem{Cor}[Thm]{Corollary}

\newtheorem{Rem}[Thm]{Remark}

\theoremstyle{remark}
\newtheorem{Ex}[Thm]{Example}

\newcommand{\Let}{\coloneqq}

\DeclareMathOperator{\vect}{vec}
\DeclareMathOperator{\Tr}{Tr}

\graphicspath{{images/}}

\usepackage[numbers,compress]{natbib}
\usepackage{hyperref}       
\usepackage{url}            
\usepackage{appendix}

\usepackage{algorithm}
\usepackage{algorithmic}

\usepackage{subfig}

\title[Learning robust controllers for LQR systems]{Learning Robust Controllers for Linear Quadratic Systems \\ with Multiplicative Noise via Policy Gradient}

\author[First]{Benjamin Gravell}
\author[Second]{Peyman Mohajerin Esfahani}
\author[Third]{Tyler Summers}

\thanks{The authors are with the Control, Optimization, and Networks lab, UT Dallas, The United States ({\tt\{benjamin.gravell, tyler.summers\}@utdallas.edu}), and  the Delft Center for Systems and Control, TU Delft, The Netherlands ({\tt P.MohajerinEsfahani@tudelft.nl}). This material is based on work supported by the Army Research Office under grant W911NF-17-1-0058.}

\begin{document}

\maketitle

\begin{abstract}
The linear quadratic regulator (LQR) problem has reemerged as an important theoretical benchmark for reinforcement learning-based control of complex dynamical systems with continuous state and action spaces. In contrast with nearly all recent work in this area, we consider multiplicative noise models, which are increasingly relevant because they explicitly incorporate inherent uncertainty and variation in the system dynamics and thereby improve robustness properties of the controller. Robustness is a critical and poorly understood issue in reinforcement learning; existing methods which do not account for uncertainty can converge to fragile policies or fail to converge at all. Additionally, intentional injection of multiplicative noise into learning algorithms can enhance robustness of policies, as observed in ad hoc work on domain randomization. Although policy gradient algorithms require optimization of a non-convex cost function, we show that the multiplicative noise LQR cost has a special property called \emph{gradient domination}, which is exploited to prove global convergence of policy gradient algorithms to the globally optimum control policy with polynomial dependence on problem parameters. Results are provided both in the model-known and model-unknown settings where samples of system trajectories are used to estimate policy gradients.
\end{abstract}

\section{Introduction}

Reinforcement learning-based control has recently achieved impressive successes in games \cite{silver2016mastering,silver2018general} and simulators \cite{mnih2015human}. But these successes are significantly more challenging to translate to complex physical systems with continuous state and action spaces, safety constraints, and non-negligible operation and failure costs that demand data efficiency. An intense and growing research effort is creating a large array of models, algorithms, and heuristics for approaching the myriad of challenges arising from these systems. To complement a dominant trend of more computationally focused work, the canonical linear quadratic regulator (LQR) problem in control theory has reemerged as an important theoretical benchmark for learning-based control \cite{recht2018tour,dean2017sample}. Despite its long history, there remain fundamental open questions for LQR with unknown models,
and a foundational understanding of learning in LQR problems can give insight into more challenging problems.

Almost all recent work on learning in LQR problems has utilized either deterministic or additive noise models \cite{recht2018tour,dean2017sample,Fazel2018,Bradtke1994,fiechter1997pac,abbasi2011regret,lewis2012reinforcement,tu2017least,abeille2018improved,umenberger2018learning,mania2019certainty,venkataraman2018recovering}, but here we consider \emph{multiplicative noise models}. 
In control theory, multiplicative noise models have been studied almost as long as their deterministic and additive noise counterparts \cite{Wonham1967,Damm2004}, although this area is somewhat less developed and far less widely known. 
We believe the study of learning in LQR problems with multiplicative noise is important for three reasons. 
First, this class of models is much richer than deterministic or additive noise while still allowing exact solutions when models are known, which makes it a  compelling additional benchmark \cite{willems1976feedback,athans1977uncertainty,bernstein1987robust}. 
Second, they explicitly incorporate model uncertainty and inherent stochasticity, thereby improving robustness properties of the controller. Robustness is a critical and poorly understood issue in reinforcement learning; existing methods which do not account for uncertainty can converge to fragile policies or fail to converge at all \cite{athans1977uncertainty,ku1977further,bertsekas1995dynamic}. Additionally, intentional injection of multiplicative noise into learning algorithms is known to enhance robustness of policies from ad hoc work on domain randomization \cite{tobin2017domain}. 
Third, in emerging difficult-to-model complex systems where learning-based control approaches are perhaps most promising, multiplicative noise models are increasingly relevant; examples include networked systems with noisy communication channels \cite{antsaklis2007special,hespanha2007survey}, modern power networks with large penetration of intermittent renewables \cite{carrasco2006power,milano2018foundations}, turbulent fluid flow \cite{lumley2007stochastic}, and neuronal brain networks \cite{breakspear2017dynamic}.

\subsection{Related literature}
Multiplicative noise LQR problems have been studied in control theory since the 1960s \cite{Wonham1967}. Since then a line of research parallel to deterministic and additive noise has developed, including basic stability and stabilizability results \cite{willems1976feedback}, semidefinite programming formulations \cite{el1995state, boyd1994linear,li2005estimation}, robustness properties \cite{Damm2004,bernstein1987robust,hinrichsen1998stochastic,bamieh2018input,gravell2020ifac}, and numerical algorithms \cite{benner2011lyapunov}. This line of research is less widely known perhaps because much of it studies continuous time systems, where the heavy machinery required to formalize stochastic differential equations is a barrier to entry for a broad audience. Multiplicative noise models are well-poised to offer data-driven model uncertainty representations and enhanced robustness in learning-based control algorithms and complex dynamical systems and processes. 
A related line of research which has seen recent activity is on learning optimal control of Markovian jump linear systems with unknown dynamics and noise distributions \cite{schuurmans2019,jansch2020convergence}, which under certain assumptions form a special case of the multiplicative noise system we analyze in this work.

In contrast to classical work on system identification and adaptive control, which has a strong focus on asymptotic results, more recent work has focused on non-asymptotic analysis using newly developed mathematical tools from statistics and machine learning. There remain fundamental open problems for learning in LQR problems, with several addressed only recently, including non-asymptotic sample complexity \cite{dean2017sample,tu2017least}, regret bounds \cite{abbasi2011regret,abeille2018improved,mania2019certainty}, and algorithmic convergence \cite{Fazel2018}.
Alternatives to reinforcement learning include other data-driven model-free optimal control schemes \cite{goncalves2019,baggio2019} and those leveraging the behavioral framework \cite{maupong2017,persis2019}. Subspace identification methods offer a model-based generalization to the output feedback setting \cite{juang1985}.

\subsection{Our contributions}
In \S\ref{sec:model} we establish the multiplicative noise LQR problem and motivate its study via a connection to robust stability.
We then give several fundamental results for policy gradient algorithms on linear quadratic problems with multiplicative noise. 
Our main contributions are as follows, which can be viewed as a generalization of the recent results of Fazel et al. \cite{Fazel2018} for deterministic LQR to multiplicative noise LQR:
\begin{itemize}
    \item In \S\ref{sec:graddom} we show that although the multiplicative noise LQR cost is generally non-convex, it has a special property called \emph{gradient domination}, which facilitates its optimization (Lemmas \ref{lemma:policy_grad_exp} and \ref{lemma:gradient_dominated}).
    \item In particular, in \S\ref{sec:conv} the gradient domination property is exploited to prove global convergence of three policy gradient algorithm variants (namely, exact gradient descent, ``natural’’ gradient descent, and Gauss-Newton/policy iteration) to the globally optimum control policy with a rate that depends polynomially on problem parameters (Theorems \ref{thm:gauss_newton_exact}, \ref{thm:nat_grad_exact_convergence}, and \ref{thm:grad_exact_convergence}).
    \item Furthermore, in \S\ref{sec:modelfree} we show that a model-free policy gradient algorithm, where the gradient is estimated from trajectory data (``rollouts'') rather than computed from model parameters, also converges globally (with high probability) with an appropriate exploration scheme and sufficiently many samples (polynomial in problem data) (Theorem \ref{thm:model-free}).
\end{itemize}

In comparison with the deterministic dynamics studied by \cite{Fazel2018}, we make the following novel technical contributions:
\begin{itemize}
    \item We quantify the increase in computational burden of policy gradient methods due to the presence of multiplicative noise, which is evident from the bounds developed in Appendices \ref{appendix:model_based_gd} and \ref{appendix:model_free_gd}. The noise acts to reduce the step size and thus convergence rate, and increases the required number of samples and rollout length in the model-free setting.

    \item A covariance dynamics operator $\mathcal{F}_K$ is established for multiplicative noise systems with a more complicated form than the deterministic case. This necessitated a more careful treatment and novel proof by induction and term matching argument in the proof of Lemma \ref{lemma:S_K_trace_bound}.
    
    \item Several restrictions on the algorithmic parameters (step size, number of rollouts, rollout length, exploration radius) which are necessary for convergence are established and treated.
    
    \item An important restriction on the support of the multiplicative noise distribution, which is naturally absent in \cite{Fazel2018}, is established in the model-free setting.
    
    \item A matrix Bernstein concentration inequality is stated explicitly and used to give explicit bounds on the algorithmic parameters in the model-free setting in terms of problem data.
    
    \item Discussion and numerical results on the use of backtracking line search is included.
    
    \item When the multiplicative variances $\alpha_i$, $\beta_j$ are all zero, the assertions of Theorems \ref{thm:gauss_newton_exact}, \ref{thm:nat_grad_exact_convergence}, \ref{thm:grad_exact_convergence}, \ref{thm:model-free} recover the same step sizes and convergence rates of the deterministic setting reported by \cite{Fazel2018}.
\end{itemize}
Thus, policy gradient algorithms for the multiplicative noise LQR problem enjoy the same global convergence properties as deterministic LQR, while significantly enhancing the resulting controller’s robustness to variations and inherent stochasticity in the system dynamics, as demonstrated by our numerical experiments in \S\ref{sec:numexp}.

To our best knowledge, the present paper is the first work to consider and obtain global convergence results using reinforcement learning algorithms for the multiplicative noise LQR problem. Our approach allows the explicit incorporation of a model uncertainty representation that significantly improves the robustness of the controller compared to deterministic and additive noise approaches.

\section{Optimal Control of Linear Systems with Multiplicative Noise and Quadratic Costs} \label{sec:model}
We consider the infinite-horizon linear quadratic regulator problem with multiplicative noise (LQRm)
\begin{alignat}{2}  \label{eq:LQRm}
    &\underset{{\pi \in \Pi}}{\text{minimize}} \quad && C(\pi) \Let \underset{x_0,\{\delta_{ti}\}, \{\gamma_{tj}\}}{\mathbb{E}} \sum_{t=0}^\infty (x_t^\intercal Q x_t + u_t^\intercal R u_t), \\
    &\text{subject to}                         \quad && x_{t+1} = (A  + \sum_{i=1}^p \delta_{ti} A_i) x_t +  (B + \sum_{j=1}^q \gamma_{tj} B_j) u_t, \nonumber
\end{alignat}
where $x_t \in \mathbb{R}^n$ is the system state, $u_t \in \mathbb{R}^m$ is the control input, the initial state $x_0$ is distributed according to $\mathcal{P}_0$ with covariance $\Sigma_0 \Let {\mathbb{E}}_{x_0} [ x_0 x_0^\intercal ]$, $\Sigma_0 \succ 0$, and $Q \succ 0$ and $R \succ 0$. The dynamics are described by a dynamics matrix $A \in \mathbb{R}^{n \times n}$ and input matrix $B \in \mathds{}{R}^{n \times m}$ and incorporate multiplicative noise terms modeled by the i.i.d. (across time), zero-mean, mutually independent scalar random variables $\delta_{ti}$ and $\gamma_{tj}$, which have variances $\alpha_i$ and $\beta_j$, respectively. The matrices $A_i \in \mathbb{R}^{n \times n}$ and $B_i \in \mathbb{R}^{n \times m}$ specify how each scalar noise term affects the system dynamics and input matrices. 
Alternatively, suppose $\bar A$ and $\bar B$ are zero-mean random matrices with a joint covariance structure\footnote{We assume $\bar A$ and $\bar B$ are independent for simplicity, but it is straightforward to include correlations between the entries of $\bar A$ and $\bar B$ into the model.} over their entries governed by the covariance matrices $\Sigma_A \Let \mathbb{E} [\mathbf{vec}(\bar A)\mathbf{vec}(\bar A)^\intercal ] \in \mathbb{R}^{n^2 \times n^2}$ and  $\Sigma_B \Let \mathbb{E} [ \mathbf{vec}(\bar B)\mathbf{vec}(\bar B)^\intercal ] \in \mathbb{R}^{nm \times nm}$.
Then it suffices to take the variances $\alpha_i$ and $\beta_{j}$ and matrices $A_i$ and $B_j$ as the eigenvalues and (reshaped) eigenvectors of $\Sigma_A$ and $\Sigma_B$, respectively, after a projection onto a set of orthogonal real-valued vectors \cite{gravell2020acc}.
The goal is to determine a closed-loop state feedback policy $\pi^*$ with $u_t = \pi^*(x_t)$ from a set $\Pi$ of admissible policies which solves the optimization in \eqref{eq:LQRm}.

We assume that the problem data $A$, $B$, $\alpha_i$, $A_i$, $\beta_j$, and $B_j$ permit existence and finiteness of the optimal value of the problem, in which case the system is called \emph{mean-square stabilizable} and requires \emph{mean-square stability} of the closed-loop system \cite{kozin1969survey,willems1976feedback}. The system in \eqref{eq:LQRm} is called \emph{mean-square stable} if $\lim_{t \to \infty}\mathbb{E}_{x_0, \delta, \gamma}[x_t x_t^\intercal] = 0$ for any given initial covariance $\Sigma_0$, where for brevity we notate expectation with respect to the noises $\mathbb{E}_{\{\delta_{ti}\},\{\gamma_{tj}\}}$ as $\mathbb{E}_{\delta, \gamma}$.  Mean-square stability is a form of robust stability, implying stability of the mean (i.e. $\lim_{t \to \infty}\mathbb{E} x_t = 0 \ \forall \ x_0$) as well as almost-sure stability (i.e. $\lim_{t \to \infty} x_t = 0$ almost surely) \cite{willems1976feedback}.
Mean-square stability requires stricter and more complicated conditions than stabilizability of the nominal system $(A,B)$ \cite{willems1976feedback}, which are discussed in the sequel. 
This essentially can limit the size of the multiplicative noise covariance \cite{athans1977uncertainty}, which can be viewed as a representation of uncertainty in the nominal system model or as inherent variation in the system dynamics.

\subsection{Control Design with Known Models: Value Iteration}
Dynamic programming can be used to show that the optimal policy $\pi^*$ is linear state feedback $u_t = \pi^*(x_t) = K^* x_t$, where $K^* \in \mathbb{R}^{m \times n}$ denotes the optimal gain matrix. When the control policy is linear state feedback ${u_t = \pi(x_t) = K x_t}$, with a very slight abuse of notation the cost becomes 
\begin{align}
    C(K) = \mathbb{E}_{x_0,\{\delta_{ti}\}, \{\gamma_{tj}\}} \sum_{t=0}^\infty x_t^\intercal (Q + K^\intercal R K ) x_t
\end{align}
Dynamic programming further shows that the resulting optimal cost is quadratic in the initial state, i.e. ${C(K^*) = \mathbb{E}_{x_0} x_0^\intercal P x_0 = \Tr(P \Sigma_0)}$, where $P \in \mathbb{R}^{n \times n}$ is a symmetric positive definite matrix \cite{bertsekas1995dynamic}. 
Note that the optimal controller does not need to directly observe the noise variables $\delta_{ti}$, $\gamma_{tj}$. 
When the model parameters are known, there are several ways to compute the optimal feedback gains and corresponding optimal cost. The optimal cost is given by the solution of the \emph{generalized} algebraic Riccati equation (GARE)
\begin{align} \label{eq:genriccati}
    P   & = Q + A^\intercal P A + \sum_{i=1}^p \alpha_i A_i^\intercal P A_i 
         - A^\intercal P B (R + B^\intercal P B + \sum_{j=1}^q \beta_j B_j^\intercal P B_j)^{-1} B^\intercal P A.
\end{align}
This is a special case of the GARE for optimal static output feedback given in \cite{bernstein1987robust} and can be solved via the value iteration
\begin{align*} 
    P_{k+1} & =  Q + A^\intercal P_k A + \sum_{i=1}^p \alpha_i A_i^\intercal P_k A_i  
            - A^\intercal P_k B (R + B^\intercal P_k B + \sum_{j=1}^q \beta_j B_j^\intercal P_k B_j)^{-1} B^\intercal P_k A,
\end{align*}
with $P_0 = Q$, or via semidefinite programming formulations \cite{boyd1994linear,el1995state,li2005estimation}, or via more exotic iterations based on the Smith method and Krylov subspaces \cite{freiling2003,ivanov2007}. The associated optimal gain matrix is 
\begin{align*}
    K^* = -  \bigg(R + B^\intercal P B + \sum_{j=1}^q \beta_j B_j^\intercal P B_j \bigg)^{-1} B^\intercal P A.
\end{align*}
It was verified in \cite{willems1976feedback} that existence of a positive definite solution to the GARE \eqref{eq:genriccati} is equivalent to mean-square stabilizability of the system, which depends on the problem data $A$, $B$, $\alpha_i$, $A_i$, $\beta_j$, and $B_j$; in particular, mean-square stability generally imposes upper bounds on the variances $\alpha_i$ and $\beta_j$ \cite{athans1977uncertainty}, but may be infinite depending on the structure of $A$, $B$, $A_i$, and $B_j$ \cite{willems1976feedback}. At a minimum, uniqueness and existence of a solution to the GARE \eqref{eq:genriccati} requires the standard conditions for uniqueness and existence of a solution to the standard ARE, namely of $(A,B)$ stabilizable and $(A, Q^{1/2})$ detectable.

Although (approximate) value iteration can be implemented using sample trajectory data, policy gradient methods have been shown to be more effective for approximately optimal control of high-dimensional stochastic nonlinear systems e.g. those arising in robotics \cite{peters2006}. This motivates our following analysis of the simpler case of stochastic linear systems wherein we show that policy gradient indeed facilitates a data-driven approach for learning optimal and robust policies. 

\subsection{Control Design with Known Models: Policy Gradient}
Consider a fixed linear state feedback policy ${u_t = Kx_t}$. 
Defining the stochastic system matrices
\begin{align}
    \widetilde{A} &= A  + \sum_{i=1}^p \delta_{ti} A_i , \\
    \widetilde{B} &= B + \sum_{j=1}^q \gamma_{tj} B_j , 
\end{align}
the deterministic nominal and stochastic closed-loop system matrices
\begin{align}
    {A}_K = A + B K, \qquad 
    \widetilde{A}_K = \widetilde{A} + \widetilde{B} K,
\end{align}
and the closed-loop state-cost matrix 
\begin{align}
    Q_K = Q+K^\intercal R K ,
\end{align}
the closed-loop dynamics become 
\begin{align*}
    x_{t+1} = \widetilde{A}_K x_t = \bigg((A  + \sum_{i=1}^p \delta_{ti} A_i) + (B + \sum_{j=1}^q \gamma_{tj} B_j)K\bigg) x_t.
\end{align*}
A gain $K$ is mean-square stabilizing if the closed-loop system is mean-square stable. Denote the set of mean-square stabilizing $K$ as $\mathcal{K}$.
If $K \in \mathcal{K}$, then the cost can be written as 
\begin{align}
    C(K) = \mathbb{E}_{x_0} x_0^\intercal P_K x_0 = \Tr(P_K \Sigma_0), 
\end{align}
where $P_K$ is the unique positive semidefinite solution to the \emph{generalized} Lyapunov equation
\begin{align} \label{eq:glyap1}
    P_{K}   & =  Q_K  + A_K^\intercal P_{K} A_K
            + \sum_{i=1}^p \alpha_i A_i^\intercal P_K A_i + \sum_{j=1}^q \beta_j K^\intercal B_j^\intercal P_K B_j K.
\end{align}
We define the state covariance matrices and the infinite-horizon aggregate state covariance matrix as
\begin{align}
    \Sigma_{t} \Let {\mathbb{E}}_{x_0,\delta, \gamma } [x_t x_t^\intercal], \qquad
    \Sigma_{K} \Let \sum_{t=0}^\infty \Sigma_t.     
\end{align}
If $K \in \mathcal{K}$ then $\Sigma_K$ also satisfies a \emph{dual} generalized Lyapunov equation
\begin{align} \label{eq:glyap2}
    \Sigma_{K}  & = \Sigma_0 + A_K \Sigma_{K} A_K^\intercal 
                + \sum_{i=1}^p \alpha_i A_i \Sigma_K A_i^\intercal + \sum_{j=1}^q \beta_j B_j K \Sigma_K K^\intercal B_j^\intercal.
\end{align}
Vectorization and Kronecker products can be used to convert \eqref{eq:glyap1} and \eqref{eq:glyap2} into systems of linear equations. Alternatively, iterative methods have been suggested for their solution \cite{freiling2003, ivanov2007}.
The state covariance dynamics are captured by two closed-loop finite-dimensional linear operators which operate on a symmetric matrix $X$:
\begin{align}
    \mathcal{T}_K(X)    & \Let \underset{\delta, \gamma}{\mathbb{E}} \sum_{t=0}^\infty \widetilde{A}_K^t X \widetilde{A}_K^{\intercal^t}, \\
    \mathcal{F}_K(X)    & \Let \underset{\delta, \gamma}{\mathbb{E}} \widetilde{A}_K X \widetilde{A}_K^\intercal = A_K X A_K^\intercal 
                        + \sum_{i=1}^p \alpha_i A_i X A_i^\intercal + \sum_{j=1}^q \beta_j B_j K X (B_j K)^\intercal .
\end{align}
Thus $\mathcal{F}_K$ (without an argument) is a linear operator whose matrix representation is 
\begin{align*}
    \mathcal{F}_K   & \Let A_K \otimes A_K
                    + \! \sum_{i=1}^p \alpha_i A_i \otimes A_i + \! \sum_{j=1}^q \beta_j (B_j K) \otimes (B_j K) .
\end{align*}
The $\Sigma_t$ evolve according to the dynamics
\begin{align}
    \Sigma_{t+1}        = \mathcal{F}_K (\Sigma_t)
    \quad \Leftrightarrow \quad
    \vect(\Sigma_{t+1}) = \mathcal{F}_K \vect(\Sigma_t)
\end{align}
We define the $t$-stage of $\mathcal{F}_K(X)$ as 
\begin{align}
    \mathcal{F}^t_K(X) &:= \mathcal{F}_K(\mathcal{F}^{t-1}_K(X)) \text{ with } \mathcal{F}^0_K(X) = X ,
\end{align}
which gives the natural characterization
\begin{align}
    \mathcal{T}_K(X) = \sum_{t=0}^\infty {\mathcal{F}^t_K(X)} ,
\end{align}
and in particular
\begin{align} \label{eq:TK_char}
    \Sigma_K = \mathcal{T}_K(\Sigma_0) = \sum_{t=0}^\infty {\mathcal{F}^t_K(\Sigma_0)} .
\end{align}
We then have the following lemma:
\begin{Lem}[Mean-square stability] \label{lemma:mss_char} 
    A gain $K$ is mean-square stabilizing if and only if the spectral radius $\rho(\mathcal{F}_{K}) < 1$.
\end{Lem}
\begin{proof}
    Mean-square stability implies ${\underset{t \to \infty}{\lim} \mathbb{E} [x_t x_t^\intercal] = 0}$, which for linear systems occurs only when $\Sigma_K$ is finite, which by \eqref{eq:TK_char} is equivalent to $\rho(\mathcal{F}_{K})<1$.
\end{proof}

Recalling the definition of $C(K)$ and \eqref{eq:glyap2}, along with the basic observation that $K \notin \mathcal{K}$ induces infinite cost, gives the following characterization of the cost:
\begin{align*}
    C(K) = \left\{
                \begin{array}{ll}
                  \Tr(Q_K \Sigma_K) = \Tr(P_K \Sigma_0) \quad &\text{if } K \in \mathcal{K} \\
                  \infty \quad &\text{otherwise.}
                \end{array}
              \right.
\end{align*}
The evident fact that $C(K)$ is expressed as a closed-form function, up to a Lyapunov equation, of $K$ leads to the idea of performing gradient descent on $C(K)$ (i.e., policy gradient) via the update $K \leftarrow K - \eta \nabla C(K)$ to find the optimal gain matrix. 
However, two properties of the LQR cost function $C(K)$ complicate a convergence analysis of gradient descent. First, $C(K)$ is extended valued since not all gain matrices provide closed-loop mean-square stability, so it does not have (global) Lipschitz gradients. Second, and even more concerning, $C(K)$ is generally non-convex in $K$ (even for deterministic LQR problems, as observed by Fazel et al. \cite{Fazel2018}), so it is unclear if and when gradient descent converges to the global optimum, or if it even converges at all. Fortunately, as in the deterministic case, we show that the multiplicative LQR cost possesses further key properties that enable proof of global convergence despite the lack of Lipschitz gradients and non-convexity.

\subsection{From Stochastic to Robust Stability}
Additional motivation for designing controllers which stabilize a stochastic system in mean-square is to ensure robustness of stability of a nominal deterministic system to model parameter perturbations.
Here we state a condition which guarantees robust deterministic stability for a perturbed deterministic system given mean-square stability of a stochastic single-state system with multiplicative noise where the noise variance and parameter perturbation size are related.

\begin{Ex}[Robust stability]
Suppose the stochastic closed-loop system
\begin{align} \label{eq:scalar_stochastic_sys}
    x_{t+1} = (a + \delta_{t}) x_t
\end{align}
where $a, x_t, \delta_t $ are scalars with $\mathbb{E} [ \delta_{t}^2 ] = \alpha $ is mean-square stable. 
Then, the perturbed deterministic system
\begin{align} \label{eq:perturbed_det_sys}
    x_{t+1} = (a + \phi) x_t
\end{align}
is stable for any constant perturbation $|\phi| \leq \sqrt{a^2 + \alpha} - | a |$.
\end{Ex}

\begin{proof}
By the bound on $\phi$ and triangle inequality we have 
$
    \rho(a + \phi) = |a + \phi | \leq |a| + | \phi | \leq \sqrt{a^2 + \alpha} .
$
From Lemma \ref{lemma:mss_char}, mean-square stability of \eqref{eq:scalar_stochastic_sys} implies 
$
    \sqrt{\rho(\mathcal{F})} = \sqrt{a^2 + \alpha} < 1 
$
and thus $\rho(a + \phi) < 1$, proving stability of \eqref{eq:perturbed_det_sys}.
\end{proof}

Although this is a simple example, it demonstrates that the robustness margin increases monotonically with the multiplicative noise variance. We also see that when $\alpha = 0$ the bound collapses so that no robustness is guaranteed, i.e., when $|a| \rightarrow 1$.
This result can be extended to multiple states, inputs, and noise directions, but the resulting conditions become considerably more complex \cite{bernstein1987robust, gravell2020ifac}. We now proceed with developing methods for optimal control.

\section{Gradient Domination and Other Properties of~the~Multiplicative~Noise~LQR~Cost} \label{sec:graddom}

In this section, we demonstrate that the multiplicative noise LQR cost function is \emph{gradient dominated}, which facilitates optimization by gradient descent. Gradient dominated functions have been studied for many years in the optimization literature \cite{Polyak1963} and have recently been discovered in deterministic LQR problems by \cite{Fazel2018}. 

\subsection{Multiplicative Noise LQR Cost is Gradient Dominated} \label{subsec:graddom}

First, we give the expression for the policy gradient of the multiplicative noise LQR cost.\footnote{We include a factor of 2 on the gradient expression that was erroneously dropped in \cite{Fazel2018}. This affects the step size restrictions by a corresponding factor of 2.}
Define
\begin{align*}
    R_K & \Let R + B^\intercal P_{K} B + \sum_{j=1}^q \beta_j B_j^\intercal P_K B_j \\
    E_K & \Let R_K K + B^\intercal P_K A .
\end{align*}

\begin{Lem}[Policy Gradient Expression] \label{lemma:policy_grad_exp}
The policy gradient is given by
\begin{align}
    \nabla_K C(K) = 2 E_K \Sigma_{K} &= 2(R_K K + B^\intercal P_K A) \Sigma_{K} 
    = 2\bigg[ \Big( R + B^\intercal P_{K} B + \sum_{j=1}^q \beta_j B_j^\intercal P_K B_j \Big) K + B^\intercal P_{K} A\bigg] \Sigma_{K} . \nonumber
\end{align}
\end{Lem}

\begin{proof}
Substituting the RHS of the generalized Lyapunov equation \eqref{eq:glyap1} into the cost $C(K) = \Tr(P_K \Sigma_0)$ yields
\begin{align*}
    C(K) 
    &= \Tr(Q_K \Sigma_0)  + \Tr(A_K^\intercal P_{K} A_K \Sigma_0) 
    + \Tr \big(\sum_{i=1}^p \alpha_i A_i^\intercal P_K A_i \Sigma_0 \big)  + \Tr \big( \sum_{j=1}^q \beta_j K^\intercal B_j^\intercal P_K B_j K \Sigma_0 \big)
\end{align*}
Taking the gradient with respect to $K$ and using the product rule and rules for matrix derivatives we obtain
\begin{align}
    \nabla_K C(K) 
    &= \nabla_K \Tr ( P_{K} \Sigma_0 ) \\
    &= \nabla_K \bigg[ \Tr((Q + K^TRK) \Sigma_0) + \Tr(A+BK)^T P_{K} (A+BK) \Sigma_0) + \Tr(\sum_{i=1}^p \alpha_i A_i^T P_K A_i \Sigma_0) \\
    & \qquad +  \Tr(\sum_{j=1}^q \beta_j K^T B_j^T P_K B_j K \Sigma_0) \bigg] \\
    &= \nabla_{\widetilde{K}} \big[ \Tr(Q_{\widetilde{K}} \Sigma_0) + \Tr(A_{\widetilde{K}}^\intercal P_{K} A_{\widetilde{K}} \Sigma_0) 
    + \Tr \big(\sum_{i=1}^p \alpha_i A_i^\intercal P_K A_i \Sigma_0 \big) + \Tr \big( \sum_{j=1}^q \beta_j \widetilde{K}^\intercal B_j^\intercal P_K B_j \widetilde{K} \Sigma_0 \big) \big] \\
    & \quad + \nabla_{\bar{K}}  \big[ \Tr(A_K^\intercal P_{\bar{K}} A_K \Sigma_0)  
    + \Tr \big(\sum_{i=1}^p \alpha_i A_i^\intercal P_{\bar{K}} A_i \Sigma_0 \big)  + \Tr \big( \sum_{j=1}^q \beta_j K^\intercal B_j^\intercal P_{\bar{K}} B_j K \Sigma_0 \big) \big] \\
    &= 2 \left[ \left( R + B^T P_K B +  \sum_{j=1}^q \beta_j B_j^T P_K B_j \right) K  +  B^T P_{K} A \right] \Sigma_0 \\
    &\quad + \nabla_{\bar K}  \Tr\left[ \left((A+BK)^T P_{\bar K} (A+BK) +  \sum_{i=1}^p \alpha_i A_i^T P_{\bar K} A_i +  \sum_{j=1}^q \beta_j K^T B_j^T P_{\bar K} B_j K \right) \Sigma_0\right] \\
    &= 2 \left[ \left( R + B^T P_K B +  \sum_{j=1}^q \beta_j B_j^T P_K B_j \right) K  +  B^T P_{K} A \right] \Sigma_0 \\
    &\quad + \nabla_{\bar K}  \Tr\left( P_{\bar K} \left[ (A+BK) \Sigma_0 (A+BK)^T +  \sum_{i=1}^p\alpha_i   A_i\Sigma_0 A_i^T +  \sum_{j=1}^q \beta_j  B_j K \Sigma_0 K^T B_j^T \right] \right) \\
    &= 2 \left[ \left( R + B^T P_K B +  \sum_{j=1}^q \beta_j B_j^T P_K B_j \right) K  +  B^T P_{K} A \right] \Sigma_0 + \nabla_{\bar K}  \Tr ( P_{\bar K} \Sigma_1 )
\end{align}
where the tilde on $\widetilde{K}$ and overbar on ${\overline K}$ are used to denote the terms being differentiated.
Applying this gradient formula recursively to the last term in the last line (namely $ \nabla_{\bar K}  \Tr ( P_{\bar{K}} \Sigma_1 )$), and recalling the definition of $\Sigma_K$ we obtain
\begin{equation}
\nabla_K C(K) = 2\left[ \left(R + B^T P_{K} B + \sum_{j=1}^q \beta_j B_j^T P_K B_j \right) K + B^T P_{K} A \right] \Sigma_{K}
\end{equation}
which completes the proof.
\end{proof}
For brevity the gradient is implied to be with respect to the gains $K$ in the rest of this work, i.e., $\nabla_K$ denoted by $\nabla$.
Now we must develop some auxiliary results before demonstrating gradient domination. Throughout $\| Z \|$ and $\|Z\|_F$ are the spectral and Frobenius norms respectively of a matrix $Z$, and $\underline{\sigma}(Z)$ and $\overline{\sigma}(Z)$ are the minimum and maximum singular values of a matrix $Z$.
The value function of $V_K(x)$ for $x=x_0$ is defined as
\begin{align*} 
    V_K(x) \Let \mathbb{E}_{\delta,\gamma} \sum_{t=0}^\infty x_t^\intercal Q_K x_t.
\end{align*}
which relates to the cost as $C(K) = \mathbb{E}_{x_0} V_K(x_0)$.
The advantage function is defined as
\begin{align}
    \mathcal{A}_K(x,u) & \Let x^\intercal Q x + u^\intercal R u + \underset{\delta,\gamma}{\mathbb{E}} V_K(\widetilde{A}x+\widetilde{B}u) - V_K(x), 
\end{align}
where the expectation is with respect to $\widetilde{A}$ and $\widetilde{B}$ \emph{inside} the parentheses of $V_{K}(\widetilde{A} x+\widetilde{B} u)$.
The advantage function can be thought of as the difference in cost (``advantage'') when starting in state $x$ of taking an action $u$ for one step instead of the action generated by policy $K$.
We also define the state, input, and cost sequences
\begin{align}
    \{x_t\}_{K,x} & \Let \{x, \widetilde{A}_K x,\widetilde{A}_K^2 x, ... , \widetilde{A}_K^t x, ... \} \\
    \{u_t\}_{K,x} & \Let K\{x_t\}_{K,x} \\
    \{c_t\}_{K,x} & \Let \{x_t\}_{K,x}^\intercal Q_K \{x_t\}_{K,x} .
\end{align}
Throughout the proofs we will consider pairs of gains $K$ and $K{^\prime}$ and their difference $\Delta \Let K{^\prime} - K$.
\begin{Lem}[Value difference] \label{lemma:value_difference} 
    Suppose $K$ and $K{^\prime}$ generate the (stochastic) state, action, and cost sequences 
    $\{x_t\}_{K,x}  , \{u_t\}_{K,x}  ,  \{c_t\}_{K,x}$
    and
    $ \{x_t\}_{K{^\prime},x}  , \{u_t\}_{K{^\prime},x}  ,  \{c_t\}_{K{^\prime},x} $.
    Then the value difference and advantage satisfy
    \begin{align}
        V_{K{^\prime}}(x)-V_{K}(x) &= \underset{\delta,\gamma}{\mathbb{E}} \sum_{t=0}^\infty \mathcal{A}_{K}\big( \{x_t\}_{K{^\prime},x}, \{u_t\}_{K{^\prime},x} \big) \\
        \mathcal{A}_{K}(x,K{^\prime} x) &= 2x^\intercal\Delta^\intercal E_{K} x + x^\intercal \Delta^\intercal R_{K} \Delta x .
    \end{align}
\end{Lem}

\begin{proof}
The proof follows the ``cost-difference'' lemma in \cite{Fazel2018} exactly substituting versions of value and cost functions, etc. which take expectation over the multiplicative noise.
By definition we have
\begin{align}
    V_K(x) = \underset{\delta,\gamma}{\mathbb{E}} \sum_{t=0}^{\infty} \{c_t\}_{K,x}
\end{align}
so we can write the value difference as
\begin{align}
     V_{K{^\prime}}(x)-V_{K}(x) & = \underset{\delta,\gamma}{\mathbb{E}} \sum_{t=0}^{\infty} \bigg[\{c_t\}_{K{^\prime},x} \bigg] - V_{K}(x) \\
    & =  \underset{\delta,\gamma}{\mathbb{E}} \sum_{t=0}^{\infty} \bigg[\{c_t\}_{K{^\prime},x}  - V_{K}(\{x_t\}_{K{^\prime},x})\bigg] 
    +  \underset{\delta,\gamma}{\mathbb{E}} \sum_{t=0}^{\infty} \bigg[V_{K}(\{x_t\}_{K{^\prime},x})\bigg]- V_{K}(x)  .
\end{align}
We expand the following value function difference as
\begin{align}
    & \underset{\delta,\gamma}{\mathbb{E}} \sum_{t=0}^{\infty}\bigg[V_{K}(\{x_t\}_{K{^\prime},x})\bigg] - V_{K}(x) \\
    & = \underset{\delta,\gamma}{\mathbb{E}} \sum_{t=0}^{\infty}\bigg[V_{K}(\{x_{t+1}\}_{K{^\prime},x})\bigg] + \underset{\delta,\gamma}{\mathbb{E}} V_{K}(\{x_{0}\}_{K{^\prime},x}) - V_{K}(x) \\
    & = \underset{\delta,\gamma}{\mathbb{E}} \sum_{t=0}^{\infty}V_{K}(\{x_{t+1}\}_{K{^\prime},x})
\end{align}
where the last equality is valid by noting that the first term in sequence $\{x_t\}_{K{^\prime},x}$ is $x$.
Continuing the value difference expression we have
\begin{align}
    V_{K{^\prime}}(x)-V_{K}(x) 
    &= \underset{\delta,\gamma}{\mathbb{E}} \sum_{t=0}^{\infty} \bigg[\{c_t\}_{K{^\prime},x}  - V_{K}(\{x_t\}_{K{^\prime},x})\bigg] + \underset{\delta,\gamma}{\mathbb{E}} \sum_{t=0}^{\infty}V_{K}(\{x_{t+1}\}_{K{^\prime},x}) \\
    &= \underset{\delta,\gamma}{\mathbb{E}} \sum_{t=0}^{\infty} \bigg[\{c_t\}_{K{^\prime},x}  + V_{K}(\{x_{t+1}\}_{K{^\prime},x})- V_{K}(\{x_t\}_{K{^\prime},x}) \bigg] \\
    &= \underset{\delta,\gamma}{\mathbb{E}} \sum_{t=0}^{\infty} \bigg[\{x_t\}_{K{^\prime},x}^\intercal Q_{K{^\prime}} \{x_t\}_{K{^\prime},x}  + V_{K}(\{x_{t+1}\}_{K{^\prime},x})- V_{K}(\{x_t\}_{K{^\prime},x}) \bigg] \\
    &= \underset{\delta,\gamma}{\mathbb{E}} \sum_{t=0}^{\infty} \bigg[\{x_t\}_{K{^\prime},x}^\intercal Q \{x_t\}_{K{^\prime},x} + \{u_t\}_{K{^\prime},x}^\intercal R \{u_t\}_{K{^\prime},x}  + V_{K}(\{x_{t+1}\}_{K{^\prime},x})  - V_{K}(\{x_t\}_{K{^\prime},x}) \bigg] \\
    &= \underset{\delta,\gamma}{\mathbb{E}} \sum_{t=0}^{\infty} \bigg[\mathcal{Q}_{K}(\{x_t\}_{K{^\prime},x},\{u_t\}_{K{^\prime},x}) - V_{K}(\{x_t\}_{K{^\prime},x}) \bigg] \\
    &= \underset{\delta,\gamma}{\mathbb{E}} \sum_{t=0}^\infty \mathcal{A}_{K}\big(\{x_t\}_{K{^\prime},x} \ , \ \{u_t\}_{K{^\prime},x}\big) ,
\end{align}
where the fifth equality holds since $\{x_{t+1}\}_{K{^\prime},x} = A\{x_t\}_{K{^\prime},x} + B\{u_t\}_{K{^\prime},x}$.

For the second part of the proof regarding the advantage expression, we expand and substitute in definitions:
\begin{align}
    \mathcal{A}_{K}(x,K{^\prime} x) &= \mathcal{Q}_{K}(x,K{^\prime}x) - V_{K}(x) \\
    &= x^\intercal Q x + x^\intercal K{^\prime}^\intercal R K{^\prime} x +\underset{\delta,\gamma}{\mathbb{E}} V_{K}(A_{K{^\prime}} x) - V_{K}(x) .
\end{align}
Now note that
\begin{align}
    \underset{\delta,\gamma}{\mathbb{E}} V_{K}(A_{K{^\prime}} x) &= x^\intercal \underset{\delta,\gamma}{\mathbb{E}} \bigg[ A_{K{^\prime}}^\intercal P_{K} A_{K{^\prime}} \bigg] x \\
    &= x^\intercal  \underset{\delta,\gamma}{\mathbb{E}} \left[ \left(A  + \sum_{i=1}^p \delta_{ti} A_i +  (B + \sum_{j=1}^q \gamma_{tj} B_j)K{^\prime} \right)^\intercal P_{K}  \left(A  + \sum_{i=1}^p \delta_{ti} A_i +  (B + \sum_{j=1}^q \gamma_{tj} B_j)K{^\prime} \right) \right] x \\
    &= x^\intercal \left[(A+BK{^\prime})^\intercal P_{K} (A+BK{^\prime})  + \mathbb{E}_{\delta_{ti}} \left(\sum_{i=1}^p \delta_{ti} A_i \right)^\intercal P_{K} \left(\sum_{i=1}^p \delta_{ti} A_i\right) \right. \\
    & \qquad \left. + \mathbb{E}_{\gamma_{tj}} \left(\sum_{j=1}^q \gamma_{tj} B_j K{^\prime}\right)^\intercal P_{K} \left(\sum_{j=1}^q \gamma_{tj} B_j K{^\prime}\right) \right] x \\
    &= x^\intercal \left[A_{K^\prime}^\intercal P_{K} A_{K^\prime} + \sum_{i=1}^p \alpha_i A_i^\intercal P_{K} A_i + \sum_{j=1}^q \beta_j K{^\prime}^\intercal B_j^\intercal P_{K} B_j K{^\prime} \right] x,
\end{align}
where the third equality follows from all of the $\delta_{ti}$ and $\gamma_{tj}$ being zero-mean and mutually independent.
Substituting and continuing,
\begin{align*}
    &\mathcal{A}_{K}(x,K{^\prime} x)  \nonumber \\
    &= x^\intercal Q x + x^\intercal K{^\prime}^\intercal R K{^\prime} x + x^\intercal \bigg[(A+BK{^\prime})^\intercal P_{K} (A+BK{^\prime}) 
    + \sum_{i=1}^p \alpha_i A_i^\intercal P_{K} A_i + \sum_{j=1}^q \beta_j K{^\prime}^\intercal B_j^\intercal P_{K} B_j K{^\prime} \bigg] x - V_{K}(x) \\
    &= x^\intercal \bigg[ Q + K{^\prime}^\intercal R K{^\prime} + (A+BK{^\prime})^\intercal P_{K} (A+BK{^\prime}) 
    + \sum_{i=1}^p \alpha_i A_i^\intercal P_{K} A_i + \sum_{j=1}^q \beta_j K{^\prime}^\intercal B_j^\intercal P_{K} B_j K{^\prime} \bigg]x - V_{K}(x) \\
    &= x^\intercal \bigg[Q + (\Delta+K)^\intercal R (\Delta+K) + (A+BK+B\Delta)^\intercal P_{K} (A+BK+B\Delta) \nonumber \\
    & \hspace{50pt} + \sum_{i=1}^p \alpha_i A_i^\intercal P_{K} A_i + \sum_{j=1}^q \beta_j (B_j K + B_j\Delta)^\intercal P_{K} (B_j K + B_j\Delta) \bigg] x - V_{K}(x) \\
    &= x^\intercal \bigg[\Delta^\intercal R \Delta + 2\Delta^\intercal R K + (B\Delta)^\intercal P_{K} (B\Delta) + 2(B\Delta)^\intercal P_{K} (A+B{K}) \nonumber \\
    & \qquad + \sum_{j=1}^q \beta_j (B_j\Delta)^\intercal P_{K} (B_j\Delta) + 2\sum_{j=1}^q \beta_j (B_j\Delta)^\intercal P_{K} B_j K \bigg] x \nonumber \\
    & \qquad + x^\intercal \bigg[ Q + K^\intercal RK  + (A+BK)^\intercal P_{K} (A+BK) 
    + \sum_{i=1}^p \alpha_i A_i^\intercal P_{K} A_i +  \sum_{j=1}^q \beta_j K^\intercal B_j^\intercal P_{K} B_j K \bigg] x - V_{K}(x) .
\end{align*}
We also have the following expression from the recursive relationship for $P_{K}$
\begin{align}
    V_{K}(x) &= x^\intercal P_{K} x = x^\intercal \bigg[ Q + K^\intercal RK  + (A+BK)^\intercal P_{K} (A+BK) 
    + \sum_{i=1}^p \alpha_i A_i^\intercal P_{K} A_i +  \sum_{j=1}^q \beta_j K^\intercal B_j^\intercal P_{K} B_j K \bigg] x .
\end{align}
Substituting, we cancel the $V_{K}(x)$ term which leads to the result after rearrangement:
\begin{align}
    \mathcal{A}_{K}(x,K{^\prime} x)  &= x^\intercal\bigg[\Delta^\intercal R \Delta + 2\Delta^\intercal R K + (B\Delta)^\intercal P_{K} (B\Delta) + 2(B\Delta)^\intercal P_{K} (A+B{K}) \nonumber \\
                                & \qquad + \sum_{j=1}^q \beta_j (B_j\Delta)^\intercal P_{K} (B_j\Delta) + 2\sum_{j=1}^q \beta_j (B_j\Delta)^\intercal P_{K} B_j K \bigg] x \\
                                &= 2x^\intercal \Delta^\intercal \bigg[R K + B^\intercal P_{K} (A+B{K}) + \sum_{j=1}^q \beta_j B_j^\intercal P_{K} B_j K \bigg] x \nonumber \\
                                & \qquad + x^\intercal \Delta^\intercal \bigg( R+B^\intercal P_{K} B+\sum_{j=1}^q \beta_j B_j^\intercal P_{K} B_j \bigg) \Delta x \\
                                &= 2x^\intercal \Delta^\intercal \bigg[\bigg(R + B^\intercal P_{K}B+\sum_{j=1}^q \beta_j B_j^\intercal P_{K} B_j \bigg) K + B^\intercal P_{K}A \bigg] x \nonumber \\
                                & \qquad + x^\intercal \Delta^\intercal \bigg( R+B^\intercal P_{K} B+\sum_{j=1}^q \beta_j B_j^\intercal P_{K} B_j \bigg) \Delta x \\
                                &= 2x^\intercal\Delta^\intercal E_{K} x + x^\intercal \Delta^\intercal R_{K} \Delta x ,
\end{align}
which completes the proof.
\end{proof}

Next, we see that the multiplicative noise LQR cost is gradient dominated. 

\begin{Lem}[Gradient domination] \label{lemma:gradient_dominated} 
    The LQR-with-multiplicative-noise cost $C(K)$ satisfies the gradient domination condition
    \begin{align}
        C(K) - C(K^*) &\leq \frac{\|\Sigma_{K^*}\|}{4\underline{\sigma}(R) {\underline{\sigma} (\Sigma_0)}^2 } \|\nabla C(K)\|_F^2 .
    \end{align}     
\end{Lem}

\begin{proof}
We start with the advantage expression
\begin{align}
    \mathcal{A}_{K}(x,K{^\prime} x) &= 2x^\intercal \Delta^\intercal E_{K} x + x^\intercal \Delta^\intercal R_{K} \Delta x \\
                                    &= 2 \Tr[xx^\intercal\Delta^\intercal E_{K}] +\Tr [xx^\intercal \Delta^\intercal R_{K} \Delta] .
\end{align} 
Next we rearrange and complete the square:
\begin{align*}
    \mathcal{A}_{K}(x,K{^\prime} x) & = \Tr \big[ xx^\intercal \left(\Delta^\intercal R_{K} \Delta + 2\Delta^\intercal E_{K} \right) \big] \\
                                    & = \Tr \big[ xx^\intercal (\Delta+R_{K}^{-1} E_{K})^\intercal R_{K} (\Delta+R_{K}^{-1} E_{K}) \big] 
                                    - \Tr \big[ xx^\intercal E_{K}^\intercal R_{K}^{-1} E_{K} \big] .
\end{align*}
Since $R_{K} \succ 0 $, we have
\begin{align}
    \mathcal{A}_{K}(x,K{^\prime} x) \geq - \Tr\big[xx^\intercal E_{K}^\intercal R_{K}^{-1} E_{K} \big] \label{eq:grad_dom_ineq1}
\end{align}
with equality only when $\Delta = - R_{K}^{-1} E_{K}$.

Let the state and control sequences associated with the optimal gain $K^*$ be $\{x_t\}_{K^*,x}$ and $\{u_t\}_{K^*,x}$ respectively. We now obtain an upper bound for the cost difference by writing the cost difference in terms of the value function as
\begin{align}
    C(K) - C(K^*) = \underset{x_0}{\mathbb{E}} \big[ V(K,x_0) \big] - \underset{x_0}{\mathbb{E}} \big[ V(K^*,x_0) \big]
                  = \underset{x_0}{\mathbb{E}} \big[ V(K,x_0)-V(K^*,x_0) \big] .
\end{align}
Using the first part of the value-difference Lemma \ref{lemma:value_difference} and negating we obtain
\begin{align}
    C(K) - C(K^*) &= -\underset{x_0}{\mathbb{E}} \bigg[ \sum_{t=0}^\infty \mathcal{A}_{K}\big(\{x_t\}_{K^*,x}  ,  \{u_t\}_{K^*,x}\big)\bigg] \\
    & \leq \underset{x_0}{\mathbb{E}} \left[ \sum_{t=0}^\infty \Tr\bigg[\{x_t\}_{K^*,x}\{x_t\}_{K^*,x}^\intercal E_{K}^\intercal R_K^{-1} E_{K} \bigg]\right] = \Tr\bigg[ \Sigma_{K^*} E_{K}^\intercal R_K^{-1} E_{K} \bigg] .
\end{align}
where the second step used the advantage inequality in \eqref{eq:grad_dom_ineq1}.
Now using $|\Tr(Y Z)| \leq \|Y\||\Tr(Z)|$ we obtain
\begin{align}
    C(K) - C(K^*) &\leq \|\Sigma_{K^*}\| \Tr\bigg[E_{K}^\intercal R_K^{-1} E_{K} \bigg]  \label{eq:grad_dom_GN} \\
                  &\leq \|\Sigma_{K^*}\| \|R_K^{-1}\| \Tr\bigg[  E_{K}^\intercal E_{K} \bigg] \label{eq:grad_dom_grad_desc} ,
\end{align}
where the first and second inequalities will be used later in the Gauss-Newton and gradient descent convergence proofs respectively.
Combining $\|R_K\| \geq \|R\| = \overline{\sigma}(R) \geq \underline{\sigma}(R) $ with $\|Z^{-1}\| \geq \|Z\|^{-1}$ we obtain
\begin{align} \label{eq:grad_dom_ng}
    C(K) - C(K^*) &\leq \frac{\|\Sigma_{K^*}\|}{\underline{\sigma}(R)} \Tr\big[  E_{K}^\intercal E_{K} \big]
\end{align}
which will be used later in the natural policy gradient descent convergence proof.
Now we rearrange and substitute in the policy gradient expression $\frac{1}{2} \nabla C(K) (\Sigma_K)^{-1} = E_K$
\begin{align}
    C(K) - C(K^*) 
    &\leq \frac{\|\Sigma_{K^*}\|}{4\underline{\sigma}(R)} \Tr\big[  ( \nabla C(K) \Sigma_K^{-1})^\intercal (\nabla C(K) \Sigma_K^{-1}) \big] \\
    & \leq \frac{\|\Sigma_{K^*}\|}{4\underline{\sigma}(R)} \|(\Sigma_K^{-1})^\intercal\Sigma_K^{-1}\| \Tr\big[ \nabla C(K)^\intercal \nabla C(K) \big] \\
    & \leq \frac{\|\Sigma_{K^*}\|}{4\underline{\sigma}(R) \underline{\sigma}(\Sigma_K)^2 } \Tr\big[ \nabla C(K)^\intercal \nabla C(K) \big].
\end{align}
where the last step used the definition and submultiplicativity of spectral norm.
Using 
\begin{align}
    \Sigma_K = \underset{x_0}{\mathbb{E}} \left[ \sum_{t=0}^{\infty} x_t x_t^\intercal \right] \succeq \underset{x_0}{\mathbb{E}}[x_0x_0^\intercal] = \Sigma_0 
    \qquad \Rightarrow \qquad
    \underline{\sigma}(\Sigma_K) < {\underline{\sigma} (\Sigma_0)}
\end{align}
 completes the proof.
\end{proof}

The gradient domination property gives the following stationary point characterization.
\begin{Cor}
    If $\nabla C(K)=0$ then either $K=K^*$ or $\text{rank } \! (\Sigma_K)<n$.
\end{Cor}
In other words, so long as $\Sigma_K$ is full rank, stationarity is both necessary and sufficient for global optimality, as for convex functions. Note that it is not sufficient to just have multiplicative noise in the dynamics with a deterministic initial state $x_0$ to ensure that $\Sigma_K$ is full rank. To see this, observe that if $x_0=0$ and $\Sigma_0 = 0$ then $\Sigma_K=0$, which is clearly rank deficient. By contrast, additive noise is sufficient to ensure that $\Sigma_K$ is full rank with a deterministic initial state $x_0$, although we will not consider this setting. Using a random initial state with $\Sigma_0 \succ 0$ ensures rank$(\Sigma_K)=n$ and thus $\nabla C(K)=0$ implies $K=K^*$.

Although the gradient of the multiplicative noise LQR cost is not globally Lipschitz continuous, it is locally Lipschitz continuous over any subset of its domain (i.e., over any set of mean-square stabilizing gain matrices). The gradient domination is then sufficient to show that policy gradient descent will converge to the optimal gains at a linear rate (a short proof of this fact for globally Lipschitz functions is given in \cite{Karimi2016}). We prove this convergence of policy gradient to the optimum feedback gain by bounding the local Lipschitz constant in terms of the problem data, which bounds the maximum step size and the convergence rate.

\subsection{Additional Setup Lemmas}
Following \cite{Fazel2018} we refer to Lipschitz continuity of the gradient as ($\mathcal{C}^1$-)smoothness, and so this section deals with showing that the LQR cost satisfies an expression that is almost of the exact form of a Lipschitz continuous gradient.
\begin{Lem}[Almost-smoothness] \label{lemma:almost_smooth} 
    The LQR-with-multiplicative-noise cost $C(K)$ satisfies the almost-smoothness expression
    \begin{align}
        C(K{^\prime}) - C(K) &= 2\Tr\big[ \Sigma_{K{^\prime}} \Delta^\intercal E_{K} \big] + \Tr\big[ \Sigma_{K{^\prime}} \Delta^\intercal R_{K}\Delta \big] .
    \end{align}
\end{Lem}

\begin{proof}
As in the gradient domination proof, we express the cost difference in terms of the advantage by taking expectation over the initial states to obtain 
\begin{align}
    C(K{^\prime}) - C(K) &= \underset{x_0}{\mathbb{E}} \bigg[ \sum_{t=0}^\infty \mathcal{A}_{K}\big(\{x_t\}_{K{^\prime},x} \ , \ \{u_t\}_{K{^\prime},x}\big)\bigg] .
\end{align}
From the value difference lemma for the advantage we have
\begin{align}
    \mathcal{A}_{K}(x,K{^\prime} x) &= 2x^\intercal\Delta^\intercal E_{K} x + x^\intercal \Delta^\intercal R_{K} \Delta x .
\end{align}
Noting that $\{u_t\}_{K^\prime,x}=K^\prime x$ we substitute to obtain
\begin{align}
     C(K{^\prime}) - C(K) & = \underset{x_0}{\mathbb{E}} \bigg[ \sum_{t=0}^\infty 2\{x_t\}_{K{^\prime},x}^\intercal\Delta^\intercal E_{K} \{x_t\}_{K{^\prime},x} 
    + \{x_t\}_{K^\prime,x}^\intercal \Delta^\intercal R_K \Delta\{x_t\}_{K{^\prime},x} \big)\bigg] .
\end{align}
Using the definition of $\Sigma_{K^\prime}$ completes the proof.
\end{proof}
\begin{Rem}
For small deviations $K^\prime - K$ the equation in the almost-smoothness lemma exactly describes a Lipschitz continuous gradient. The naming should not be taken to imply that the LQRm cost is not smooth, but rather that the equation as stated does not immediately yield a Lipschitz constant; indeed the Lipschitz constant is what much of the later proofs go towards bounding (implicitly) i.e. by bounding higher-order terms which must be accounted for when $K^\prime \neq K$.

To be specific, a Lipschitz continuous gradient to $C(K)$ implies there exists a Lipschitz constant $L$ such that
\begin{align}
    C(K^\prime) - C(K) \leq \Tr(\nabla C(K)^\intercal \Delta) + \frac{L}{2} \Tr (\Delta^\intercal \Delta)
\end{align}
for all $K^\prime$, $K$. This is the quadratic upper bound which is used e.g. in Thm. 1 of \cite{Karimi2016} to prove convergence of gradient descent on a gradient dominated objective function.
The ``almost''-smoothness condition is that
\begin{align}
    C(K{^\prime}) - C(K) &= 2\Tr\big[ \Sigma_{K{^\prime}} \Delta^\intercal E_{K} \big] + \Tr\big[ \Sigma_{K{^\prime}} \Delta^\intercal R_{K}\Delta \big] .
\end{align}
For $K^\prime \approx K$ we have $\Sigma_{K{^\prime}} \approx \Sigma_K$ so using $\nabla C(K) = 2 E_K \Sigma_K$ we have
\begin{align}
    C(K{^\prime}) - C(K) 
    &\approx 2\Tr\big[ \Sigma_{K} \Delta^\intercal E_{K} \big] + \Tr\big[ \Sigma_{K} \Delta^\intercal R_{K}\Delta \big] \\
    &= 2\Tr\big[ \Delta^\intercal E_{K} \Sigma_{K} \big] + \Tr\big[ \Sigma_{K} \Delta^\intercal R_{K}\Delta \big] \\
    &= \Tr\big[ \Delta^\intercal \nabla C(K) \big] + \Tr\big[ \Sigma_{K} \Delta^\intercal R_{K}\Delta \big] \\
    &= \Tr\big[ \nabla C(K)^\intercal \Delta \big] + \Tr\big[ \Sigma_{K} \Delta^\intercal R_{K}\Delta \big] \\
    &\leq \Tr\big[ \nabla C(K)^\intercal \Delta \big] + \| \Sigma_{K} \| \Tr\big[ \Delta^\intercal R_{K}\Delta \big] \\
    &= \Tr\big[ \nabla C(K)^\intercal \Delta \big] + \| \Sigma_{K} \| \Tr\big[  R_{K} \Delta \Delta^\intercal \big] \\
    &\leq \Tr\big[ \nabla C(K)^\intercal \Delta \big] + \| \Sigma_{K} \| \| R_{K} \| \Tr\big[ \Delta \Delta^\intercal \big] \\
    &= \Tr\big[ \nabla C(K)^\intercal \Delta \big] + \| \Sigma_{K} \| \| R_{K} \| \Tr\big[ \Delta^\intercal \Delta  \big]
\end{align}
which is exactly of the form of the Lipschitz gradient condition with Lipschitz constant $2 \| \Sigma_{K} \| \| R_{K} \|$. Note this is not a global Lipschitz condition since $2 \| \Sigma_{K} \| \| R_{K} \|$ becomes unbounded as $K$ becomes mean-square destabilizing, but rather a local Lipschitz condition since $2 \| \Sigma_{K} \| \| R_{K} \|$ is bounded on any sublevel set of $C(K)$.
\end{Rem}

\begin{Lem}[Cost bounds] \label{lemma:cost_bounds}
    We always have
    \begin{align}
        \|P_K\| \leq \frac{C(K)}{{\underline{\sigma} (\Sigma_0)}} \quad \text{and} \quad \|\Sigma_K\| \leq \frac{C(K)}{\underline{\sigma}(Q)} .
    \end{align}
\end{Lem}

\begin{proof}
The proof follows that in  \cite{Fazel2018} exactly.
    The cost is lower bounded as
    \begin{align}
        C(K) &= \Tr\left[ P_K \Sigma_0 \right] \geq \|P_K\| \Tr(\Sigma_0)  \geq \|P_K\| \underline{\sigma}(\Sigma_0)  ,
    \end{align}
    which gives the first inequality.
    The cost is also lower bounded as
    \begin{align}
        C(K) &= \Tr\left[ Q_K \Sigma_K \right] \geq \|\Sigma_K\| \Tr(Q_K) \geq \|\Sigma_K\| \underline{\sigma}(Q_K) \geq \|\Sigma_K\| \underline{\sigma}(Q) ,
    \end{align}
    which gives the second inequality.
\end{proof}

\section{Global Convergence of Policy Gradient in the~Model-Based~Setting} \label{sec:conv}
In this section we show that the policy gradient algorithm and two important variants for multiplicative noise LQR converge globally to the optimal policy. In contrast with \cite{Fazel2018}, the policies we obtain are robust to uncertainties and inherent stochastic variations in the system dynamics. 
We analyze three policy gradient algorithm variants:
\begin{description}
    \item [Gradient]                \tabto{25mm} $K_{s+1} = K_s - \eta \nabla C(K_s) $
    \item [Natural Gradient]        \tabto{25mm} $K_{s+1} = K_s - \eta \nabla C(K_s) \Sigma_{K_s}^{-1}$
    \item [Gauss-Newton]            \tabto{25mm} $K_{s+1} = K_s - \eta R_{K_s}^{-1} \nabla C(K_s) \Sigma_{K_s}^{-1}$
\end{description}

The more elaborate natural gradient and Gauss-Newton variants provide superior convergence rates and simpler proofs. A development of the natural policy gradient is given in \cite{Fazel2018} building on ideas from \cite{Kakade2002}. The Gauss-Newton step with step size $\frac{1}{2}$ is in fact identical to the policy improvement step in policy iteration 
(a short derivation is given shortly) and was first studied for deterministic LQR in \cite{Hewer1971}.
This was extended to a model-free setting using policy iteration and Q-learning in \cite{Bradtke1994}, proving asymptotic convergence of the gain matrix to the optimal gain matrix. For multiplicative noise LQR, we have the following results.

\subsection{Derivation of the Gauss-Newton step from policy iteration} \label{app:GNstep}
We start with the policy improvement expression for the LQR problem:
\begin{align}
    u_{s+1} &= \underset{u}{\text{argmin}}\bigg[ x^\intercal Q x + u^\intercal R u + \underset{\delta,\gamma}{\mathbb{E}} V_{K_s}(\widetilde{A}x+\widetilde{B}u) \bigg] = \underset{u}{\text{argmin}}\bigg[ \mathcal{Q}_{K_s}(x,u) \bigg] .
\end{align}
Stationary points occur when the gradient is zero, so differentiating with respect to $u$ we obtain
\begin{align} \label{eq:grad_u}
    \frac{\partial}{\partial u} \mathcal{Q}_{K_s}(x,u) = 2\bigg[ (B^\intercal P_{K_s} A)x + (R+B^\intercal P_{K_s} B + \sum_{j=1}^q \beta_j B_j^\intercal P_{K_s} B_j)u \bigg] .
\end{align}
Setting \eqref{eq:grad_u} to zero and solving for $u$ gives
\begin{align}
    u_{s+1} = -(R+B^\intercal P_{K_s} B + \sum_{j=1}^q \beta_j B_j^\intercal P_{K_s} B_j)^{-1}(B^\intercal P_{K_s} A)x .
\end{align}
Differentiating \eqref{eq:grad_u} with respect to $u$ we obtain
\begin{align}
    \frac{\partial^2}{{\partial u}^2} \mathcal{Q}_{K_s}(x,u) = 2\bigg[R+B^\intercal P_{K_s} B + \sum_{j=1}^q \beta_j B_j^\intercal P_{K_s} B_j \bigg] \succeq 0 \ \forall \ u ,
\end{align}
confirming that the stationary point is indeed a global minimum.

Thus the policy iteration gain matrix update is 
\begin{align} \label{eq:policy_iteration}
    K_{s+1} = -(R+B^\intercal P_{K_s} B + \sum_{j=1}^q \beta_j B_j^\intercal P_{K_s} B_j)^{-1}(B^\intercal P_{K_s} A) .
\end{align}
This can be re-written in terms of the gradient as so:
\begin{align}
  K_{s+1}  &= K_s - K_s - (R+B^\intercal P_{K_s} B + \sum_{j=1}^q \beta_j B_j^\intercal P_{K_s} B_j)^{-1}(B^\intercal P_{K_s} A) \\
            &= K_s - K_s - R_{K_s}^{-1} B^\intercal P_{K_s} A = K_s - R_{K_s}^{-1}\bigg[R_{K_s} K_s + B^\intercal P_{K_s} A \bigg] \\
            &= K_s - R_{K_s}^{-1}\bigg[R_{K_s} K_s + B^\intercal P_{K_s} A \bigg]\Sigma_{K_s}\Sigma_{K_s}^{-1} = K_s - R_{K_s}^{-1}\bigg[(R_{K_s} K_s + B^\intercal P_{K_s} A) \Sigma_{K_s}\bigg]\Sigma_{K_s}^{-1} \\
            &= K_s - \frac{1}{2} R_{K_s}^{-1}\nabla C(K_s) \Sigma_{K_s}^{-1} .
\end{align}
Parameterizing with a step size gives the Gauss-Newton step
\begin{align}
  K_{s+1} &= K_s - \eta R_{K_s}^{-1}\nabla C(K_s) \Sigma_{K_s}^{-1} .
\end{align}

\subsection{Gauss-Newton Descent}
\begin{Thm}[Gauss-Newton convergence] \label{thm:gauss_newton_exact} 
    Using the Gauss-Newton step
    \begin{align}
        K_{s+1} = K_s - \eta R_{K_s}^{-1} \nabla C(K_s) \Sigma_{K_s}^{-1} 
    \end{align}
    with step size $0 < \eta \leq \frac{1}{2}$ gives global convergence to the optimal gain matrix $K^*$ at a linear rate described by
    \begin{align}
       \frac{ C(K_{s+1}) - C(K^*)}{C(K_{s}) - C(K^*)}
      \leq
      1 - 2 \eta \frac{{\underline{\sigma} (\Sigma_0)}}{\|\Sigma_{K^*}\|} .
    \end{align}
\end{Thm}

\begin{proof}
The next-step gain matrix difference is 
\begin{align}
    \Delta  & = K_{s+1} - K_s = -\eta R_{K_s}^{-1} \nabla C(K_s) \Sigma_{K_s}^{-1} = -2\eta R_{K_s}^{-1} E_{K_s} .
\end{align}
Using the almost-smoothness Lemma \ref{lemma:almost_smooth} and substituting in the next-step gain matrix difference we obtain
\begin{align}
    C({K_{s+1}}) - C({K_s})
    &= 2\Tr\big[ \Sigma_{{K_{s+1}}} \Delta^\intercal E_{{K_s}} \big] + \Tr\big[ \Sigma_{{K_{s+1}}} \Delta^\intercal R_{K_s} \Delta \big] \\
    &= 2\Tr\big[ \Sigma_{K_{s+1}} (-2\eta R_{K_s}^{-1} E_{K_s})^\intercal E_{K_s} \big]
    + \Tr\big[ \Sigma_{K_{s+1}} (-2\eta R_{K_s}^{-1} E_{K_s})^\intercal R_{K_s} (-2\eta R_{K_s}^{-1} E_{K_s}) \big] \\
    &= 4(-\eta+\eta^2) \Tr\big[ \Sigma_{K_{s+1}} E_{K_s}^\intercal R_{K_s}^{-1}  E_{K_s} \big] .
\end{align}
By hypothesis we require $0 \leq \eta \leq \frac{1}{2}$ so we have
\begin{align*}
    C({K_{s+1}}) - C({K_s}) &\leq -2 \eta \Tr\big[ \Sigma_{{K_{s+1}}} E_{K_s}^\intercal R_{K_s}^{-1}  E_{{K_s}} \big] \leq -2 \eta  \underline{\sigma}(\Sigma_{{K_{s+1}}}) \Tr\big[ E_{K_s}^\intercal R_{K_s}^{-1}  E_{{K_s}} \big] \\
                            &\leq -2 \eta  {\underline{\sigma} (\Sigma_0)} \Tr\big[ E_{K_s}^\intercal R_{K_s}^{-1}  E_{{K_s}} \big] .
\end{align*}
Recalling and substituting in \eqref{eq:grad_dom_GN} we obtain
\begin{align}
    C({K_{s+1}}) - C({K_s}) &\leq -2 \eta \frac{ {\underline{\sigma} (\Sigma_0)}}{\|\Sigma_{K^*}\|} \big(C(K_s) - C(K^*) \big) .
\end{align}
Adding $C(K_s) - C(K^*)$ to both sides and rearranging completes the proof.
\end{proof}

\subsection{Natural Policy Gradient Descent}

\begin{Thm}[Natural policy gradient convergence] \label{thm:nat_grad_exact_convergence} 
    Using the natural policy gradient step
    \begin{align} \label{eq:nat_grad_update}
        K_{s+1} = K_s - \eta \nabla C(K_s) \Sigma_{K_s}^{-1} 
    \end{align}
    with step size $0 < \eta \leq c_{npg}$ where 
    \begin{align*}
        c_{npg} \Let \frac{1}{2} \bigg(\|R\| + \Big(\|B\|^2 + \sum_{j=1}^q \beta_j \|B_j\|^2\Big) \frac{C(K_0)}{{\underline{\sigma} (\Sigma_0)}}\bigg)^{-1}
    \end{align*}
    gives global convergence to the optimal gain matrix $K^*$ at a linear rate described by
    \begin{align}
      \frac{C(K_{s+1}) - C(K^*)}{ C(K_{s}) - C(K^*)}
      \leq 
       1 - 2 \eta \frac{\underline{\sigma}(R){\underline{\sigma} (\Sigma_0)} }{\|\Sigma_{K^*}\|} .
    \end{align}
\end{Thm}

\begin{proof}
First we bound the one-step progress, where the step size depends explicitly on the current gain  $K_s$.
Using the update \eqref{eq:nat_grad_update}, the next-step gain matrix difference is 
\begin{align}
    \Delta = K_{s+1} - K_s  &= -\eta \nabla C(K_s) \Sigma_{K_s}^{-1} = -2\eta E_{K_s} .
\end{align}
Using Lemma \ref{lemma:almost_smooth} and substituting we obtain
\begin{align*}
    & C({K_{s+1}}) - C({K_s}) \\
    &= 2\Tr\big[ \Sigma_{{K_{s+1}}} \Delta^\intercal E_{{K_s}} \bigg] + \Tr\big[ \Sigma_{{K_{s+1}}} \Delta^\intercal R_{K_s} \Delta \big] \\
    &= 2\Tr\big[ \Sigma_{K_{s+1}} (-2\eta E_{K_s})^\intercal E_{K_s} \big] 
    + \Tr\big[ \Sigma_{K_{s+1}} (-2\eta E_{K_s})^\intercal R_{K_s} (-2\eta E_{K_s}) \big] \\
    &= -4\eta \Tr\big[ \Sigma_{K_{s+1}} E_{K_s}^\intercal E_{K_s} \big] +4\eta^2 \Tr\big[ \Sigma_{K_{s+1}} E_{K_s}^\intercal R_{K_s} E_{K_s} \big] \\
    &\leq 4 (-\eta + \eta^2 \|R_{K_s}\|) \Tr\big[ \Sigma_{K_{s+1}} E_{K_s}^\intercal E_{K_s} \big] .
\end{align*}
If we choose step size $0 < \eta \leq \frac{1}{2 \|R_{K_s}\|}$, then
\begin{align}
    C({K_{s+1}}) - C({K_s}) &\leq -2\eta \Tr\big[ \Sigma_{K_{s+1}} E_{K_s}^\intercal E_{K_s} \big] \leq -2\eta \underline{\sigma}(\Sigma_{K_{s+1}}) \Tr\big[E_{K_s}^\intercal E_{K_s} \big] \leq -2\eta {\underline{\sigma} (\Sigma_0)} \Tr\big[E_{K_s}^\intercal E_{K_s} \big] .
\end{align}
Recalling and substituting \eqref{eq:grad_dom_ng} we obtain
\begin{align}
    C({K_{s+1}}) - C({K_s}) \leq -2\eta {\underline{\sigma} (\Sigma_0)} \frac{\underline{\sigma}(R)}{\|\Sigma_{K^*}\|} \big( C(K_s) - C(K^*) \big) .
\end{align}
Adding $C(K_s) - C(K^*)$ to both sides and rearranging gives the one step progress bound
\begin{align} \label{eq:nat_grad_exact_one_step}
  \frac{ C(K_{s+1}) - C(K^*)}{C(K_{s}) - C(K^*)} \leq  1 - 2 \eta \frac{\underline{\sigma}(R) {\underline{\sigma} (\Sigma_0)}}{\|\Sigma_{K^*}\|} .
\end{align}
Next, using the cost bound in Lemma \ref{lemma:cost_bounds}, the triangle inequality, and submultiplicativity of spectral norm we have
\begin{align}
    \frac{1}{\|R_K\|}   & = \frac{1}{\| R + B^\intercal P_{K} B + \sum_{j=1}^q \beta_j B_j^\intercal P_K B_j \|} \\
                        & \geq \frac{1}{\|R\| + (\|B\|^2 + \sum_{j=1}^q \beta_j \|B_j\|^2) \|P_K\|} \geq \frac{1}{\|R\| + (\|B\|^2 + \sum_{j=1}^q \beta_j \|B_j\|^2) \frac{C(K)}{{\underline{\sigma} (\Sigma_0)}}} .
\end{align}
Accordingly, choosing the step size as $0 < \eta \leq c_{npg}$ ensures \eqref{eq:nat_grad_exact_one_step} holds at the first step. This ensures that $C(K_1) \leq C(K_0)$ which in turn ensures
\begin{align}
    \eta &\leq \frac{1}{\|R\| + (\|B\|^2 + \sum_{j=1}^q \beta_j \|B_j\|^2) \frac{C(K_0)}{{\underline{\sigma} (\Sigma_0)}}} \leq \frac{1}{\|R\| + (\|B\|^2 + \sum_{j=1}^q \beta_j \|B_j\|^2) \frac{C(K_1)}{{\underline{\sigma} (\Sigma_0)}}} \leq \frac{1}{\|R_{K_1}\|}
\end{align}
which allows \eqref{eq:nat_grad_exact_one_step} to be applied at the next step as well. Proceeding inductively by applying \eqref{eq:nat_grad_exact_one_step} at each successive step completes the proof.
\end{proof}

\subsection{Policy Gradient Descent}

\begin{Thm}[Policy gradient convergence] \label{thm:grad_exact_convergence} 
    Using the policy gradient step
    \begin{align}
        K_{s+1} = K_s - \eta \nabla C(K_s)
    \end{align}
    with step size $0 < \eta \leq c_{pg}$ gives global convergence to the optimal gain matrix $K^*$ at a linear rate described by
    \begin{align}
      \frac{C(K_{s+1}) - C(K^*)}{C(K_{s}) - C(K^*) }
      \leq 
      1 - 2 \eta \frac{ \underline{\sigma}(R){\underline{\sigma} (\Sigma_0)}^2}{\|\Sigma_{K^*}\|}
    \end{align}
    where $c_{\text{pg}}$ is a polynomial in the problem data $A$, $B$, $\alpha_i$, $\beta_j$, $A_i$, $B_j$, $Q$, $R$, $\Sigma_0$, $K_0$ given in the proof in Appendix \ref{appendix:model_based_gd}.
\end{Thm}

\begin{proof}
The proof is developed in Appendix \ref{appendix:model_based_gd}.
\end{proof}

The proofs for these results explicitly incorporate the effects of the multiplicative noise terms $\delta_{ti}$ and $\gamma_{tj}$ in the dynamics. 
For the policy gradient and natural policy gradient algorithms, we show explicitly how the maximum allowable step size depends on problem data and in particular on the multiplicative noise terms. Compared to deterministic LQR, the multiplicative noise terms decrease the allowable step size and thereby decrease the convergence rate; specifically, the state-multiplicative noise increases the initial cost $C(K_0)$ and the norms of the covariance $\Sigma_{K^*}$ and cost $P_K$, and the input-multiplicative noise also increases the denominator term $\|B\|^2+\sum_{j=1}^q \beta_j \|B_j\|^2$. This means that the algorithm parameters for deterministic LQR in \cite{Fazel2018} may cause failure to converge on problems with multiplicative noise. Moreover, even the optimal policies for deterministic LQR may actually \emph{destabilize} systems in the presence of small amounts of multiplicative noise uncertainty, indicating the possibility for a catastrophic lack of robustness; observe the results of the example in Section \ref{sec:numexp_A}. The results and proofs also differ from that of \cite{Fazel2018} because the more complicated mean-square stability must be accounted for, and because \emph{generalized} Lyapunov equations must be solved to compute the gradient steps, which requires specialized solvers.

\section{Global Convergence of Policy Gradient in the~Model-Free~Setting} \label{sec:modelfree}

The results in the previous section are model-based; the policy gradient steps are computed exactly based on knowledge of the model parameters. In the model-free setting, the policy gradient is estimated to arbitrary accuracy from sample trajectories with a  sufficient number of sample trajectories $n_{\text{sample}}$ of sufficiently long horizon length $\ell$ using gain matrices randomly selected from a Frobenius-norm ball around the current gain of sufficiently small exploration radius $r$. We show for multiplicative noise LQR that with a finite number of samples polynomial in the problem data, the model-free policy gradient algorithm still converges to the globally optimal policy, despite small perturbations on the gradient.

In the model-free setting, the policy gradient method proceeds as before except that at each iteration Algorithm \ref{algorithm:algo1} is called to generate an estimate of the gradient via the zeroth-order optimization procedure described by Fazel et al. \cite{Fazel2018}.

\begin{algorithm}
\caption{Model-Free policy gradient estimation}
\begin{algorithmic}[1]
\label{algorithm:algo1}
    \REQUIRE Gain matrix $K$, number of samples $n_{\text{sample}}$, rollout length $\ell$, exploration radius $r$
    \FOR{$i=1,\ldots,n_{\text{sample}}$}
        \STATE Generate a sample gain matrix $\widehat{K}_{i}=K+U_{i},$ where $U_{i}$ is drawn uniformly at random over matrices with Frobenius norm $r$ \;
        \STATE Generate a sample initial state $x_0^{(i)} \sim \mathcal{P}_0$ \;
        \STATE Simulate the closed-loop system for $\ell$ steps according to the stochastic dynamics in \eqref{eq:LQRm} starting from $x_0^{(i)}$ with $u_t^{(i)} = \widehat{K}_{i} x_t^{(i)}$, yielding the state sequence $\{ x_t^{(i)} \}_{t=0}^{t=\ell}$ \;
        \STATE Collect the empirical finite-horizon cost estimate ${\widehat{C}_{i} \Let \sum_{t=0}^{\ell} {x_t^{(i)}}^\intercal (Q+\widehat{K}_{i}^\intercal R \widehat{K}_{i}) x_t^{(i)}}$ \;
    \ENDFOR
    \ENSURE Gradient estimate $\widehat{\nabla} C(K) \Let \frac{1}{n_{\text{sample}}} \sum_{i=1}^{n_{\text{sample}}} \frac{mn}{r^{2}} \widehat{C}_{i} U_{i}$
\end{algorithmic}
\end{algorithm}

\begin{Thm}[Model-free policy gradient convergence] \label{thm:model-free}
    Let $\epsilon$ and $\mu$ be a given small tolerance and probability respectively and $N$ be the number of gradient descent steps taken.
    Suppose that the distribution of the initial states is bounded such that $x_0 \sim \mathcal{P}_0$ implies $\|x_0^i\| \leq L_0$ almost surely for any given realization $x_0^i$ of $x_0$. 
    Suppose additionally that the distribution of the multiplicative noises is bounded such that the following inequality is satisfied almost surely for any given realized sequence $x_t^i$ of $x_t$ with a positive scalar $z \geq 1$:
    \begin{align*} 
        \sum_{t=0}^{\ell-1} \Big( {x_{t}^{i}}^\intercal Q x_{t}^{i}+{u_{t}^{i}}^\intercal R u_{t}^{i} \Big)  \leq z \underset{\delta, \gamma}{\mathbb{E}} \left[ \sum_{t=0}^{\ell-1} \Big( x_t^\intercal Q x_t + u_t^\intercal R u_t \Big) \right]
    \end{align*}
    under the closed-loop dynamics with any gain such that $C(K) \leq 2 C(K_0)$.
    Suppose the step size $\eta$ is chosen according to the restriction in Theorem \ref{thm:grad_exact_convergence} and at every iteration the gradient is estimated according to the finite-horizon procedure in Algorithm \ref{algorithm:algo1} where the number of samples $n_{\text{sample}}$, rollout length $\ell$, and exploration radius $r$ are chosen according to the fixed polynomials of the problem data $A$, $B$, $\alpha_i$, $\beta_j$, $A_i$, $B_j$, $Q$, $R$, $\Sigma_0$, $K_0$, $L_0$ and $z$ which are all defined in the proofs in Appendix \ref{appendix:model_free_gd}.
    Then, with high probability of at least $1 - \mu$, performing gradient descent results in convergence to the global optimum over all $N$ steps: at each step, either progress is made at the linear rate
    \begin{align}
         \frac{C(K_{s+1})-C(K^*)}{C(K_s)-C\left(K^{*}\right)} \leq 1 - \eta  \frac{\underline{\sigma}(R)\underline{\sigma}(\Sigma_0)^{2}}{\left\|\Sigma_{K^{*}}\right\|} .
    \end{align}
    or convergence has been attained with $C(K_s) - C(K^*) \leq \epsilon$.
\end{Thm}

\begin{proof}
The proof is developed in Appendix \ref{appendix:model_free_gd}.
\end{proof}

From a sample complexity standpoint, it is notable that the number of samples $n_{\text{sample}}$, rollout length $\ell$, and exploration radius $r$ in Theorem \ref{thm:model-free} are polynomial in the problem data $A$, $B$, $\alpha_i$, $\beta_j$, $A_i$, $B_j$, $Q$, $R$, $\Sigma_0$, $C(K_0)$.
The constant $z$ imposes a bound on the multiplicative noise, which is naturally absent in \cite{Fazel2018}. 
Note that $z \geq 1$ since any upper bound of a scalar distribution with finite support must be equal to or greater than the mean. In general, this implicitly requires the noises to have bounded support. 
Such an assumption is qualitatively the same as the condition imposed on the initial states.
These assumptions are reasonable; in a practical setting with a physical system the initial state and noise distributions will have finite support. There is no restriction on how large the support is, only that it not be unbounded.
Also note that the rate is halved compared with the model-based case of Theorem \ref{thm:grad_exact_convergence}; this is because the ``other half'' is consumed by the error between the estimated and true gradient.

\section{Numerical Experiments} \label{sec:numexp}
In this section we present results for three systems:
\begin{enumerate}[label=\Alph*.]
    \item Shows that ``optimal'' control that ignores actual multiplicative noise can lead to loss of mean-square stability,
    \item Shows the efficacy of the policy gradient algorithms on a networked system,
    \item Shows the increased difficulty of estimating the gradient from sample data in the presence of multiplicative noise.
\end{enumerate}
All systems considered permitted a solution to the GARE \eqref{eq:genriccati}.
The bounds on the step size, number of rollouts, and rollout length given by the theoretical analysis can be rather conservative.
For practicality, we selected the constant step size, number of rollouts, rollout length, and exploration radius according to a grid search over reasonable values. Additionally, we investigated the use of backtracking line search to adaptively select the step size; see e.g. \cite{boyd2004convex}. Throughout the simulations, we computed the baseline optimal cost $C(K^*)$ by solving the GARE \eqref{eq:genriccati} to high precision via value iteration.
Python code which implements the algorithms and generates the figures reported in this work can be found in the GitHub repository at \url{https://github.com/TSummersLab/polgrad-multinoise/}.
The code was run on a desktop PC with a quad-core Intel i7 6700K 4.0GHz CPU, 16GB RAM; no GPU computing was utilized.

\subsection{Importance of Accounting for Multiplicative Noise} \label{sec:numexp_A}

We first considered an open-loop mean-square unstable system with four states and one input representing an active two-mass suspension converted from continuous to discrete time using a standard bilinear transformation, with parameters:
\begin{align*}
    A =
    \begin{bmatrix}
    +0.261 & +0.315 & +0.093 & -0.008 \\
    -2.955 & +0.261 & +0.373 & -0.033 \\
    +1.019 & +0.255 & -0.853 & +0.011 \\
    -3.170 & -0.793 & -4.902 & -0.146
    \end{bmatrix}  , \quad
    B = 
    \begin{bmatrix}
        0.133 \\
        0.532 \\
        0.161 \\
        2.165
    \end{bmatrix}  , \quad
    Q = I_4, \quad
    R = I_1, \\   
    [A_i]_{y,z} = 
    \begin{cases}
    1 \text{ if } z = i , \\
    0 \text{ otherwise,}
    \end{cases}
    B_1 = \mathbf{1}_{4 \times 1} , \quad
    \{ \alpha_i \} = \{ 0.017, 0.017, 0.017, 0.017 \} , \quad
    \beta_1  = 0.035 .
\end{align*}
We performed model-based policy gradient descent; at each iteration gradients were calculated by solving generalized Lyapunov equations \eqref{eq:glyap1} and \eqref{eq:glyap2} using the problem data. 
The gains $K_m$ and $K_\ell$ represent iterates during optimization of (``training'' on) the LQRm and LQR cost (with the multiplicative noise variances set to zero), respectively.
We performed the optimization starting from the same feasible initial gain, 
which was generated by perturbing the exact solution of the generalized algebraic Riccati equation such that the LQRm cost under the initial control was approximately 10 times that of the optimal control.
The step size was chosen via backtracking line search. The optimization stopped once the Frobenius norm of the gradient fell below a small threshold. 
The plot in Fig. \ref{fig:plot_cost_vs_iteration_suspension_both} shows the ``testing'' cost of the gains at each iteration evaluated on the LQRm cost (with multiplicative noise).
From this figure, it is clear that $K_m$ minimized the LQRm as desired.
When there was high multiplicative noise, the noise-ignorant controller $K_\ell$ actually \textit{destabilized} the system in the mean-square sense; this can be seen as the LQRm cost exploded upwards to infinity after iteration 10.
In this sense, the multiplicative noise-aware optimization is generally safer and more robust than noise-ignorant optimization, and in examples like this is actually \textit{necessary} for mean-square stabilization.

\begin{figure}[ht]
\centering
 \subfloat[LQRm cost. \label{fig:plot_lqrm_cost_vs_iteration_suspension}]{%
  \includegraphics[width=0.45\textwidth]{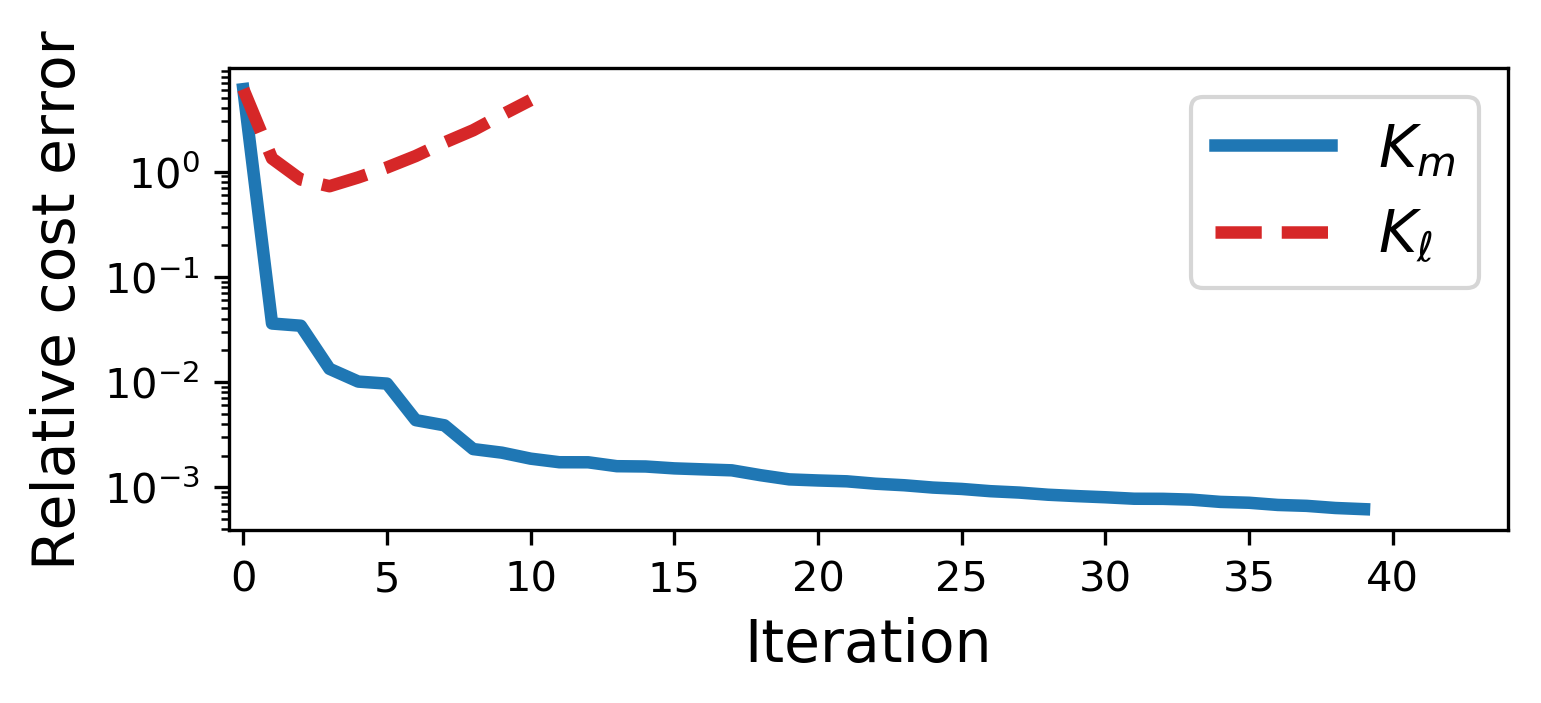}
 }
 \quad
 \subfloat[LQR cost. \label{fig:plot_lqr_cost_vs_iteration_suspension}]{%
  \includegraphics[width=0.45\textwidth]{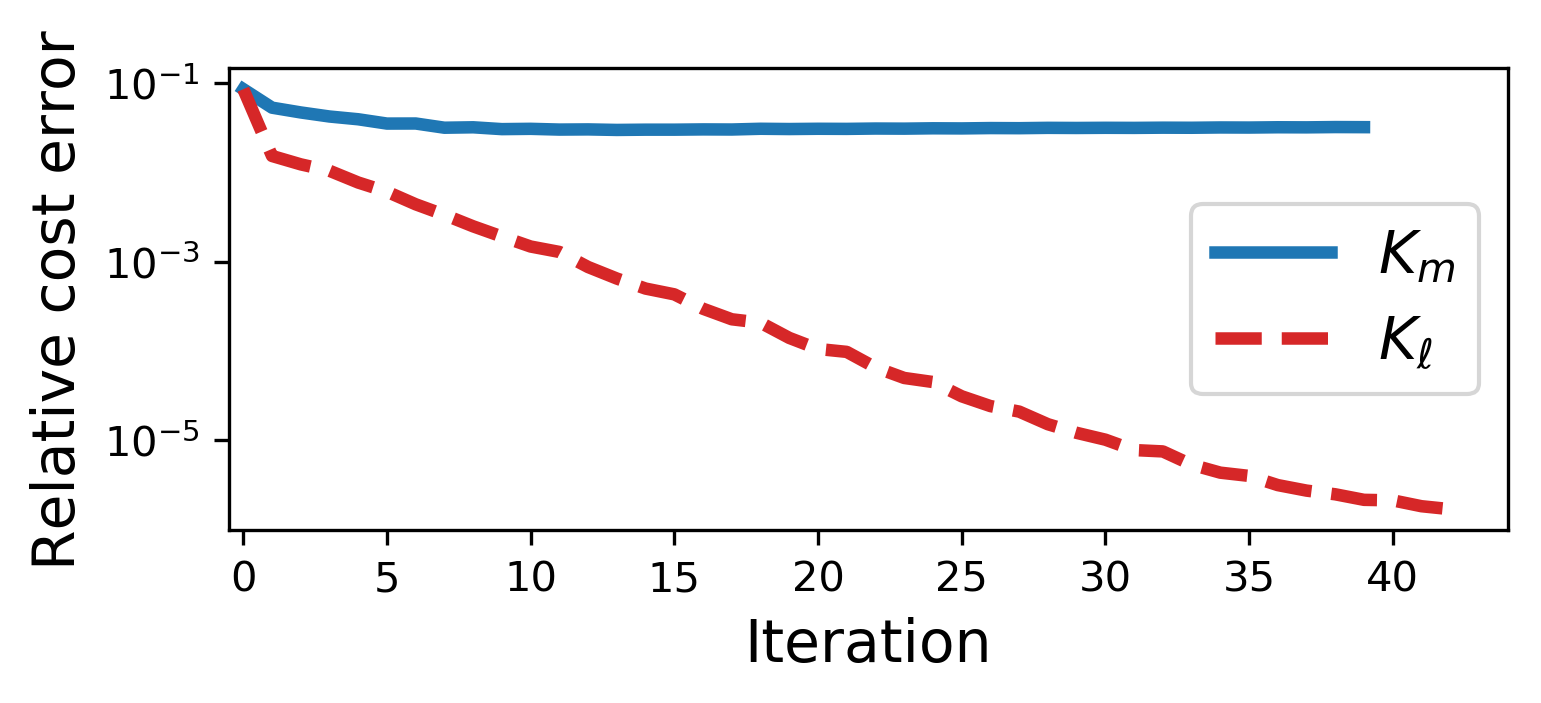}
 }
 \caption{Relative cost error $\frac{C(K)-C(K^*)}{C(K^*)}$ vs. iteration during policy gradient descent on the 4-state, 1-input suspension example system.}
 \label{fig:plot_cost_vs_iteration_suspension_both}
\end{figure}

\subsection{Policy Gradient Methods Applied to a Network}
Many practical networked systems can be approximated by diffusion dynamics with losses and stochastic diffusion constants (edge weights) between nodes; examples include heat flow through uninsulated pipes, hydraulic flow through leaky pipes, information flow between processors with packet loss, electrical power flow between generators with resistant electrical power lines, etc. A derivation of the discrete-time dynamics of this system is given in \cite{gravell2020acc}.
We considered a particular 4-state, 4-input system and open-loop mean-square stable with the following parameters:
\begin{align*}
    A &= 
    \begin{bmatrix}
    0.795 & 0.050 & 0.100 & 0.050 \\
    0.050 & 0.845 & 0.050 & 0.050 \\
    0.100 & 0.050 & 0.695 & 0.150 \\
    0.050 & 0.050 & 0.150 & 0.745
    \end{bmatrix}, \quad
    B = Q = R = \Sigma_0 = I_4 , \\
    \{ \alpha_i \} &= \{ 0.005, 0.015, 0.010 , 0.015, 0.005, 0.020 \} , \quad
    \{ \beta_j \} = \{ 0.050, 0.150, 0.050, 0.100 \} , \\
    [A_i]_{y,z} &= 
    \begin{cases}
    +1 \text{ if } \{ c_i \! = \! y \ \& \ d_i \! = \! y \} \text{ or } \{ c_i \! = \! z \ \& \ d_i \! = \! z \} , \\
    -1 \text{ if } \{ c_i \! = \! z \ \& \ d_i \! = \! y \} \text{ or } \{ c_i \! = \! y \ \& \ d_i \! = \! z \} , \\
    0 \text{ otherwise, }
    \end{cases} \quad
    [B_j]_{y,z} = 
    \begin{cases}
    +1 \text{ if } j=y=z, \\
    0 \text{ otherwise. }
    \end{cases} \\
    \{ (c_i, d_i) \} &= \{ (1, 2), (1, 3), (1, 4), (2, 3), (2, 4), (3, 4) \} .
\end{align*}
This system is open-loop mean-square stable, so we initialized the gains to all zeros for each trial.
We performed policy optimization using the model-free gradient, and the model-based gradient, model-based natural gradient, and model-based Gauss-Newton step directions on 20 unique problem instances using two step size schemes: \\
\noindent
\textbf{Backtracking line search:} Step sizes $\eta$ were chosen adaptively at each iteration by backtracking line search with parameters $\alpha = 0.01$, $\beta = 0.5$ (see \cite{boyd2004convex} for a description), except for Gauss-Newton which used the optimal constant step-size of $1/2$. Model-free gradients and costs were estimated with 100,000 rollouts per iteration. We ran a fixed number, 20, of iterations chosen such that the final cost using model-free gradient descent was no more than $5\%$ worse than optimal. \\
\noindent
\textbf{Constant step size:} Step sizes were set to constants chosen as large as possible without observing infeasibility or divergence, which on this problem instance was $\eta = 5 \times 10^{-5}$ for gradient, $\eta = 2 \times 10^{-4}$ for natural gradient, and $\eta = 1/2$ for Gauss-Newton step directions. Model-free gradients were estimated with 1,000 rollouts per iteration. We ran a fixed number, 20,000, of iterations chosen such that convergence was achieved with all step directions.

In both cases sample gains were chosen for model-free gradient estimation with exploration radius $r \! = \! 0.1$ and the rollout length was set to $\ell \! = \! 20$.
The plots in Fig. \ref{fig:plot_costnorm_vs_iteration_random} show the relative cost over the iterations; for the model-free gradient descent, the bold centerline is the mean of all trials and the shaded region is between the 10\textsuperscript{th} and 90\textsuperscript{th} percentile of all trials. 
Using backtracking line search, it is evident that in terms of convergence the Gauss-Newton step was extremely fast, and both the natural gradient and model-based gradient were slightly slower, but still quite fast. The model-free policy gradient converged to a reasonable neighborhood of the minimum cost quickly, but stagnated with further iterations; this is a consequence of the inherent gradient and cost estimation errors that arise due to random sampling and the multiplicative noise. 
Using constant stepsizes, we were forced to take small steps due to the steepness of the cost function near the initial gains, slowing overall convergence using the gradient and natural gradient methods. Here we observed that Gauss-Newton again converged most quickly, followed by natural gradient and lastly the gradient methods. The smaller step size also allowed us to use far fewer samples in the model-free setting, where we observed somewhat faster initial cost decrease with eventual stagnation around $10^{-2}$, or 1\%,  relative error, which represents excellent control performance. All algorithms exhibited convergence to the optimum, confirming the asserted theoretical claims.

\begin{figure}[ht]
\centering
 \subfloat[Backtracking line search. \label{fig:plot_costnorm_vs_iteration_random_backtrack}]{%
  \includegraphics[width=0.45\textwidth]{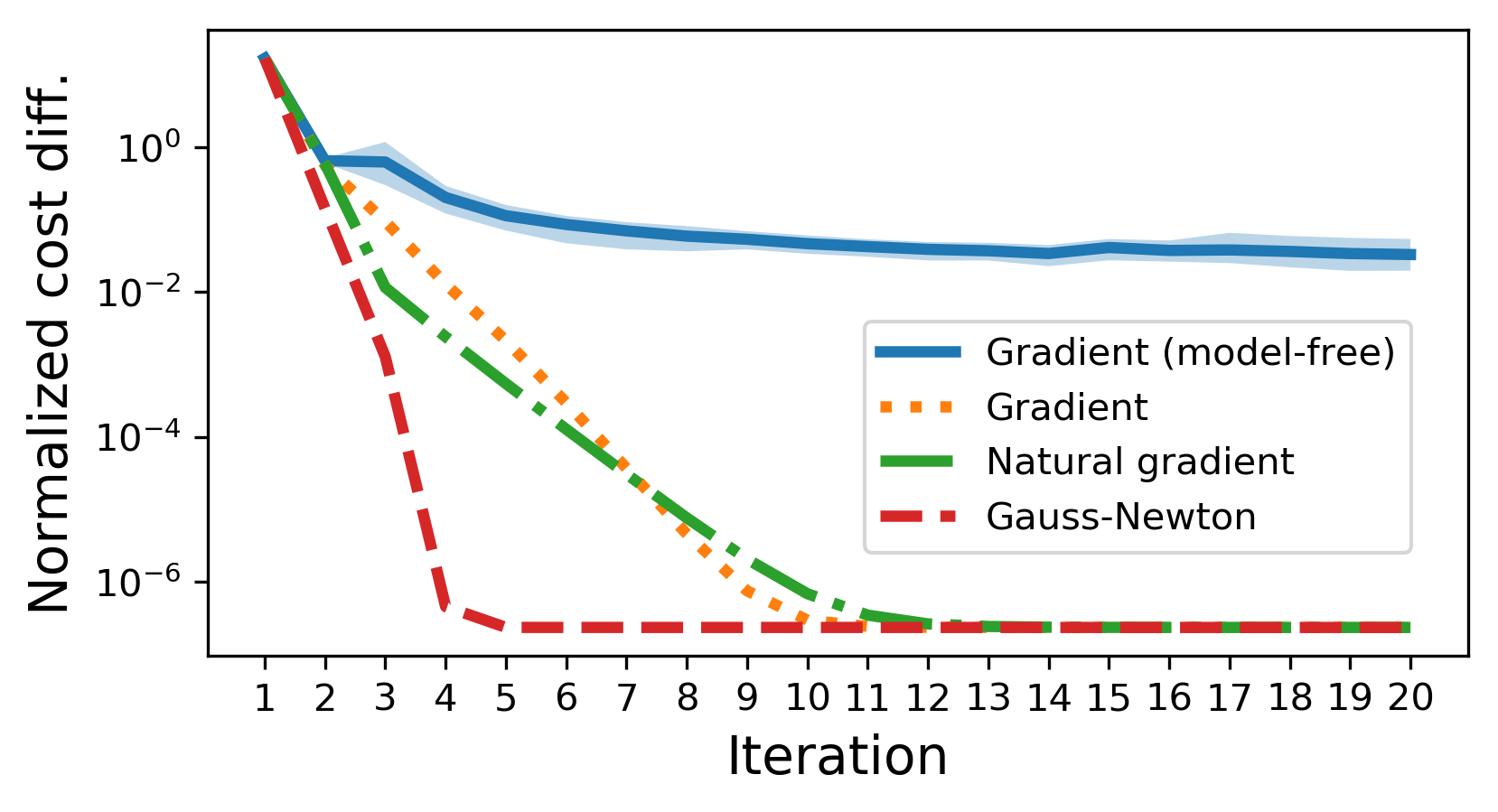}
 }\quad 
 \subfloat[Constant step sizes. \label{fig:plot_costnorm_vs_iteration_random_constant}]{%
  \includegraphics[width=0.45\textwidth]{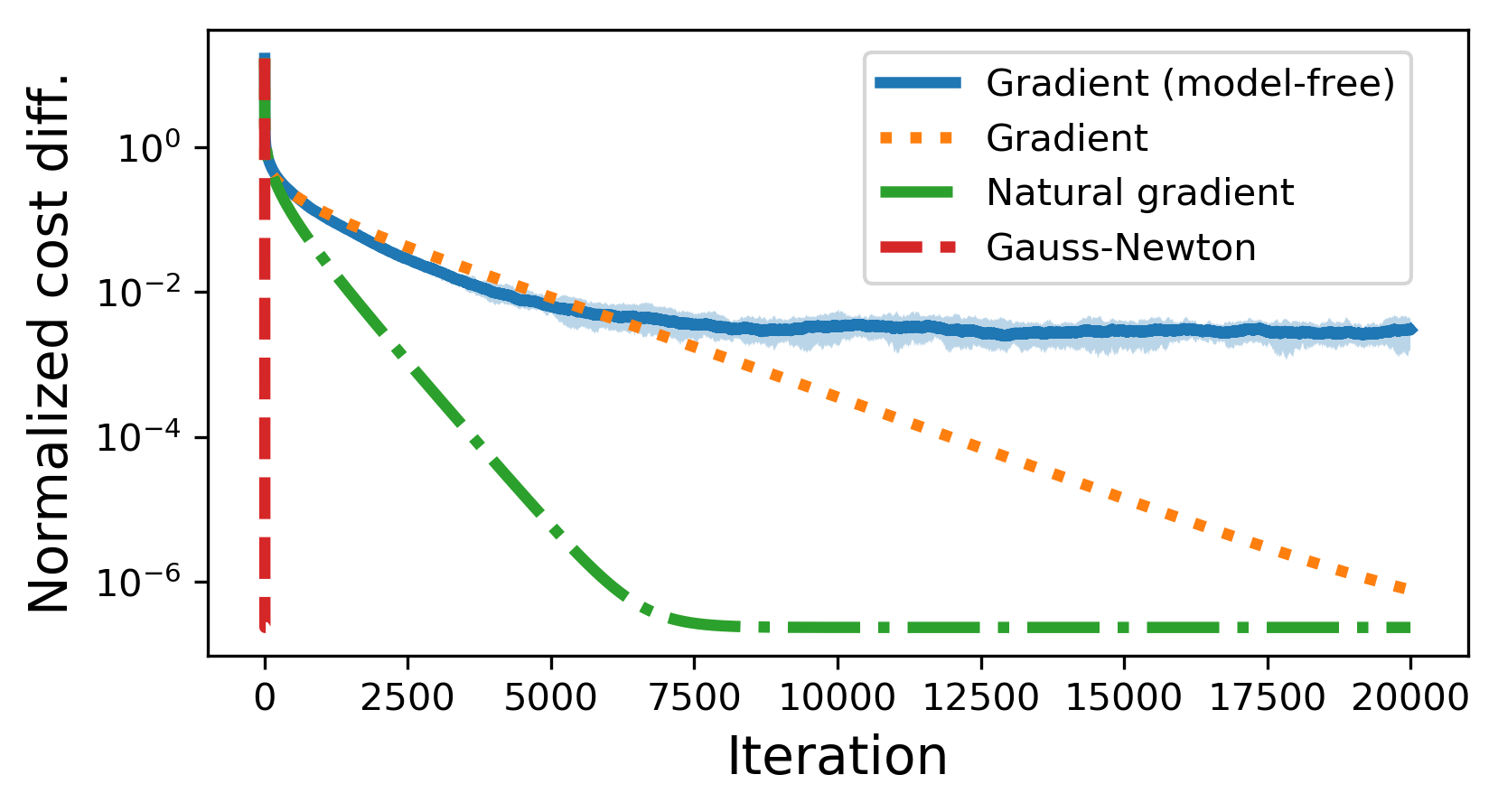}
 }
 \caption{Relative cost error $\frac{C(K)-C(K^*)}{C(K^*)}$ vs. iteration during policy gradient methods on a 4-state, 4-input lossy diffusion network with multiplicative noise using a) backtracking line search and b) constant step sizes.}
 \label{fig:plot_costnorm_vs_iteration_random}
\end{figure}

\subsection{Gradient Estimation}
Multiplicative noise can significantly increase the variance and sample complexity of cost gradient estimates relative to the noiseless case, which is novelly reflected in the theoretical analysis for the number of rollouts and rollout length. 
To demonstrate this empirically, we evaluated the relative gradient estimation error vs. number of rollouts for the system
\begin{align} \label{eq:gradient_estimation_example}
    {x_{t+1}=\left( \begin{bmatrix} 0.8 & 0.1 \\ 0.1 & 0.8 \end{bmatrix} + \delta_t \begin{bmatrix} 0 & 1 \\ 1 & 0 \end{bmatrix} + \begin{bmatrix} 1 \\ 0 \end{bmatrix} K \right) x_t}
\end{align}
with $K = 0, Q = \Sigma_0 = I_2, R = 1$, $\delta_t \sim \mathcal{N}(0,0.1)$, rollout length $l=40$, exploration radius $r = 0.2$, averaged over 10 gradient estimates. The results are plotted in Figure \ref{fig:gradient_estimation}. To achieve the same gradient estimate error of $10\%$, the system with multiplicative noise required $200 \times$ the number of rollout samples ($10^8$) as when there was no noise ($5 \times 10^5$). 

\begin{figure}[ht]
\centering
\includegraphics[width=0.45\textwidth]{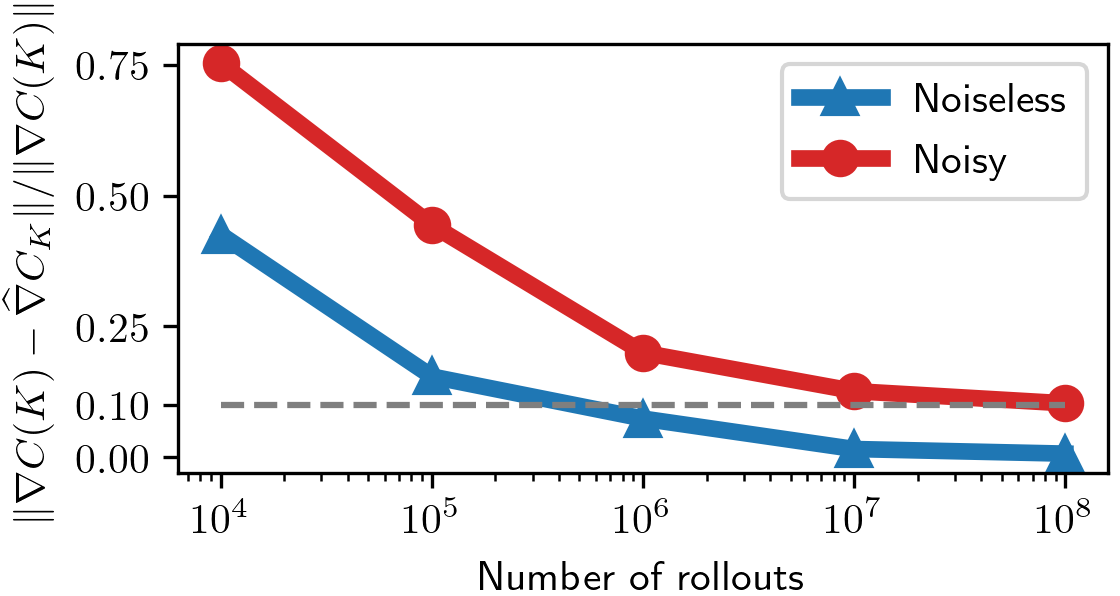}
\caption{Relative gradient estimation error vs. number of rollouts for \eqref{eq:gradient_estimation_example}.}
\label{fig:gradient_estimation}
\end{figure}

\section{Conclusions}

We have shown that policy gradient methods in both model-based and model-free settings give global convergence to the globally optimal policy for LQR systems with multiplicative noise. These techniques are directly applicable for the design of robust controllers of uncertain systems and serve as a benchmark for data-driven control design. Our ongoing work is exploring ways of mitigating the relative sample inefficiency of model-free policy gradient methods by leveraging the special structure of LQR models and Nesterov-type acceleration, and exploring alternative system identification and adaptive control approaches. We are also investigating other methods of building robustness through $\mathcal{H}_\infty$ and dynamic game approaches. Another extension relevant to networked control systems is enforcing sparse structure constraints on the gain matrix via projected policy gradient as suggested in \cite{bu2019lqr}.

\appendix
\section{Standard matrix expressions} \label{sec:matrix}

Before proceeding with the proof of the main results of this study, we first review several basic expressions that will be used later throughout the section. 
In this section we let $A$, $B$, $C$, $M_i$ be generic matrices $\in \mathds{R}^{n \times m}$, $a$, $b$ be generic vectors, and $s$ be a generic scalar.
\begin{description}
\item[Spectral norm] \hfill \\
We denote the matrix spectral norm as $\|A\| = \sigma_{\text{max}}(A)$ which clearly satisfies
\begin{equation} \label{eq:spectral_max_sing}
    \|A\| = \sigma_{\text{max}}(A) \geq \sigma_{\min}(A) .
\end{equation}

\item[Frobenius norm] \hfill \\
We denote the matrix Frobenius norm as $\|A\|_F$ whose square satisfies
\begin{equation} \label{eq:fro_sq}
    \|A\|_F^2 = \Tr(A^T A) .
\end{equation}

\item[Frobenius norm $\geq$ spectral norm] \hfill \\
For any matrix $A$ the Frobenius norm is greater than or equal to the spectral norm:
\begin{equation} \label{eq:fro_greaterthan_spectral}
    \|A\|_F \geq \|A\| .
\end{equation}

\item[Inverse of spectral norm inequality]
\begin{equation} \label{eq:spectral_norm_inv}
    \|A^{-1}\| \geq \|A\|^{-1} .
\end{equation}

\item[Invariance of trace under cyclic permutation]
\begin{equation} \label{eq:trace_invar}
    \Tr \left( \prod_{i = 1}^{n}M_i \right) = \Tr \left( M_n \prod_{i = 1}^{n-1}M_i \right) .
\end{equation}

\item[Invariance of trace under arbitrary permutation for a product of three matrices]
\begin{equation} \label{eq:trace_invar_triple}
    \Tr(ABC) = \Tr(BCA) = \Tr(CAB) = \Tr(ACB) = \Tr(BAC) = \Tr(CBA) .
\end{equation}

\item[Scalar trace equivalence]
\begin{equation} \label{eq:scalar_trace}
    s = \Tr(s) .
\end{equation}

\item[Trace-spectral norm inequalities]
\begin{equation} \label{eq:trace_norm1}
    |\Tr(A^TB)| \leq \|A^T\||\Tr(B)| = \|A\||\Tr(B)| .
\end{equation}
If $A \in \mathds{R}^{n \times n}$
\begin{equation} \label{eq:trace_norm2}
|\Tr(A)| \leq n\|A\|
\end{equation}
and if $A \succeq 0$
\begin{equation} \label{eq:trace_norm3}
\Tr(A) \geq \|A\| .
\end{equation}

\item[Sub-multiplicativity of spectral norm]
\begin{equation} \label{eq:sub_mult_spectral}
    \|AB\| \leq \|A\| \|B\| .
\end{equation}

\item[Positive semidefinite matrix inequality] \hfill \\
Suppose $A \succeq 0$ and $B \succeq 0$. Then
\begin{equation} \label{eq:psd_ineq}
    A + B \succeq A \text{ \ and \ } A + B \succeq B .
\end{equation}

\item[Vector self outer product positive semidefiniteness]
\begin{equation} \label{eq:vec_self_outer_product}
    aa^T \succeq 0
\end{equation}
since $b^Taa^Tb=(a^Tb)^T(a^Tb) \geq 0 \text{ for any } b$.

\item[Singular value inequality for positive semidefinite matrices] \hfill \\
Suppose $A \succeq 0$ and $B \succeq 0$ and $A \succeq B$. Then
\begin{equation} \label{eq:psd_spectral_ineq}
    \sigma_{\min}(A) \geq \sigma_{\min}(B) .
\end{equation}

\item[Weyl's Inequality for singular values] \hfill \\
Suppose $B=A+C$. Let singular values of $A$, $B$, and $C$ be
\begin{align}
    \sigma_1(A) \geq \sigma_2(A) \geq \ldots \geq \sigma_r(A) \geq 0 \nonumber \\
    \sigma_1(B) \geq \sigma_2(B) \geq \ldots \geq \sigma_r(B) \geq 0 \nonumber \\
    \sigma_1(C) \geq \sigma_2(C) \geq \ldots \geq \sigma_r(C) \geq 0 \nonumber 
\end{align}
where $r = \min \{m,n\}$. Then we have
\begin{align}
    \sigma_{i+j-1}(B) \leq \sigma_i(A) + \sigma_j(C) \ \forall \ i \in \{1,2,\ldots r\},j \in \{1,2,\ldots r\},i+j-1 \in \{1,2,\ldots r\} .
\end{align}
Consequently, we have
\begin{equation} \label{eq:weyl1}
    \|B\| \leq \|A\| + \|C\| 
\end{equation}
and
\begin{equation} \label{eq:weyl2}
    \sigma_{\min}(B) \geq \sigma_{\min}(A) -\|C\| .
\end{equation}

\item[Vector Bernstein inequality] \hfill \\
Suppose $\hat{a}=\sum_{i} \hat{a}_{i},$ where $\hat{a}_{i}$ are independent random vectors of dimension $n$.
Let $\mathbb{E}[\hat{a}]=a$ , and the variance $\sigma^{2}=\mathbb{E}\left[\sum_{i}\left\|\hat{a}_{i}\right\|^{2}\right]$.
If every $\hat{a}_{i}$ has norm $\left\|\hat{a}_{i}\right\| \leq s$ then with high probability we have
\begin{align}
    \|\hat{a}-a\| \leq \mathcal{O}\left(s \log n+\sqrt{\sigma^{2} \log n}\right) .
\end{align}
This is the same inequality given in \cite{Fazel2018}. See \cite{Tropp2012} for the exact scale constants and a proof.

\end{description}

\section{Model-based policy gradient descent} \label{appendix:model_based_gd}

The proof of convergence using gradient descent proceeds by establishing several technical lemmas, bounding the infinite-horizon covariance $\Sigma_K$, then using that bound to limit the step size, and finally obtaining a one-step bound on gradient descent progress and applying it inductively at each successive step.
We begin with a bound on the induced operator norm of $\mathcal{T}_K$:
\begin{Lem} \label{lemma:TK_norm_bound} ($\mathcal{T}_K$ norm bound)
The following bound holds for any mean-square stabilizing $K$:
    \begin{align}
        \|\mathcal{T}_K\| \Let \underset{X}{\text{sup}} \frac{\|\mathcal{T}_K(X)\|}{\|X\|} \leq \frac{C(K)}{{\underline{\sigma} (\Sigma_0)} \underline{\sigma}(Q)} .
    \end{align}
\end{Lem}

\begin{proof}
    The proof follows that given in \cite{Fazel2018} using our definition of $\mathcal{T}_K$.
\end{proof}



\begin{Lem} \label{lemma:F_K_perturbation} ($\mathcal{F}_K$ perturbation)
    Consider a pair of mean-square stabilizing gain matrices $K$ and $K{^\prime}$.
    The following $\mathcal{F}_K$ perturbation bound holds:
    \begin{align} \label{eq:mathcal_F_bound}
       \| \mathcal{F}_{K{^\prime}} - \mathcal{F}_{K} \| 
       & \leq 
       2 \| A+BK \| \|B\|\|\Delta\| + h_B \|B\| \|\Delta\|^2 \\
       \text{where } \quad h_B &\Let {\|B\|}^{-1}\big(\|B\|^2+\sum_{j=1}^q \beta_j \|B_j\|^2 \big) . 
    \end{align}
\end{Lem}

\begin{proof}
    Let $\Delta^\prime = -\Delta$.
    For any matrix $X$ we have
    \begin{align}
        (\mathcal{F}_{K}-\mathcal{F}_{K{^\prime}})(X) 
        &= \underset{\delta,\gamma}{\mathbb{E}} \left[ \widetilde{A}_{K} X \widetilde{A}_{K}^\intercal - \widetilde{A}_{K{^\prime}} X \widetilde{A}_{K{^\prime}}^\intercal \right] \\
        &= \underset{\delta,\gamma}{\mathbb{E}} \left[ \widetilde{A}_{K} X (\widetilde{B} \Delta^\prime)^\intercal + (\widetilde{B} \Delta^\prime) X \widetilde{A}_{K}^\intercal - (\widetilde{B} \Delta^\prime) X (\widetilde{B} \Delta^\prime)^\intercal \right] \\
        &=  A_K X ({B} \Delta^\prime)^\intercal + ({B} \Delta^\prime) X A_K^\intercal - \underset{\gamma_{tj}}{\mathbb{E}} \left[ (\widetilde{B} \Delta^\prime) X (\widetilde{B} \Delta^\prime)^\intercal \right] \\
        &=  A_K X ({B} \Delta^\prime)^\intercal + ({B} \Delta^\prime) X A_K^\intercal 
        - (B \Delta^\prime) X (B \Delta^\prime)^\intercal - \sum_{j=1}^q \beta_j (B_j \Delta^\prime) X (B_j \Delta^\prime)^\intercal . \label{eq:mathcal_F_bound_last_step}
    \end{align}
    The operator norm $\| \mathcal{F}_{K{^\prime}} - \mathcal{F}_{K} \|$ is
    \begin{align}
        \| \mathcal{F}_{K{^\prime}} - \mathcal{F}_{K} \| = \| \mathcal{F}_{K} - \mathcal{F}_{K{^\prime}}\| = \underset{X}{\text{sup}} \ \frac{\|(\mathcal{F}_{K}-\mathcal{F}_{K{^\prime}})(X)\|}{\|X\|}
    \end{align}
    Applying submultiplicativity of spectral norm to \eqref{eq:mathcal_F_bound_last_step} and noting that $\|\Delta^\prime\|=\|\Delta\|$ completes the proof.
\end{proof}

\begin{Lem}[$\mathcal{T}_K$ perturbation] \label{lemma:T_K_perturbation} 
    If $K$ and $K{^\prime}$ are mean-square stabilizing and 
    $
        \| \mathcal{T}_{K} \| \|\mathcal{F}_{K{^\prime}} - \mathcal{F}_{K} \| \leq \frac{1}{2}
    $
    then
    \begin{align}
        \| (\mathcal{T}_{K{^\prime}}-\mathcal{T}_{K})(\Sigma) \| &\leq 2 \| \mathcal{T}_{K} \| \|\mathcal{F}_{K{^\prime}} - \mathcal{F}_{K} \| \| \mathcal{T}_{K}(\Sigma) \| \\
                                                            &\leq 2 \| \mathcal{T}_{K} \|^2 \|\mathcal{F}_{K{^\prime}} - \mathcal{F}_{K} \| \|\Sigma\|  .
    \end{align}    
\end{Lem}

\begin{proof}
The proof follows \cite{Fazel2018} using our modified definitions of $\mathcal{T}_K$ and $\mathcal{F}_K$.
\end{proof}

\begin{Lem}[$\Sigma_K$ trace bound] \label{lemma:S_K_trace_bound} 
    If $\rho(\mathcal{F}_K) < 1$ then 
    \begin{align}
        \Tr \left(\Sigma_{K}\right) \geq \frac{{\underline{\sigma} (\Sigma_0)}}{1-\rho(\mathcal{F}_K)} .
    \end{align}    
\end{Lem}

\begin{proof}
    We have by \eqref{eq:TK_char} that
    \begin{align} \label{eq:sk_trace_bound1}
        \Tr (\Sigma_{K})    &= \Tr (\mathcal{T}_K(\Sigma_0)) = \sum_{t=0}^\infty \Tr ({\mathcal{F}^t_K(\Sigma_0)}) .
    \end{align}
    Since $\Sigma_0 \succeq {\underline{\sigma} (\Sigma_0)} I$ we know the $t^{th}$ term satisfies the inequality
    $
        {\mathcal{F}^t_K(\Sigma_0)} \geq {\underline{\sigma} (\Sigma_0)} {\mathcal{F}^t_K(I)} ,
    $
    so we have
    \begin{align} \label{eq:sk_trace_bound2}
        \Tr (\Sigma_{K}) \geq {\underline{\sigma} (\Sigma_0)} \sum_{t=0}^\infty \Tr ({\mathcal{F}^t_K(I)}) .
    \end{align}
    We have a generic inequality for a sum of $n$ matrices $M_i$:
    \begin{align}
        \Tr \left[ \sum_i^n M_i M_i^\intercal \right]   = \sum_i^n \Tr \left[ M_i M_i^\intercal \right] = \sum_i^n \left\| M_i \right\|_F^2 \label{eq:sum_kron_fro}
                                                        = \sum_i^n \left\| M_i \otimes M_i \right\|_F \geq  \left\| \sum_i^n M_i \otimes M_i \right\|_F \nonumber
    \end{align}
    where the last step is due to the triangle inequality.
    Recalling the definitions of $\mathcal{F}^t_K(I)$ and $\mathcal{F}^t_K$ we see they are of the form of the LHS and RHS in \eqref{eq:sum_kron_fro} with all terms matched between $\mathcal{F}^t_K(I)$ and $\mathcal{F}^t_K$ so that the inequality in \eqref{eq:sum_kron_fro} holds; this can be seen by starting with $t=1$ and incrementing $t$ up by $1$ which will give $(1+p+q)^t$ terms which are all matched.
    Thus,
    \begin{align}
        \Tr [\mathcal{F}^t_K(I)] \geq \|\mathcal{F}^t_K\|_F \geq \rho(\mathcal{F}_K)^t .
    \end{align}
    Continuing from \eqref{eq:sk_trace_bound2} we have
    \begin{align}
        \Tr (\Sigma_{K}) \geq {\underline{\sigma} (\Sigma_0)} \sum_{t=0}^\infty \rho(\mathcal{F}_K)^t .
    \end{align}
    By hypothesis $\rho(\mathcal{F}_K) < 1$, and taking the sum of the geometric series completes the proof.
\end{proof}

\begin{Lem}[$\Sigma_K$ perturbation] \label{lemma:Sigma_K_perturbation} 
    If $K$ is mean-square stabilizing and
    $\|\Delta\| \leq h_\Delta(K)$
    where $h_\Delta(K)$ is the polynomial
    \begin{align} 
        h_\Delta(K)    & \Let \frac{\underline{\sigma}(Q) {\underline{\sigma} (\Sigma_0)}}{4 h_B C(K)  \left(\|A_{K}\|+1\right)},
    \end{align}
    then the associated state covariance matrices satisfy
    \begin{align*} 
        \|\Sigma_{K^\prime} \! - \! \Sigma_{K}\|   
        \leq 4 \left( \frac{C(K)}{\underline{\sigma}(Q)} \right)^2  \frac{\|B\| (\|A_{K}\|+1)}{{\underline{\sigma} (\Sigma_0)}} \|\Delta\|
        \leq \frac{C(K)}{\underline{\sigma}(Q)} .
    \end{align*}
\end{Lem}

\begin{proof}
    First we show that $K$ is mean-square stabilizing and $ \|\Delta\| \leq h_\Delta(K) $ then $K{^\prime}$ is also mean-square stabilizing. This follows from an analogous argument in \cite{Fazel2018} by characterizing mean-square stability in terms of $\rho(\mathcal{F}_K)$ rather than $\rho(A_K)$ and using Lemma \ref{lemma:S_K_trace_bound}.
    Let $K^{\prime\prime}$ be distinct from $K$ with $\rho(\mathcal{F}_{K^{\prime\prime}}) < 1$ and $ \|K^{\prime\prime}- K\| \leq h_\Delta $.
    We have
    \begin{align}
        |\Tr(\Sigma_{K^{\prime\prime}} - \Sigma_{K} )| \leq n \|\Sigma_{K^{\prime\prime}} - \Sigma_{K}\| .
    \end{align}
    Since $K$ and $K^{\prime\prime}$ are mean-square stabilizing Lemma \ref{lemma:Sigma_K_perturbation} holds so we have
    \begin{align}
        |\Tr(\Sigma_{K^{\prime\prime}} - \Sigma_{K} )| \leq n  \frac{C(K)}{\sigma_{\min}(Q)} 
        \Rightarrow
        \Tr(\Sigma_{K^{\prime\prime}}) \leq \Gamma ,
    \end{align}
    where we define
    \begin{align}
        \Gamma \Let \Tr(\Sigma_{K}) + n \frac{C(K)}{\sigma_{\min}(Q)} , \quad \epsilon = \frac{\sigma_{\min} (\Sigma_0)}{2\Gamma} .
    \end{align}
    Using Lemma \ref{lemma:S_K_trace_bound} we have
    \begin{align}
        \Tr \left(\Sigma_{K}\right) \geq \frac{{\sigma_{\min} (\Sigma_0)}}{1-\rho(\mathcal{F}_K)} = \frac{{2 \epsilon \Gamma}}{1-\rho(\mathcal{F}_K)} 
    \end{align}
    Rearranging and substituting for $\Gamma$,
    \begin{align}
        \rho(\mathcal{F}_K)     \leq 1 - \frac{2 \epsilon \Gamma} {\Tr (\Sigma_{K})} &= 1 - \frac{2 \epsilon \left(\Tr(\Sigma_{K}) + n \frac{C(K)}{\sigma_{\min}(Q)}\right)} {\Tr (\Sigma_{K})} \nonumber \\
                                & = 1 - 2 \epsilon \left( 1 +  \frac{n C(K)} {\sigma_{\min}(Q) \Tr (\Sigma_{K})} \right) \\
                                & < 1-2\epsilon < 1-\epsilon .
    \end{align}
    Now we construct the proof by contradiction. Suppose there is a $K^\prime$ with $\rho(\mathcal{F}_{K^{\prime}}) > 1$ satisfying the perturbation restriction
    \begin{align}
        \|K^\prime-K\| \leq h_\Delta .
    \end{align}
    Since spectral radius is a continuous function (see \cite{Tyrtyshnikov2012}) there must be a point $K^{\prime\prime\prime}$ on the path between $K$ and $K^\prime$ such that $\rho(\mathcal{F}_{K^{\prime\prime\prime}}) = 1-\epsilon < 1$. Since $K$ and $K^{\prime\prime\prime}$ are mean-square stabilizing Lemma \ref{lemma:Sigma_K_perturbation} holds so we have 
    \begin{align}
        |\Tr(\Sigma_{K^{\prime\prime\prime}} - \Sigma_{K} )| \leq n  \frac{C(K)}{\sigma_{\min}(Q)},
    \end{align}
    and rearranging
    \begin{align}
        \Tr(\Sigma_{K^{\prime\prime\prime}}) \leq \Tr(\Sigma_{K}) + n \frac{C(K)}{\sigma_{\min}(Q)} = \Gamma .
    \end{align}
    However since $K^{\prime\prime\prime}$ is mean-square stabilizing Lemma \ref{lemma:S_K_trace_bound} holds so we have
    \begin{align}
        \Tr \left(\Sigma_{K^{\prime\prime\prime}}\right) \geq \frac{{\sigma_{\min} (\Sigma_0)}}{1-\rho(\mathcal{F}_{K^{\prime\prime\prime}})} = \frac{{2 \epsilon \Gamma}}{1-(1-\epsilon)} = 2 \Gamma
    \end{align}
    which is a contradiction. Therefore no such mean-square unstable $K^\prime$ satisfying the hypothesized perturbation restriction can exist, completing the first part of the proof.
    
    The rest of the proof follows \cite{Fazel2018} by using the condition on $\|\Delta\|$, ${\|\Sigma_K\| \geq {\underline{\sigma} (\Sigma_0)}}$, and Lemmas \ref{lemma:cost_bounds}, \ref{lemma:TK_norm_bound}, and \ref{lemma:T_K_perturbation}.
    The condition on $\|\Delta\|$ directly implies
    \begin{align}
        h_B \|\Delta\| &\leq \frac{\underline{\sigma}(Q) {\underline{\sigma} (\Sigma_0)}}{4C(K) \left(\|A_K\|+1\right)} \leq \frac{\underline{\sigma}(Q) {\underline{\sigma} (\Sigma_0)}}{4C(K)} \leq \frac{1}{4} . 
    \end{align}
    where the last step is due to the combination of Lemma \ref{lemma:cost_bounds} with ${\|\Sigma_K\| \geq {\underline{\sigma} (\Sigma_0)}}$.
    By Lemma \ref{lemma:F_K_perturbation} we have
    \begin{align}
        \| \mathcal{F}_{K{^\prime}} - \mathcal{F}_{K} \|    & \leq 2  \|B\|\|\Delta\| \big(  \| A+BK \| + \|\Delta\| h_B/2   \big) \\
                                                            & \leq 2  \|B\|\|\Delta\| \big( \| A+BK \| + 1/8 \big) \\
                                                            & \leq 2  \|B\|\|\Delta\| \big( \| A+BK \| + 1 \big) .
    \end{align}
    Combining this with Lemma \ref{lemma:TK_norm_bound} we have
    \begin{align}
      \|T_{K}\| \| \mathcal{F}_{K{^\prime}} - \mathcal{F}_{K} \| 
      \leq \left(\frac{C(K)}{{\underline{\sigma} (\Sigma_0)} \underline{\sigma}(Q)}\right) \left(2  \|B\|\|\Delta\| \left( \| A_{K} \| + 1 \right)\right) . \label{eq:Sigma_K_perturbation_intstep2}
    \end{align}
    By the condition on $\|\Delta\|$ we have
    \begin{align}
        \|\Delta\|  &\leq \frac{\underline{\sigma}(Q) {\underline{\sigma} (\Sigma_0)}}{4C(K) \|B\| \left(\|A_{K}\|+1\right)} ,
    \end{align}
    so 
    \begin{align}
      \|T_{K}\| \| \mathcal{F}_{K{^\prime}} - \mathcal{F}_{K} \| &\leq \frac{1}{2}
    \end{align}
    which allows us to use Lemma \ref{lemma:T_K_perturbation} by which we have
    \begin{align*}
        \| (\mathcal{T}_{K}-\mathcal{T}_{K})(\Sigma_0) \| \leq 2 \| \mathcal{T}_{K} \| \|\mathcal{F}_{K{^\prime}} - \mathcal{F}_{K} \| \| \mathcal{T}_{K}(\Sigma_0) \| 
        \leq \left(\frac{2 C(K)}{{\underline{\sigma} (\Sigma_0)} \underline{\sigma}(Q)}\right) \left(2  \|B\|\|\Delta\| \left( \| A_{K} \| + 1 \right)\right) \| \mathcal{T}_{K}(\Sigma_0) \| 
    \end{align*}
    where the last step used \eqref{eq:Sigma_K_perturbation_intstep2}.
    Using $\mathcal{T}_{K}(\Sigma_0)=\Sigma_K$ gives
    \begin{align*}
        \| \Sigma_{K{^\prime}}-\Sigma_{K} \| 
        \leq  \left(\frac{2 C(K)}{{\underline{\sigma} (\Sigma_0)} \underline{\sigma}(Q)}\right) \left(2  \|B\|\|\Delta\| \left( \| A_{K} \| + 1 \right)\right) \| (\Sigma_{K}) \| .
    \end{align*}
    Using Lemma \ref{lemma:cost_bounds} completes the proof.
\end{proof}

Now we bound the one step progress of policy gradient where we allow the step size to depend explicitly on the current gain matrix iterate $K_s$.
\begin{Lem}[Gradient descent, one-step] \label{lemma:grad_exact_one_step} 
    Using the policy gradient step update
    $
        K_{s+1} = K_s - \eta \nabla C(K_s)
    $
    with step size 
    \begin{align*}
        0 < \eta \leq \frac{1}{16}\min \Bigg\{ \frac{ \left( \frac{\underline{\sigma}(Q) {\underline{\sigma} (\Sigma_0)}}{C(K)} \right)^2 }{ h_B \|\nabla C(K)\| (\|A_K\|+1)} , \frac{\underline{\sigma}(Q)}{C(K) \|R_K\|} \Bigg\}
    \end{align*}
    gives the one step progress bound
    \begin{align}
      \frac{ C(K_{s+1}) - C(K^*)}{ C(K_{s}) - C(K^*)} \leq  1 - 2 \eta \frac{ \underline{\sigma}(R) {\underline{\sigma} (\Sigma_0)}^2}{\|\Sigma_{K^*}\|} .
    \end{align}
\end{Lem}

\begin{proof}
The gradient update yields
$
    \Delta =  -2\eta E_{K_s} \Sigma_{K_s} .
$
Putting this into Lemma \ref{lemma:almost_smooth} gives
\begin{align*}       
    & C(K_{s+1}) - C(K_s) \nonumber \\
    & = 2\Tr\big[ \Sigma_{K_{s+1}} \Delta^\intercal E_{K_s} \big] + \Tr\big[ \Sigma_{K_{s+1}} \Delta^\intercal R_{K_s}\Delta \big] \\
    & = -4\eta\Tr\big[ \Sigma_{K_{s+1}} \Sigma_{K_s} E_{K_s}^\intercal E_{K_s} \big]  + 4\eta^2 \Tr\big[ \Sigma_{K_{s}} \Sigma_{K_{s+1}}  \Sigma_{K_s} E_{K_s}^\intercal R_{K_s} E_{K_s} \big] \\
    & = -4\eta\Tr\big[ \Sigma_{K_{s}} \Sigma_{K_s} E_{K_s}^\intercal E_{K_s} \big] 
    + 4\eta\Tr\big[ (-\Sigma_{K_{s+1}}+\Sigma_{K_{s}}) \Sigma_{K_s} E_{K_s}^\intercal E_{K_s} \big] \\
    &\quad + 4\eta^2 \Tr\big[ \Sigma_{K_{s}} \Sigma_{K_{s+1}}  \Sigma_{K_s} E_{K_s}^\intercal R_{K_s} E_{K_s} \big] \\
    & \leq -4\eta\Tr\big[ \Sigma_{K_{s}} \Sigma_{K_s} E_{K_s}^\intercal E_{K_s} \big] 
    +4 \eta \|\Sigma_{K_{s+1}}-\Sigma_{K_{s}}\| \Tr\big[ \Sigma_{K_s} E_{K_s}^\intercal E_{K_s} \big] \\
    & \quad + 4\eta^2 \|\Sigma_{K_{s+1}} \| \|R_{K_s} \| \Tr\big[ \Sigma_{K_{s}} \Sigma_{K_s} E_{K_s}^\intercal  E_{K_s} \big] \\
    & \leq -4\eta\Tr\big[ \Sigma_{K_{s}} \Sigma_{K_s} E_{K_s}^\intercal E_{K_s} \big]
    +4 \eta \frac{\|\Sigma_{K_{s+1}}-\Sigma_{K_{s}}\|}{\underline{\sigma}(\Sigma_{K_{s}})} \Tr\big[ \Sigma_{K_s}^\intercal E_{K_s}^\intercal E_{K_s} \Sigma_{K_s} \big] \\
    & \quad + 4\eta^2 \|\Sigma_{K_{s+1}} \| \|R_{K_s} \| \Tr\big[ \Sigma_{K_{s}} \Sigma_{K_s} E_{K_s}^\intercal  E_{K_s} \big] \\
    & = -4\eta \bigg( 1 - \frac{\|\Sigma_{K_{s+1}}-\Sigma_{K_{s}}\|}{\underline{\sigma}(\Sigma_{K_{s}})} - \eta \|\Sigma_{K_{s+1}} \| \|R_{K_s} \| \bigg) 
    \times \Tr\big[ \Sigma_{K_{s}}^\intercal  E_{K_s}^\intercal E_{K_s} \Sigma_{K_s} \big] \\
    & = -\eta \bigg( 1 - \frac{\|\Sigma_{K_{s+1}}-\Sigma_{K_{s}}\|}{\underline{\sigma}(\Sigma_{K_{s}})} - \eta \|\Sigma_{K_{s+1}} \| \|R_{K_s} \| \bigg) 
    \times \Tr\big[ \nabla C(K_{s})^\intercal \nabla C(K_{s}) \big] \\
    & \leq -\eta \bigg( 1 - \frac{\|\Sigma_{K_{s+1}}-\Sigma_{K_{s}}\|}{{\underline{\sigma} (\Sigma_0)}} - \eta \|\Sigma_{K_{s+1}} \| \|R_{K_s} \| \bigg) 
    \times \Tr\big[ \nabla C(K_{s})^\intercal \nabla C(K_{s}) \big] \\
    & \leq -\eta \bigg( 1 - \frac{\|\Sigma_{K_{s+1}}-\Sigma_{K_{s}}\|}{{\underline{\sigma} (\Sigma_0)}} - \eta \|\Sigma_{K_{s+1}} \| \|R_{K_s} \| \bigg)
    \times 4 \frac{\underline{\sigma}(R) {\underline{\sigma} (\Sigma_0)}^2 }{\|\Sigma_{K^*}\|}  (C(K_s)-C(K^*)) .
\end{align*}
where the last step is due to $\underline{\sigma} (\Sigma_0) \leq \underline{\sigma}(\Sigma_{K_{s}})$ and Lemma \ref{lemma:gradient_dominated}.
Note that the assumed condition on the step size ensures the gain matrix difference satisfies the condition for Lemma \ref{lemma:Sigma_K_perturbation} as follows:
\begin{align}
   \|\Delta\| = \eta \|\nabla C(K_s)\| 
                &\leq  \frac{1}{16} \left( \frac{\underline{\sigma}(Q) {\underline{\sigma} (\Sigma_0)}}{C(K_s)} \right)^2 \frac{\|\nabla C(K_s)\|}{h_B \|\nabla C(K_s)\| (\|A_{K_s}\|+1)}  \\
                &\leq \frac{1}{4} \left( \frac{\underline{\sigma}(Q) {\underline{\sigma} (\Sigma_0)}}{C(K_s)} \right)^2 \frac{1}{h_B  (\|A_{K_s}\|+1)}  
                \leq h_\Delta(K)
\end{align}
where the last inequality is due to Lemma \ref{lemma:cost_bounds}.
Thus we can indeed apply Lemma \ref{lemma:Sigma_K_perturbation}, by which we have
\begin{align*}
    \frac{\|\Sigma_{K_{s+1}}-\Sigma_{K_s}\|}{{\underline{\sigma} (\Sigma_0)}}  
    \leq  \frac{4 C(K_{s})^2}{\underline{\sigma}(Q)^2 {\underline{\sigma} (\Sigma_0)^2}}  \|B\| (\|A_{K_{s}}\|+1) \|\Delta\| \leq \frac{1}{4} 
\end{align*}
where the last inequality is due to using the substitution ${\|\Delta\|  = \eta \|\nabla C(K_s)\|}$ and the hypothesized condition on $\eta$.
Using this and Lemma \ref{lemma:cost_bounds} we have
\begin{align*}
    \|\Sigma_{K_{s+1}}\|    \leq \|\Sigma_{K_{s+1}}-\Sigma_{K_s}\|+\|\Sigma_{K_s}\| 
                            \leq \frac{{\underline{\sigma} (\Sigma_0)}}{4}+\frac{C(K_s)}{\underline{\sigma}(Q)} \leq \frac{\|\Sigma_{K_{s+1}}\|}{4}+\frac{C(K_s)}{\underline{\sigma}(Q)} .
\end{align*}
Solving for $\|\Sigma_{K_{s+1}}\|$ gives
$
    \|\Sigma_{K_{s+1}}\| \leq \frac{4}{3} \frac{C(K_s)}{\underline{\sigma}(Q)} ,
$
so
\begin{align}
    1 - \frac{\|\Sigma_{K_{s+1}}-\Sigma_{K_{s}}\|}{{\underline{\sigma} (\Sigma_0)}} - \eta \|\Sigma_{K_{s+1}} \| \|R_{K_s} \|   
    \geq 1 - \frac{1}{4} - \eta \frac{4}{3} \frac{C(K_s)}{\underline{\sigma}(Q)} \|R_{K_s}\| 
    \geq 1 - \frac{1}{4} - \frac{4}{3} \cdot \frac{1}{16} = \frac{2}{3} \geq \frac{1}{2}
\end{align}
where the second-to-last inequality used the hypothesized condition on $\eta$.
Therefore
\begin{align}
    \frac{C(K_{s+1}) - C(K_s)}{C(K_s)-C(K^*)}
    & \leq - 2 \eta \frac{\underline{\sigma}(R) {\underline{\sigma} (\Sigma_0)}^2 }{\|\Sigma_{K^*}\|} .
\end{align}
Adding 1 to both sides completes the proof.
\end{proof}

\begin{Lem}[Cost difference lower bound] \label{lemma:cost_difference_lower_bound} 
    The following cost difference inequality holds:
    \begin{align}
        C(K)-C(K^*) \geq \frac{{\underline{\sigma} (\Sigma_0)}}{\|R_K\|} \Tr(E_K^\intercal E_K) .
    \end{align}    
\end{Lem}

\begin{proof}
    The proof follows that for an analogous condition located in the gradient domination lemma in \cite{Fazel2018}.
    Let $K$ and $K{^\prime}$ generate the (stochastic) state and action sequences 
    $ \{x_t\}_{K,x} \ , \{u_t\}_{K,x} $ and $ \{x_t\}_{K{^\prime},x} \ , \{u_t\}_{K{^\prime},x}$
    respectively. 
    By definition of the optimal gains we have $C(K^*) \leq C(K{^\prime})$.
    Then by Lemma \ref{lemma:value_difference} we have
    \begin{align}
        C(K) - C(K^*)   &\geq C(K) - C(K{^\prime}) \\
                        &= -\underset{x_0}{\mathbb{E}} \bigg[ \sum_{t=0}^\infty \mathcal{A}_{K}\big(\{x_t\}_{K{^\prime},x} , \{u_t\}_{K{^\prime},x}\big)\bigg] .
    \end{align}
    Choose $K^\prime$ such that $\Delta = K{^\prime}-K = - R_{K}^{-1} E_{K}$ so that \eqref{eq:grad_dom_ineq1} from Lemma \ref{lemma:gradient_dominated} holds with equality as
    \begin{align}
        \mathcal{A}_{K}(x,K{^\prime} x) = -\Tr\big[xx^\intercal E_{K}^\intercal R_{K}^{-1} E_{K} \big] .
    \end{align}
    Thus we have
    \begin{align*}
        C(K) - C(K^*)  
        & \geq \underset{x_0}{\mathbb{E}} \bigg[ \sum_{t=0}^\infty \Tr\bigg(\{x_t\}_{K{^\prime},x} \{x_t\}_{K{^\prime},x}^\intercal E_{K}^\intercal R_{K}^{-1} E_{K} \bigg)\bigg] \\
        & = \Tr(\Sigma_{K{^\prime}} E_{K}^\intercal R_{K}^{-1} E_{K}) \\
        &\geq \frac{{\underline{\sigma} (\Sigma_0)}}{\|R_K\|} \Tr(E_{K}^\intercal E_{K}) .
    \end{align*}
\end{proof}

\begin{Lem} \label{lemma:gradient_gain_bounds}
    The following inequalities hold:
    \begin{align*}
        \|\nabla C(K)\| \leq \|\nabla C(K)\|_F \leq h_1(K)  \quad \text{ and } \quad
        \|K\| \leq h_2(K) .
    \end{align*}
    where $h_0(K)$, $h_1(K)$, $h_2(K)$ are the polynomials
    \begin{align*}
        h_0(K)    & \Let \sqrt{\frac{\|R_{K}\|(C(K)-C(K^*))}{{\underline{\sigma} (\Sigma_0)}}} , \\
        h_1(K)    & \Let 2\frac{C(K) h_0(K)}{\underline{\sigma}(Q)}  , \quad
        h_2(K)      \Let \frac{h_0(K) + \|B^\intercal P_{K} A\|}{\underline{\sigma}(R)} .
    \end{align*}
\end{Lem}

\begin{proof}
    The proof follows \cite{Fazel2018} with $R_K$ defined here.
    From the policy gradient expression we have
    \begin{align*}
        \|\nabla C(K)\|_F^2 = \|2 E_K \Sigma_K\|_F^2 & = 4\Tr(\Sigma_K^\intercal E_K^\intercal E_K \Sigma_K) \\
        & \leq 4 \|\Sigma_K\|^2 \Tr(E_K^\intercal E_K) .
    \end{align*}
    Using Lemma \ref{lemma:cost_bounds} we have
    \begin{align}
        \|\nabla C(K)\|^2 \leq 4 \left(\frac{C(K)}{\underline{\sigma}(Q)}\right)^2 \Tr(E_K^\intercal E_K),
    \end{align}
    and using Lemma \ref{lemma:cost_difference_lower_bound} we have
    \begin{align}
        \Tr(E_K^\intercal E_K) \leq \frac{\|R_K\| (C(K) - C(K^*))}{{\underline{\sigma} (\Sigma_0)}},
    \end{align}
    so
    \begin{align}
        \|\nabla C(K)\|_F^2 \leq 4 \left(\frac{C(K)}{\underline{\sigma}(Q)}\right)^2 \frac{\|R_K\| (C(K) - C(K^*))}{{\underline{\sigma} (\Sigma_0)}}
    \end{align}
    Taking square roots completes the proof of the first part of the lemma. \\
    For the second part of the lemma, we have
    \begin{align}
        \|K\|   \leq \|R_K^{-1}\|\|R_K K\| &\leq \frac{1}{\underline{\sigma}(R)}\|R_K K\| \\
                & \leq \frac{1}{\underline{\sigma}(R)}(\|R_K K +B^\intercal P_K A\|+\|B^\intercal P_K A\|) \\
                & \leq \frac{1}{\underline{\sigma}(R)}\left(\sqrt{\Tr(E_K^\intercal E_K)}+\|B^\intercal P_K A\|\right)
    \end{align}
    where the second line is due to Weyl's inequality for singular values \cite{Fazel2018}. Using Lemma \ref{lemma:cost_difference_lower_bound} again on the $\Tr(E_K^\intercal E_K)$ term completes the proof.
\end{proof}

We now give the parameter and proof of global convergence of policy gradient descent in Theorem \ref{thm:grad_exact_convergence}.
\begin{Thm}[Policy gradient convergence] \label{thm:grad_exact_convergence_restated} 
    Consider the assumptions and notations of Theorem \ref{thm:grad_exact_convergence} and define
    \begin{align*}
        c_{\text{pg}} & \Let \frac{1}{16}\min \Bigg\{ \frac{\left( \frac{\underline{\sigma}(Q) {\underline{\sigma} (\Sigma_0)}}{C(K_0)} \right)^2}{h_B \overline{h_{1}} (\|A\| \! + \! \overline{h_{2}} \|B\| \! + \! 1)} , \frac{\underline{\sigma}(Q)}{C(K_0) \overline{\|R_{K}\|} } \Bigg\} \\
        \overline{h_{1}} & \Let \text{ max}_K \ h_1(K) \text{ subject to } C(K) \leq C(K_0) , \\
        \overline{h_{2}} & \Let \text{ max}_K \ h_2(K) \text{ subject to } C(K) \leq C(K_0) , \\
        \overline{\|R_{K}\|} & \Let \text{ max}_K \ \| R_K \| \text{ subject to } C(K) \leq C(K_0) .
    \end{align*}
    Then the claim of Theorem \ref{thm:grad_exact_convergence} holds.
\end{Thm}

\begin{proof}
    We have by Weyl's inequality for singular values \eqref{eq:weyl1}, submultiplicativity of spectral norm, and Lemma \ref{lemma:gradient_gain_bounds} that
    \begin{align}
        \|B\| \|\nabla C(K)\| (\|A+BK\|+1) 
        & \leq \|B\| \|\nabla C(K)\| (\|A\|+\|B\|\|K\|+1) \\
        & \leq \|B\| h_1(K) (\|A\| + \|B\|h_2(K) +1)
    \end{align}
    Thus by choosing $0 < \eta \leq c_{pg}$ we satisfy the requirements for Lemma \ref{lemma:grad_exact_one_step} at $s=1$, which implies that progress is made at $s=1$, i.e., that $C(K_1)\leq C(K_0)$ according to the rate in Lemma \ref{lemma:grad_exact_one_step}. Proceeding inductively and applying Lemma \ref{lemma:grad_exact_one_step} at each step completes the proof.
\end{proof}

\begin{Rem} \label{rem:param_bounds}
The quantities $\overline{h_{1}}$, $\overline{h_{2}}$, and $\overline{\|R_{K}\|}$ may be upper bounded by quantities that depend only on problem data and $C(K_0)$ e.g. using the cost bounds in Lemma \ref{lemma:cost_bounds}, which we omit for brevity, so a conservative minimum step size $\eta$ may be computed exactly.
\end{Rem}

\section{Model-free policy gradient descent} \label{appendix:model_free_gd}

The overall proof technique of showing high probability convergence of model-free policy gradient proceeds by showing that the difference between estimated gradient and true gradient is bounded by a sufficiently small value that the iterative descent progress remains sufficiently large to maintain a linear rate. We use a matrix Bernstein inequality to ensure that enough sample trajectories are collected to estimate the gradient to high enough accuracy.

We begin with a lemma shows that $C(K)$ and $\Sigma_K$ can be estimated with arbitrarily high accuracy as the rollout length $\ell$ increases. 
\begin{Lem}[Approximating \texorpdfstring{$C(K)$}{CK} and \texorpdfstring{$\Sigma_K$}{SK} with infinitely many finite horizon rollouts] \label{lemma:finite_horizon} 
    Let $\epsilon$ be an arbitrary small constant.
    Suppose $K$ gives finite $C(K)$. Define the finite-horizon estimates 
    \begin{align*}
        \Sigma_{K}^{(\ell)}     \Let \mathbb{E}\left[\sum_{i=0}^{\ell-1} x_{i} x_{i}^\intercal\right], \
        C^{(\ell)}(K)           \Let \mathbb{E}\left[\sum_{i=0}^{\ell-1} x_{i}^\intercal Q x_{i}+u_{i}^\intercal R u_{i} \right] 
    \end{align*}
    where expectation is with respect to $x_0,\{\delta_{ti}\}, \{\gamma_{tj}\}$. 
    Then the following hold:
    \begin{align*}
        \ell \geq \overline{h_{\ell}} (\epsilon) & \Let \frac{n \cdot C^{2}(K)}{\epsilon {\underline{\sigma} (\Sigma_0)} \underline{\sigma}^{2}(Q)} 
        \quad \Rightarrow \quad
        \|\Sigma_{K}^{(\ell)}-\Sigma_{K}\| \leq \epsilon \\
        \ell \geq h_\ell(\epsilon) & \Let \overline{h_{\ell}} (\epsilon) \| Q_K \|
        \qquad \Rightarrow \quad
        | C^{(\ell)}(K) - C(K) | \leq \epsilon.
    \end{align*}
\end{Lem}
\begin{proof}
    The proof follows \cite{Fazel2018} exactly using suitably modified definitions of $C(K)$, $\mathcal{T}_K$, $\mathcal{F}_K$.
\end{proof}

Next we bound cost and gradient perturbations in terms of gain matrix perturbations and problem data.
Using the same restriction as in Lemma \ref{lemma:Sigma_K_perturbation} we have Lemmas \ref{lemma:C_K_perturbation} and \ref{lemma:nabla_K_C_K_perturbation}.
\begin{Lem}[$C(K)$ perturbation] \label{lemma:C_K_perturbation} 
    If
    $
        \|\Delta\| \leq h_\Delta(K)  ,
    $
    then the cost difference is bounded as
    \begin{align}
        |C({K{^\prime}})-C({K})| &\leq h_{\text{cost}}(K) C(K) \|\Delta\| 
    \end{align}
    where $h_{\text{cost}}(K)$ is the polynomial
    \begin{align*}
        h_{\text{cost}}(K) & \Let \frac{4 \Tr (\Sigma_0) \|R\|}{\underline{\sigma} (\Sigma_0) \underline{\sigma}(Q)} \Big( \| K \| + \frac{h_\Delta(K)}{2} + \|B\| \|K\|^2  ( \|A_{K}\| + 1) \frac{C(K)}{\underline{\sigma} (\Sigma_0) \underline{\sigma}(Q)} \Big)
    \end{align*}
\end{Lem}
\begin{proof}
The proof follows \cite{Fazel2018} using suitably modified definitions of $C(K)$, $\mathcal{T}_K$, $\mathcal{F}_K$, however compared with \cite{Fazel2018} we terminate the proof bound earlier so as to avoid a degenerate bound in the case of $K = 0$, and we also correct typographical errors.
Note that $\|\Delta\|$ has a more restrictive upper bound due to the multiplicative noise.
\end{proof}

\begin{Lem}[$\nabla C(K)$ perturbation] \label{lemma:nabla_K_C_K_perturbation} 
    If
    $
        \|\Delta\| \leq  h_\Delta(K) ,
    $
    then the policy gradient difference is bounded as
    \begin{align*}
        \|\nabla C({K{^\prime}})-\nabla C({K})\|   & \leq  h_{\text{grad}}(K) \|\Delta\| ,  \\
        \text{and } 
        \|\nabla C({K{^\prime}})-\nabla C({K})\|_F & \leq  h_{\text{grad}}(K) \|\Delta\|_F ,  
    \end{align*}
    where
    \begin{align*}        
        h_{\text{grad}}(K) & \Let 
        4 \left(\frac{C(K)}{\underline{\sigma}(Q)}\right) \bigg[ \|R\| + \|B\| \left( \|A\| + h_B (\|K\|+h_\Delta(K)) \right)   
        \left(\frac{h_{\text{cost}}(K) C(K)}{\Tr (\Sigma_0)} \right) + h_B \|B\| \left(\frac{C(K)}{{\underline{\sigma} (\Sigma_0)}}\right) \bigg] \\
        & \quad + 8 \left( \frac{C(K)}{\underline{\sigma}(Q)} \right)^2 \left( \frac{\|B\| (\|A_{K}\|+1)}{{\underline{\sigma} (\Sigma_0)}} \right)  h_0(K)  .
    \end{align*}
\end{Lem}

\begin{proof}
The proof generally follows \cite{Fazel2018} using Lemmas \ref{lemma:Sigma_K_perturbation}, \ref{lemma:C_K_perturbation}, and \ref{lemma:cost_difference_lower_bound} with $R_K$ and $E_K$ modified appropriately.
Recalling $\nabla C(K)=2 E_{K} \Sigma_{K}$, and using the triangle inequality,
\begin{align} \label{eq:int2}
    \|\nabla C\left(K^{\prime}\right)-\nabla C(K) \|  
    &= 
    \| 2 E_{K^{\prime}} \Sigma_{K^{\prime}}-2 E_{K} \Sigma_{K} \| \nonumber \\
    & 
    \leq 2 \|\left(E_{K^{\prime}}-E_{K}\right) \Sigma_{K^{\prime}} \| + 2 \| E_{K}\left(\Sigma_{K^{\prime}}-\Sigma_{K}\right) \| .
\end{align}
First we bound the second term of \eqref{eq:int2}.
By Lemma \ref{lemma:cost_difference_lower_bound}
\begin{align}
    \|E_K\| \leq \|E_K\|_F = \sqrt{\Tr(E_K^\intercal E_K)} \leq  h_0(K) .
\end{align}    
Since $\|\Delta\| \leq h_\Delta(K)$ Lemma \ref{lemma:Sigma_K_perturbation} holds and we have
\begin{align*}
    \|\Sigma_{K^\prime}-\Sigma_{K}\| \leq 4 \left( \frac{C(K)}{\underline{\sigma}(Q)} \right)^2  \frac{\|B\| (\|A_{K}\|+1)}{{\underline{\sigma} (\Sigma_0)}} \|\Delta\|  .
\end{align*}
Therefore the second term is bounded as
\begin{align}
    2 \| E_{K}\left(\Sigma_{K^{\prime}}-\Sigma_{K}\right)\| & \leq 2  \|E_{K}\| \|\Sigma_{K^{\prime}}-\Sigma_{K}\| \nonumber \\
    &\leq 8 \left( \frac{C(K)}{\underline{\sigma}(Q)} \right)^2 \left( \frac{\|B\| (\|A_{K}\|+1)}{{\underline{\sigma} (\Sigma_0)}} \right)  h_0(K) \|\Delta\| .
\end{align}
Now we bound the first term of \eqref{eq:int2}.
Since $\|\Delta\| \leq h_\Delta(K)$ Lemma \ref{lemma:Sigma_K_perturbation} holds and by the reverse triangle inequality
\begin{align}
    \|\Sigma_{K^\prime}\|-\|\Sigma_{K}\| \leq \|\Sigma_{K^\prime}-\Sigma_{K}\| \leq \frac{C(K)}{\underline{\sigma}(Q)} .
\end{align}
Rearranging and using Lemma \ref{lemma:cost_bounds} gives
\begin{align}
    \|\Sigma_{K^\prime}\| \leq \|\Sigma_{K}\| + \frac{C(K)}{\underline{\sigma}(Q)} \leq \frac{2 C(K)}{\underline{\sigma}(Q)} .
\end{align}
By an intermediate step in the proof of Lemma \ref{lemma:C_K_perturbation} we have
\begin{align} \label{eq:int1}
    \|P_{K^\prime}-P_K\| \leq \frac{h_{\text{cost}}(K) C(K)}{\Tr (\Sigma_0)} \|\Delta\| .
\end{align}
Expanding the difference $E_{K^\prime} - E_K$ gives
\begin{align}
    E_{K^\prime} - E_K  & = R\left(K^{\prime}-K\right)+B^\intercal\left(P_{K^{\prime}}-P_{K}\right) A \nonumber \\
                         &\quad +B^\intercal\left(P_{K^{\prime}}-P_{K}\right) B K^{\prime} +\sum_{j=1}^q \beta_j B_j^\intercal\left(P_{K^{\prime}}-P_{K}\right) B_j K^{\prime} \nonumber \\
                         &\quad +B^\intercal P_{K} B\left(K^{\prime}-K\right) +\sum_{j=1}^q \beta_j B_j^\intercal P_{K} B_j\left(K^{\prime}-K\right) .
\end{align}
Using the reverse triangle inequality with the $\Delta$ bound yields 
\begin{align}
    \|K^{\prime}\| - \|K\| \leq \|\Delta\| \leq h_\Delta(K)  \rightarrow \|K^{\prime}\| \leq \|K\| + h_\Delta(K)
\end{align}
so
\begin{align}
    \|E_{K^\prime} - E_K\|  
    & \leq  \|R\| \|K^{\prime}-K\|+\|A\|\|B\|\|P_{K^{\prime}}-P_{K}\| \nonumber \\
    &\quad  + \bigg( \|B\|^2 + \sum_{j=1}^q \beta_j \|B_j\|^2 \bigg) (\|K\| + h_\Delta(K)) \|P_{K^{\prime}}-P_{K}\| \nonumber \\
    &\quad  + \bigg( \|B\|^2 + \sum_{j=1}^q \beta_j \|B_j\|^2 \bigg) \|P_K\| \|K^{\prime}-K\| .
\end{align}
Substituting in \eqref{eq:int1} and $\|P_K\| \leq \frac{C(K)}{{\underline{\sigma} (\Sigma_0)}}$ from Lemma \ref{lemma:cost_bounds}:
\begin{align*}
    \|E_{K^\prime} \! - \! E_K\|  
    \leq  \left[ \|R\| + \left( \|A\|\|B\| + h_B \|B\| (\|K\| \! + \! h_\Delta(K)) \right)  \left(\frac{h_{\text{cost}}(K) C(K)}{\Tr (\Sigma_0)}\right)  +  h_B \|B\| \left(\frac{C(K)}{{\underline{\sigma} (\Sigma_0)}}\right) \right] \|\Delta\| .
\end{align*}
Using $\|\left(E_{K^{\prime}}-E_{K}\right) \Sigma_{K^{\prime}} \| \leq \|E_{K^{\prime}}-E_{K} \| \|\Sigma_{K^{\prime}} \| $ and adding the two terms of \eqref{eq:int2} completes the proof.
\end{proof}

We now discuss smoothing of the cost in the context of model-free gradient descent.
As in \cite{Fazel2018}, in the model-free setting we apply Frobenius-norm ball smoothing to the cost.
Let $\mathbb{S}_r$ be the uniform distribution over all matrices with Frobenius norm $r$ (the boundary of the ball), and $\mathbb{B}_r$ be the uniform distribution over all matrices with Frobenius norm at most $r$ (the entire ball). The smoothed cost is
\begin{align}
    C_{r}(K)=\mathbb{E}_{U \sim \mathbb{B}_r}[C(K+U)]
\end{align}
where $U$ is a random matrix with the same dimensions as $K$ and Frobenius norm $r$.
The following lemma shows that the gradient of the smoothed function can be estimated just with an oracle of the function value.
\begin{Lem}[Zeroth-order gradient estimation] \label{lemma:zeroth_order_optimization} 
The gradient of the smoothed cost is related to the unsmoothed cost by 
    \begin{align}
        \nabla C_{r}(K)=\frac{mn}{r^{2}} \mathbb{E}_{U \sim \mathbb{S}_r}[C(K+U) U] .
    \end{align}
\end{Lem}
\begin{proof}
    The result is proved in Lemma 2.1 in \cite{Flaxman2005}. For completeness, we prove it here.
    By Stokes' Theorem we have
    \begin{align}
        \nabla \int_{\mathbb{B}_r} C(K+U) dU = \int_{\mathbb{S}_r} C(K+U) \left(\frac{U}{\|U\|_{F}}\right) dU=\int_{\mathbb{S}_r} C(K+U) \left(\frac{U}{r}\right) dU .
    \end{align}
    By definition we have
    \begin{align}
        C_{r}(K)=\frac{\int_{\mathbb{B}_r} C(K+U) dU}{\operatorname{vol}_{d}\left(\mathbb{B}_r\right)} .
    \end{align}
    We also have
    \begin{align}
        \mathbb{E}_{U \sim \mathbb{S}_r}[C(K+U) U] &= r \mathbb{E}_{U \sim \mathbb{S}_r}\left[C(K+U) \left(\frac{U}{r}\right)\right] = r \frac{\int_{\delta \mathbb{S}_r} C(K+U) \left(\frac{U}{r}\right) dU}{\operatorname{vol}_{d-1}\left(\mathbb{S}_r\right)}
    \end{align}
    and the ratio of ball surface area to volume 
    \begin{align}
        \frac{\operatorname{vol}_{d-1}\left(\mathbb{S}_r\right)}{\operatorname{vol}_{d}\left(\mathbb{B}_r\right)} = \frac{d}{r} .
    \end{align}
    Combining, we have
    \begin{align}
        \nabla C_{r}(K) &=  \frac{\nabla \int_{\mathbb{B}_r} C(K+U) dU}{\operatorname{vol}_{d}\left(\mathbb{B}_r\right)} = \frac{\int_{\mathbb{S}_r} C(K+U) \left(\frac{U}{r}\right) dU}{\operatorname{vol}_{d-1}\left(\mathbb{S}_r\right)} = \frac{d}{r^2} \mathbb{E}_{U \sim \mathbb{S}_r}[C(K+U) U] .
    \end{align}
\end{proof}
Lemma \ref{lemma:zeroth_order_optimization} shows that the gradient of the smoothed cost can be found exactly with infinitely many infinite-horizon rollouts. Much of the remaining proofs goes towards showing that the error between the gradient of the smoothed cost and the unsmoothed cost, the error due to using finite-horizon rollouts, and the error due to using finitely many rollouts can all be bounded by polynomials of the problem data.
As noted by \cite{Fazel2018} the reason for smoothing in a Frobenius norm ball rather than over a Gaussian distribution is to ensure stability and finiteness of the cost of every gain within the smoothing domain, although now in the multiplicative noise case we must be even more restrictive about our choice of perturbation on $K$ because we require not only mean stability, but also mean-square stability. 

We now give a Bernstein inequality for random matrices and some simple variants; this allows us to bound the difference between the sample average of a random matrix and its expectation.

\begin{Lem}[Matrix Bernstein inequality \cite{Tropp2012}] \label{eq:matrix_bernstein_original}
Let $\{ Z_i \}_{i=1}^N$ be a set of $N$ independent random matrices of dimension $d_1 \times d_2$ with $d_+ = d_1 + d_2$, $\mathbb{E} [Z_i] = 0$, $\| Z_i \| \leq R$ almost surely, and maximum variance 
$\sigma^2 := \max \left( \left\| \sum_{i=1}^N \mathbb{E} (Z_i Z_i^\intercal) \right\|, \left\| \sum_{i=1}^N \mathbb{E} ( Z_i^\intercal Z_i) \right\| \right)$.
Then for all $\epsilon \geq 0$
\begin{align}
    \mathbb{P} \left[ \left\| \sum_{i=1}^N Z_i \right\| \geq \epsilon \right] \leq d_+ \exp \left( -\frac{3}{2} \cdot \frac{\epsilon^2}{3 \sigma^2 + R \epsilon} \right)
\end{align}
\end{Lem}
\begin{proof}
    The lemma and proof follow \cite{Tropp2012} exactly.
\end{proof}

\begin{Lem}[Matrix Bernstein inequality, sample mean variant] \label{eq:matrix_bernstein_sample_mean_variant}
Let $\{ Z_i \}_{i=1}^N$ be a set of $N$ independent random matrices of dimension $d_1 \times d_2$ with $d_+=d_1 + d_2$, $\mathbb{E} [Z_i] = Z$, $\| Z_i - Z\| \leq R$ almost surely, and maximum variance 
$\sigma^2 := \max \left( \left\|  \mathbb{E} (Z_i Z_i^\intercal)  - Z Z^\intercal \right\|, \left\| \mathbb{E} ( Z_i^\intercal Z_i) - Z^\intercal Z\right\| \right)$, and sample average $\widehat{Z} := \frac{1}{N} \sum_{i=1}^N Z_i$.
Then for all $\epsilon \geq 0$
\begin{align}
    \mathbb{P} \left[ \left\| \widehat{Z} - Z \right\| \geq \epsilon \right] \leq d_+ \exp \left( -\frac{3}{2} \cdot \frac{\epsilon^2 N}{3 \sigma^2 + R \epsilon} \right)
\end{align}
\end{Lem}
\begin{proof}
    The lemma follows readily from the matrix Bernstein inequality in Lemma \ref{eq:matrix_bernstein_original} by variable substitutions. Notice that the bound in the RHS of this variant depends on the number of samples $N$.
\end{proof}

\begin{Lem}[Matrix Bernstein inequality, rephrased] \ \\
Let $\{ Z_i \}_{i=1}^N$ be a set of $N$ independent random matrices of dimension ${d_1 \times d_2}$ with ${\mathbb{E} [Z_i] = Z}$, ${\| Z_i - Z\| \leq R_Z}$ almost surely, and maximum variance 
${\max \left( \left\|  \mathbb{E} (Z_i Z_i^\intercal)  - Z Z^\intercal \right\|, \left\| \mathbb{E} ( Z_i^\intercal Z_i) - Z^\intercal Z\right\| \right) \leq \sigma_Z^2}$, and sample average ${\widehat{Z} := \frac{1}{N} \sum_{i=1}^N Z_i}$.
Let a small tolerance ${\epsilon \geq 0}$ and small probability ${ 0 \leq \mu \leq 1 }$ be given.
If
\begin{align*}
    N \geq \frac{2\min(d_1, d_2)}{\epsilon^2} \left( \sigma_Z^2 + \frac{R_Z \epsilon}{3\sqrt{\min(d_1, d_2)}}\right) \log \left[ \frac{d_1+d_2}{\mu} \right]
\end{align*}
then
\begin{align}
    \mathbb{P} \left[ \left\| \widehat{Z} - Z \right\|_F \leq \epsilon \right] \geq 1-\mu .
\end{align}
\end{Lem}
\begin{proof}
    The lemma follows readily from the matrix Bernstein inequality in Lemma \ref{eq:matrix_bernstein_sample_mean_variant} by rearrangement and the bound $\| M \|_F \leq \sqrt{\min(d_1, d_2)} \| M \| $.
\end{proof}

\begin{Lem}[Estimating $\nabla C(K)$ with finitely many infinite-horizon rollouts] \label{lemma:est_grad_inf}
Given an arbitrary tolerance $\epsilon$ and probability $\mu$ suppose the exploration radius $r$ is chosen as 
\begin{align}
    r \leq h_r \left(\frac{\epsilon}{2} \right) \Let \min \left\{h_\Delta(K) , \frac{1}{h_{\text{cost}}(K)},\frac{\epsilon}{2h_{\text{grad}}(K)} \right\}
\end{align}
and the number of samples $n_{\text{sample}}$ of $U_{i} \sim \mathbb{S}_{r}$ is chosen as 
\begin{align*}
    n_{\text{sample}} \geq h_{\text{sample}} \left( \frac{\epsilon}{2}, \mu \right) 
    & \Let \frac{8 \min(m,n)}{\epsilon^2} \bigg( \! \sigma^2_{\widehat{\nabla}} + \frac{R_{\widehat{\nabla}} \epsilon}{6 \sqrt{\min(m, n)}} \bigg) \log \! \left[ \frac{m+n}{\mu} \right] \! , \\
    R_{\widehat{\nabla}} &:= \frac{2mn C(K)}{r} + \frac{\epsilon}{2} + h_1(K) , \\
    \sigma^2_{\widehat{\nabla}} &:= \left(\frac{2 m n C(K)}{r}\right)^2  + \left(\frac{\epsilon}{2} + h_1(K) \right)^2 .
\end{align*}
Then with high probability of at least $1-\mu$ the estimate
\begin{align}
    \hat{\nabla} C(K)=\frac{1}{n_{\text{sample}}} \sum_{i=1}^{n_{\text{sample}}} \frac{mn}{r^{2}} C\left(K+U_{i}\right) U_{i}
\end{align}
satisfies the error bound
\begin{align}
    \|\hat{\nabla} C(K) - \nabla C(K) \|_F \leq \epsilon .
\end{align}
\end{Lem}

\begin{proof}
First note that $\|K^\prime - K\|_F=\|\Delta\|_F=\|U\|_F=r$.
We break the difference between estimated and true gradient $\hat{\nabla} C(K) - \nabla C(K)$ into two terms as
\begin{align} \label{eq:int3}
    \big(\nabla C_{r}(K)-\nabla C(K)\big)+\big(\hat{\nabla} C(K) -\nabla C_{r}(K)\big) .
\end{align}
Since $r \leq h_\Delta(K)$ we see that Lemmas \ref{lemma:C_K_perturbation} and \ref{lemma:nabla_K_C_K_perturbation} hold.
By enforcing the bound $r \leq \frac{1}{h_{\text{cost}}(K)}$, by Lemma \ref{lemma:C_K_perturbation} and noting that $\|\Delta\| \leq \|\Delta\|_F$ we have
\begin{align} 
    |C(K \! +\! U)-C(K)| \leq C(K)  \rightarrow  C(K \! +\! U) \leq 2C(K) . \phantom{\hspace{5mm}} \label{eq:cost_bound_2}
\end{align}
This ensures stability of the system under the perturbed gains so that $C(K+U)$ is well-defined.
For the first term $\nabla C_{r}(K)-\nabla C(K)$, by enforcing $r \leq \frac{\epsilon}{2h_{\text{grad}}(K)}$, by Lemma \ref{lemma:nabla_K_C_K_perturbation} we have 
\begin{align} 
    \|\nabla C({K+U})-\nabla C({K})\|_F \leq \frac{\epsilon}{2} .
\end{align}   
Since $\nabla C_{r}(K)$ is the expectation of $\nabla C({K+U})$, by the triangle inequality we have
\begin{align} \label{eq:est_grad_inf_first_term}
    \|\nabla C_{r}(K)-\nabla C({K})\|_F \leq \frac{\epsilon}{2} .
\end{align}  
For the second term $\hat{\nabla} C(K) - \nabla C_{r}(K)$, we work towards using the matrix Bernstein inequality and adopt the notation of the associated lemma.
First note that by Lemma \ref{lemma:zeroth_order_optimization} we have $Z := \nabla C_{r}(K) = \mathbb{E}[\hat{\nabla} C(K)]$.
Each individual sample $Z_i := \left(\frac{mn}{r^2}\right) C(K+U_i) U_i$ has the bounded Frobenius norm
\begin{align*}
    \| Z_i \|_F = \left\| \left(\frac{mn}{r^2}\right) C(K+U_i) U_i \right\|_F = \frac{mn C(K+U_i) \|U_i\|_F}{r^2} 
    = \frac{mn C(K+U_i) r}{r^2} = \frac{mn C(K+U_i)}{r} \leq \frac{2mn C(K)}{r} .
\end{align*}
Next, from \eqref{eq:est_grad_inf_first_term} and Lemma \ref{lemma:gradient_gain_bounds} we have
\begin{align}
    \| Z \|_F = \|\nabla C_{r}(K)\|_F 
    & \leq \frac{\epsilon}{2} + \|\nabla C({K})\|_F \\
    & \leq \frac{\epsilon}{2} + h_1(K)
\end{align}
so by the triangle inequality each sample difference has the bounded Frobenius norm
\begin{align*}
    \|Z_i - Z\|_F
    & \leq \| Z_i \|_F + \|Z\|_F \\
    & \leq \frac{2mn C(K)}{r} + \frac{\epsilon}{2} + h_1(K) .
\end{align*}
Using \eqref{eq:cost_bound_2} and $\|U_i\|_F \leq r$, the variance of the differences is likewise bounded as
\begin{align}
    \left\|  \mathbb{E} (Z_i Z_i^\intercal)  - Z Z^\intercal \right\| 
    & \leq \left\|  \mathbb{E} (Z_i Z_i^\intercal) \right\|_F + \| Z Z^\intercal \|_F \\
    & \leq \max_{Z_i} \left( \|  Z_i \|_F \right)^2 + \| Z \|_F^2 \\
    & \leq \left(\frac{2 mn C(K)}{r}\right)^2 + \left(\frac{\epsilon}{2} + h_1(K) \right)^2
\end{align}
An identical argument holds for $\left\|  \mathbb{E} (Z_i^\intercal Z_i)  - Z^\intercal Z \right\| $ so the assumed choice of $\sigma^2_{\widehat{\nabla}}$ is valid.
Thus, using the assumed number of samples $n_{\text{sample}} \geq h_{\text{sample}}$ satisfies the condition of the matrix Bernstein inequality, and thus with high probability of at least $1-\mu$ we have
\begin{align*}
    \|\hat{\nabla} C(K) - \mathbb{E}[\hat{\nabla} C(K)]\|_F
    = \| \hat{\nabla} C(K) -\nabla C_{r}(K)\|_F \leq \frac{\epsilon}{2} .
\end{align*}
Adding the bounds on the two terms in \eqref{eq:int3} and using the triangle inequality completes the proof.
\end{proof}

\begin{Lem}[Estimating $\nabla_K C(K)$ with finitely many finite-horizon rollouts] \label{lemma:est_grad_fin}
Given an arbitrary tolerance $\epsilon$ and probability $\mu$, suppose the exploration radius $r$ is chosen as 
\begin{align}
    r \leq h_{r}\left(\frac{\epsilon}{4}\right) = \min \left\{h_\Delta(K), \frac{1}{h_{\text{cost}}(K)},\frac{\epsilon}{4h_{\text{grad}}(K)} \right\}
\end{align}
and the rollout length $\ell$ is chosen as
\begin{align}
    \ell \geq h_{\ell}\left(\frac{ r \epsilon }{ 4 m n }\right) = \frac{ 4 m n^2 C^2(K) \left(\|Q\|+\|R\|\|K\|^{2}\right)}{r \epsilon {\underline{\sigma} (\Sigma_0)} \underline{\sigma}^{2}(Q)} .
\end{align}
Suppose that the distribution of the initial states is such that $x_0 \sim \mathcal{P}_0$ implies $\|x_0^i\| \leq L_0$ almost surely for any given realization $x_0^i$ of $x_0$. 
Suppose additionally that the multiplicative noises are distributed such that the following bound is satisfied almost surely under the closed-loop dynamics with any gain $K+U_i$ where $\|U_i\| \leq r$ for any given realized sequence $x_t^i$ of $x_t$ with a positive scalar $z \geq 1$
\begin{align*} 
    \sum_{t=0}^{\ell-1} \Big( {x_{t}^{i}}^\intercal Q x_{t}^{i}+{u_{t}^{i}}^\intercal R u_{t}^{i} \Big)  \leq z \underset{\delta, \gamma}{\mathbb{E}} \left[ \sum_{t=0}^{\ell-1} \Big( x_t^\intercal Q x_t + u_t^\intercal R u_t \Big) \right] .
\end{align*}
Suppose the number $n_{\text{sample}}$ of $U_{i} \sim \mathbb{S}_{r}$ is chosen as
\begin{align*}
    n_{\text{sample}} \geq h_{\text{sample,trunc}} \left( \frac{\epsilon}{4}, \mu, \frac{L_0^2}{\underline{\sigma} (\Sigma_0)}, z \right)
    & \Let \frac{32 \min(m,n)}{\epsilon^2} \bigg( \sigma_\nabla^2 + \frac{R_{\widehat{\nabla}} \epsilon}{12 \sqrt{\min(m,n)}}\bigg) \log \left[ \frac{m+n}{\mu} \right] \\
    \text{where} \quad
    R_{\widetilde{\nabla}} & := \frac{2 m n z L_0^2 C(K)}{r \underline{\sigma}(\Sigma_0)} + \frac{\epsilon}{2} + h_1(K)  \\
    \sigma^2_{\widetilde{\nabla}} & := \left(\frac{2 m n z L_0^2 C(K)}{r \underline{\sigma}(\Sigma_0)} \right)^2  + \left(\frac{\epsilon}{2} + h_1(K) \right)^2
\end{align*}
The finite-horizon estimate of the cost is defined as
\begin{align}
    \hat{C}\left(K+U_{i}\right) \Let \sum_{t=0}^{\ell-1} \Big( {x_{t}^{i}}^\intercal Q x_{t}^{i}+{u_{t}^{i}}^\intercal R u_{t}^{i} \Big)
\end{align}
under the closed loop dynamics with gain $K+U_i$.
Then with high probability of at least $1-\mu$ the estimated gradient
\begin{align}
    \widetilde{\nabla} C(K) \Let \frac{1}{n_{\text{sample}}} \sum_{i=1}^{n_{\text{sample}}} \frac{mn}{r^{2}} \hat{C}\left(K+U_{i}\right) U_{i}
\end{align}
satisfies the error bound
\begin{align}
    \|\widetilde{\nabla} C(K) - \nabla C(K) \|_F \leq \epsilon .
\end{align}
\end{Lem}

\begin{proof}
Similar to before, we break the difference between estimated and true gradient into three terms as
\begin{align*} 
     \widetilde{\nabla} C(K) - \nabla C(K)  
     &= 
     ( \widetilde{\nabla} - \nabla^{\prime} )
     + ( \nabla^{\prime} - \hat{\nabla} ) 
     + (\hat{\nabla} - \nabla ) \\
    \text{where} \quad
    \nabla^{\prime} C(K) &= \frac{1}{n_{\text{sample}}} \sum_{i=1}^{n_{\text{sample}}} \frac{mn}{r^{2}} C^{(\ell)}\left(K+U_{i}\right) U_{i}
\end{align*}
and $\hat{\nabla} C(K)$ is defined as in Lemma \ref{lemma:est_grad_inf}.
The third term is handled by Lemma \ref{lemma:est_grad_inf}. 
Notice that since $\|x_0^i\| \leq L_0$ we have
$\underline{\sigma}(\Sigma_0) \leq \overline{\sigma}(\Sigma_0) \leq L_0^2$ so $\frac{L_0^2}{\underline{\sigma}(\Sigma_0)} \geq 1$. 
Similarly $z \geq 1$, and thus 
\begin{align}
    h_{\text{sample,trunc}} \left( \frac{\epsilon}{4}, \mu, \frac{L_0^2}{\underline{\sigma} (\Sigma_0)}, z \right) 
    \geq 
    h_{\text{sample}} \left( \frac{\epsilon}{4}, \mu \right) .
\end{align}
Thus the choice of $r$ and $n_{\text{sample}}$ satisfy the conditions of Lemma \ref{lemma:est_grad_inf}, so with high probability of at least $1-\mu$ 
\begin{align}
    \|\hat{\nabla} C(K)-\nabla C(K)\|_F \leq \frac{\epsilon}{4} \label{eq:est_grad_fin_first_term}.
\end{align}
For the second term, by using the choices $\ell \geq h_{\ell}\left(\frac{ r \epsilon }{ 4 m n }\right)$ and $C(K+U_i) \leq 2C(K)$, Lemma \ref{lemma:finite_horizon} holds and implies that
\begin{align}
    \left\|C^{(\ell)}\left(K+U_i \right)-C\left(K+U_i \right)\right\|_F \leq \frac{r \epsilon}{4 mn} .
\end{align}
By the triangle inequality, submultiplicativity, and $\| U_i \|_F \leq r$
\begin{align} 
    \left\|\frac{1}{n_{\text{sample}}} \sum_{i=1}^{n_{\text{sample}}} \frac{mn}{r^{2}} \left[ C^{(\ell)}\left(K+U_{i}\right) - C\left(K+U_{i}\right) \right]  U_{i} \right\|_F 
    = \| \nabla^{\prime}_K C(K)-\hat{\nabla} C(K) \|_F \leq \frac{\epsilon}{4} . \label{eq:est_grad_fin_second_term} 
\end{align}
For the first term, $\|x_0^i\| \leq L_0$ implies
\begin{align}
    \frac{L_0^2}{\underline{\sigma}(\Sigma_0)}\Sigma_0  \succeq x_0^i {x_0^i}^\intercal .
\end{align}
Applying this to the cost, summing over time and using the assumed restriction on the multiplicative noise we have
\begin{align}
    \frac{2 z L_0^2 C(K)}{\underline{\sigma}(\Sigma_0)} 
    &\geq \frac{z L_0^2}{\underline{\sigma}(\Sigma_0)} C(K+U_i) \\
    &\geq z \underset{\delta, \gamma}{\mathbb{E}} \left[ \sum_{t=0}^{\infty} \Big( {x_{t}^{i}}^\intercal Q x_{t}^{i}+{u_{t}^{i}}^\intercal R u_{t}^{i} \Big) \right] \\
    &\geq \sum_{t=0}^{\ell-1} \Big( {x_{t}^{i}}^\intercal Q x_{t}^{i}+{u_{t}^{i}}^\intercal R u_{t}^{i} \Big) .
\end{align}
Using this and an argument identical to Lemma \ref{lemma:est_grad_inf}, each sample $Z_i := \left(\frac{mn}{r^2}\right) \hat{C}(K+U_i) U_i$ has bounded Frobenius norm
\begin{align}
    & \| Z_i \|_F = \left\| \left(\frac{mn}{r^2}\right) \hat{C}(K+U_i) U_i \right\|_F \leq
    \frac{2 m n z L_0^2 C(K)}{r \underline{\sigma}(\Sigma_0)} .
\end{align}
By \eqref{eq:est_grad_fin_first_term} and \eqref{eq:est_grad_fin_second_term} we have for $Z := \mathbb{E} [\widetilde{\nabla} C(K)] = \nabla^{\prime}_K C(K)$
\begin{align}
    \| Z \|_F  = \| \nabla^{\prime}_K C(K) \|_F 
    &\leq \frac{\epsilon}{4} + \| \hat{\nabla} C(K) \|_F \\
    &\leq \frac{\epsilon}{4} + \frac{\epsilon}{4} + \| \nabla_K C(K) \|_F \\
    &\leq \frac{\epsilon}{2} + h_1(K)
\end{align}
Using arguments identical to Lemma \ref{lemma:est_grad_inf} we obtain the bounds on the sample difference $R_{\widetilde{\nabla}} = \| Z_i - Z\|_F$ and variance $\sigma_{\widetilde{\nabla}}^2$ given in the assumption.
Thus the polynomial $h_{\text{sample,trunc}}$ is large enough so the matrix Bernstein inequality implies
\begin{align}
    \|\widetilde{\nabla} C(K) - \nabla^{\prime}_K C(K)\|_F \leq \frac{\epsilon}{4}
\end{align}
with high probability $1-\mu$. Adding the three terms together and using the triangle inequality completes the proof.
\end{proof}

We now give the parameters and proof of high-probability global convergence in Theorem \ref{thm:model-free}.
\begin{Thm}[Model-free policy gradient convergence] \label{thm:model-free_restated}
    Consider the assumptions and notations of Theorem \ref{thm:model-free} where the number of samples $n_{\text{sample}}$, rollout length $\ell$, and exploration radius $r$ are chosen according to the fixed quantities
    \begin{align*}
        r \geq h_{r,\text{GD}} & \Let  h_{r} \left(\frac{\epsilon^\prime}{4} \right) , \\
        \ell \geq h_{\ell,\text{GD}} & \Let h_{\ell}\left(\frac{ r \epsilon^\prime }{ 4 m n }\right) , \\
        n_{\text{sample}} \geq h_{\text{sample,GD}} & \Let h_{\text{sample,trunc}} \left( \frac{\epsilon^\prime}{4}, \frac{\mu}{N}, \frac{L_0^2}{\underline{\sigma} (\Sigma_0)}, z \right) , \\
        \text{where } \qquad \epsilon^\prime     & \Let \min \left\{ \frac{\underline{\sigma}(\Sigma_0)^{2} \underline{\sigma}(R) }{ \left\|\Sigma_{K^{*}} \right\| C(K_0) \overline{h_{\text{cost}}} }\cdot \epsilon, \frac{ \overline{h_\Delta} }{\eta} \right\} , \\
        \overline{h_{\text{cost}}} & \Let \max_K \ h_{\text{cost}}(K) \text{ subject to } C(K) \leq 2 C(K_0) \\
        \overline{h_\Delta} & \Let \min_K \ h_\Delta(K) \text{ subject to } C(K) \leq 2 C(K_0)
    \end{align*}
    Then the claim of Theorem \ref{thm:model-free} holds.
\end{Thm}

\begin{proof}
The proof follows \cite{Fazel2018} using the polynomials defined in our theorem. The last part of the proof is the same as in Theorem \ref{thm:grad_exact_convergence}. 
As noted by \cite{Fazel2018}, the monotonic decrease in the function value during gradient descent and the choice of exploration radius $r$ are sufficient to ensure that all cost values encountered throughout the entire algorithm are bounded by $2 C(K_0)$, ensuring that all polynomial quantities used are bounded as well. 
We also require $\epsilon^\prime \leq \frac{ \overline{h_\Delta} }{\eta}$ in order for $ \| \Delta \| = \eta \| \widetilde{\nabla} C(K) - \nabla_K C(K)) \|$ to satisfy the condition of Lemma \ref{lemma:C_K_perturbation}.
\end{proof}

\begin{Rem}
As in Remark \ref{rem:param_bounds}, the quantities $\overline{h_{\text{cost}}}$ and $\overline{h_\Delta}$ may be upper (lower) bounded by quantities that depend on problem data and $C(K_0)$, so a conservative minimum exploration radius $r$, number of rollouts $n_{\text{sample}}$, and rollout length $\ell$ can be computed exactly in terms of problem data.
Looking back across the terms that feed into the step size, number of rollouts, rollout length, and exploration radius, we see $C(K)$, $\|\Sigma_K\|$, $\|P_K\|$, and $\|B\|^2 + \sum_{j=1}^q \beta_j \|B_j\|^2$ are necessarily greater with state- and/or input-dependent multiplicative noise, and thus the algorithmic parameters are worsened by the noise.
\end{Rem}

\bibliographystyle{plain}
\bibliography{main.bib}

\begin{thebibliography}{10}

\bibitem{abbasi2011regret}
Yasin Abbasi-Yadkori and Csaba Szepesv{\'a}ri.
\newblock Regret bounds for the adaptive control of linear quadratic systems.
\newblock In {\em Proceedings of the 24th Annual Conference on Learning
  Theory}, pages 1--26, 2011.

\bibitem{abeille2018improved}
Marc Abeille and Alessandro Lazaric.
\newblock Improved regret bounds for thompson sampling in linear quadratic
  control problems.
\newblock In {\em International Conference on Machine Learning}, pages 1--9,
  2018.

\bibitem{antsaklis2007special}
Panos Antsaklis and John Baillieul.
\newblock Special issue on technology of networked control systems.
\newblock {\em Proceedings of the IEEE}, 95(1):5--8, 2007.

\bibitem{athans1977uncertainty}
Michael Athans, Richard Ku, and Stanley Gershwin.
\newblock The uncertainty threshold principle: Some fundamental limitations of
  optimal decision making under dynamic uncertainty.
\newblock {\em IEEE Transactions on Automatic Control}, 22(3):491--495, 1977.

\bibitem{baggio2019}
G.~{Baggio}, V.~{Katewa}, and F.~{Pasqualetti}.
\newblock Data-driven minimum-energy controls for linear systems.
\newblock {\em IEEE Control Systems Letters}, 3(3):589--594, 2019.

\bibitem{bamieh2018input}
Bassam Bamieh and Maurice Filo.
\newblock An input-output approach to structured stochastic uncertainty.
\newblock {\em arXiv preprint arXiv:1806.07473}, 2018.

\bibitem{benner2011lyapunov}
Peter Benner and Tobias Damm.
\newblock Lyapunov equations, energy functionals, and model order reduction of
  bilinear and stochastic systems.
\newblock {\em SIAM journal on control and optimization}, 49(2):686--711, 2011.

\bibitem{bernstein1987robust}
Dennis Bernstein.
\newblock Robust static and dynamic output-feedback stabilization:
  Deterministic and stochastic perspectives.
\newblock {\em IEEE Transactions on Automatic Control}, 32(12):1076--1084,
  1987.

\bibitem{bertsekas1995dynamic}
Dimitri~P Bertsekas.
\newblock {\em Dynamic programming and optimal control}, volume~1.
\newblock Athena scientific Belmont, MA, 1995.

\bibitem{boyd1994linear}
S.~Boyd, L.~El~Ghaoui, E.~Feron, and V.~Balakrishnan.
\newblock {\em Linear Matrix Inequalities in System and Control Theory}.
\newblock Society for Industrial and Applied Mathematics, 1994.

\bibitem{boyd2004convex}
Stephen Boyd, Stephen~P Boyd, and Lieven Vandenberghe.
\newblock {\em Convex optimization}.
\newblock Cambridge university press, 2004.

\bibitem{Bradtke1994}
Steven~J Bradtke, B~Erik Ydstie, and Andrew~G Barto.
\newblock Adaptive linear quadratic control using policy iteration.
\newblock In {\em Proceedings of 1994 American Control Conference}, volume~3,
  pages 3475--3479. IEEE, 1994.

\bibitem{breakspear2017dynamic}
Michael Breakspear.
\newblock Dynamic models of large-scale brain activity.
\newblock {\em Nature neuroscience}, 20(3):340, 2017.

\bibitem{bu2019lqr}
Jingjing Bu, Afshin Mesbahi, Maryam Fazel, and Mehran Mesbahi.
\newblock {LQR} through the lens of first order methods: Discrete-time case.
\newblock {\em arXiv preprint arXiv:1907.08921}, 2019.

\bibitem{carrasco2006power}
Juan~Manuel Carrasco, Leopoldo Garc{\'\i}a~Franquelo, Jan~T Bialasiewicz,
  Eduardo Galv{\'a}n, Ram{\'o}n~Carlos Portillo~Guisado, Mar{\'\i}a de
  los~{\'A}ngeles Mart{\'\i}n~Prats, Jos{\'e}~Ignacio Le{\'o}n, and Narciso
  Moreno-Alfonso.
\newblock Power-electronic systems for the grid integration of renewable energy
  sources: A survey.
\newblock {\em IEEE Transactions on Industrial Electronics, 53 (4),
  1002-1016.}, 2006.

\bibitem{Damm2004}
Tobias Damm.
\newblock {\em Rational matrix equations in stochastic control}, volume 297.
\newblock Springer Science \& Business Media, 2004.

\bibitem{dean2017sample}
Sarah Dean, Horia Mania, Nikolai Matni, Benjamin Recht, and Stephen Tu.
\newblock On the sample complexity of the linear quadratic regulator.
\newblock {\em arXiv preprint arXiv:1710.01688}, 2017.

\bibitem{el1995state}
Laurent El~Ghaoui.
\newblock State-feedback control of systems with multiplicative noise via
  linear matrix inequalities.
\newblock {\em Systems \& Control Letters}, 24(3):223--228, 1995.

\bibitem{Fazel2018}
Maryam Fazel, Rong Ge, Sham Kakade, and Mehran Mesbahi.
\newblock Global convergence of policy gradient methods for the linear
  quadratic regulator.
\newblock In {\em Proceedings of the 35th International Conference on Machine
  Learning}, volume~80 of {\em Proceedings of Machine Learning Research}, pages
  1467--1476. PMLR, 10--15 Jul 2018.

\bibitem{fiechter1997pac}
Claude-Nicolas Fiechter.
\newblock {PAC} adaptive control of linear systems.
\newblock In {\em Proceedings of the Tenth Annual Conference on Computational
  Learning Theory}, COLT '97, pages 72--80, New York, NY, USA, 1997. ACM.

\bibitem{Flaxman2005}
Abraham~D. Flaxman, Adam~Tauman Kalai, Adam~Tauman Kalai, and H.~Brendan
  McMahan.
\newblock Online convex optimization in the bandit setting: Gradient descent
  without a gradient.
\newblock In {\em Proceedings of the Sixteenth Annual ACM-SIAM Symposium on
  Discrete Algorithms}, pages 385--394, Philadelphia, PA, USA, 2005. Society
  for Industrial and Applied Mathematics.

\bibitem{freiling2003}
G.~Freiling and A.~Hochhaus.
\newblock Properties of the solutions of rational matrix difference equations.
\newblock {\em Computers \& Mathematics with Applications}, 45(6):1137 -- 1154,
  2003.
\newblock Advances in Difference Equations IV.

\bibitem{goncalves2019}
G.~R. {Gonçalves da Silva}, A.~S. {Bazanella}, C.~{Lorenzini}, and
  L.~{Campestrini}.
\newblock Data-driven lqr control design.
\newblock {\em IEEE Control Systems Letters}, 3(1):180--185, 2019.

\bibitem{gravell2020ifac}
Benjamin Gravell, Peyman~Mohajerin Esfahani, and Tyler Summers.
\newblock Robust control design for linear systems via multiplicative noise.
\newblock {\em IFAC Proceedings Volumes}, 2020.
\newblock 18th IFAC World Congress (to appear), arXiv preprint
  arXiv:2004.08019.

\bibitem{hespanha2007survey}
Joao~P Hespanha, Payam Naghshtabrizi, and Yonggang Xu.
\newblock A survey of recent results in networked control systems.
\newblock {\em Proceedings of the IEEE}, 95(1):138--162, 2007.

\bibitem{Hewer1971}
G.~{Hewer}.
\newblock An iterative technique for the computation of the steady state gains
  for the discrete optimal regulator.
\newblock {\em IEEE Transactions on Automatic Control}, 16(4):382--384, August
  1971.

\bibitem{hinrichsen1998stochastic}
Diederich Hinrichsen and Anthony~J Pritchard.
\newblock Stochastic ${H}^\infty$.
\newblock {\em SIAM Journal on Control and Optimization}, 36(5):1504--1538,
  1998.

\bibitem{ivanov2007}
Ivan~Ganchev Ivanov.
\newblock Properties of {Stein} ({Lyapunov}) iterations for solving a general
  {Riccati} equation.
\newblock {\em Nonlinear Analysis: Theory, Methods \& Applications}, 67(4):1155
  -- 1166, 2007.

\bibitem{jansch2020convergence}
Joao~Paulo Jansch-Porto, Bin Hu, and Geir Dullerud.
\newblock Convergence guarantees of policy optimization methods for markovian
  jump linear systems.
\newblock In {\em IEEE 2020 American Control Conference (to appear), arXiv
  preprint arXiv:2002.04090}, 2020.

\bibitem{juang1985}
Jer-Nan Juang and Richard~S. Pappa.
\newblock An eigensystem realization algorithm for modal parameter
  identification and model reduction.
\newblock {\em Journal of Guidance, Control, and Dynamics}, 8(5):620--627,
  1985.

\bibitem{Kakade2002}
Sham~M Kakade.
\newblock A natural policy gradient.
\newblock In {\em Advances in neural information processing systems}, pages
  1531--1538, 2002.

\bibitem{Karimi2016}
Hamed Karimi, Julie Nutini, and Mark Schmidt.
\newblock Linear convergence of gradient and proximal-gradient methods under
  the polyak-{\l}ojasiewicz condition.
\newblock In {\em Machine Learning and Knowledge Discovery in Databases}, pages
  795--811, Cham, 2016. Springer International Publishing.

\bibitem{kozin1969survey}
Frank Kozin.
\newblock A survey of stability of stochastic systems.
\newblock {\em Automatica}, 5(1):95--112, 1969.

\bibitem{ku1977further}
Richard Ku and Michael Athans.
\newblock Further results on the uncertainty threshold principle.
\newblock {\em IEEE Transactions on Automatic Control}, 22(5):866--868, 1977.

\bibitem{lewis2012reinforcement}
Frank~L Lewis, Draguna Vrabie, and Kyriakos~G Vamvoudakis.
\newblock Reinforcement learning and feedback control: Using natural decision
  methods to design optimal adaptive controllers.
\newblock {\em IEEE Control Systems Magazine}, 32(6):76--105, 2012.

\bibitem{li2005estimation}
Weiwei Li, Emanuel Todorov, and Robert~E Skelton.
\newblock Estimation and control of systems with multiplicative noise via
  linear matrix inequalities.
\newblock In {\em Proceedings of the 2005, American Control Conference, 2005.},
  pages 1811--1816. IEEE, 2005.

\bibitem{lumley2007stochastic}
John~L Lumley.
\newblock {\em Stochastic tools in turbulence}.
\newblock Courier Corporation, 2007.

\bibitem{mania2019certainty}
Horia Mania, Stephen Tu, and Benjamin Recht.
\newblock Certainty equivalent control of {LQR} is efficient.
\newblock {\em arXiv preprint arXiv:1902.07826}, 2019.

\bibitem{maupong2017}
T.M. Maupong and P.~Rapisarda.
\newblock Data-driven control: A behavioral approach.
\newblock {\em Systems \& Control Letters}, 101:37 -- 43, 2017.
\newblock Jan C. Willems Memorial Issue, Volume 2.

\bibitem{milano2018foundations}
Federico Milano, Florian D{\"o}rfler, Gabriela Hug, David~J Hill, and Gregor
  Verbi{\v{c}}.
\newblock Foundations and challenges of low-inertia systems.
\newblock In {\em 2018 Power Systems Computation Conference (PSCC)}, pages
  1--25. IEEE, 2018.

\bibitem{mnih2015human}
Volodymyr Mnih, Koray Kavukcuoglu, David Silver, Andrei~A. Rusu, Joel Veness,
  Marc~G. Bellemare, Alex Graves, Martin Riedmiller, Andreas~K. Fidjeland,
  Georg Ostrovski, Stig Petersen, Charles Beattie, Amir Sadik, Ioannis
  Antonoglou, Helen King, Dharshan Kumaran, Daan Wierstra, Shane Legg, and
  Demis Hassabis.
\newblock Human-level control through deep reinforcement learning.
\newblock {\em Nature}, 518(7540):529, 2015.

\bibitem{persis2019}
C.~D. {Persis} and P.~{Tesi}.
\newblock On persistency of excitation and formulas for data-driven control.
\newblock In {\em 2019 IEEE 58th Conference on Decision and Control (CDC)},
  pages 873--878, 2019.

\bibitem{peters2006}
J.~{Peters} and S.~{Schaal}.
\newblock Policy gradient methods for robotics.
\newblock In {\em 2006 IEEE/RSJ International Conference on Intelligent Robots
  and Systems}, pages 2219--2225, 2006.

\bibitem{Polyak1963}
B.T. Polyak.
\newblock Gradient methods for the minimisation of functionals.
\newblock {\em USSR Computational Mathematics and Mathematical Physics},
  3(4):864 -- 878, 1963.

\bibitem{recht2018tour}
Benjamin Recht.
\newblock A tour of reinforcement learning: The view from continuous control.
\newblock {\em Annual Review of Control, Robotics, and Autonomous Systems},
  2018.

\bibitem{schuurmans2019}
M.~{Schuurmans}, P.~{Sopasakis}, and P.~{Patrinos}.
\newblock Safe learning-based control of stochastic jump linear systems: a
  distributionally robust approach.
\newblock In {\em 2019 IEEE 58th Conference on Decision and Control (CDC)},
  pages 6498--6503, 2019.

\bibitem{silver2016mastering}
David Silver, Aja Huang, Chris~J. Maddison, Arthur Guez, Laurent Sifre, George
  van~den Driessche, Julian Schrittwieser, Ioannis Antonoglou, Veda
  Panneershelvam, Marc Lanctot, Sander Dieleman, Dominik Grewe, John Nham, Nal
  Kalchbrenner, Ilya Sutskever, Timothy Lillicrap, Madeleine Leach, Koray
  Kavukcuoglu, Thore Graepel, and Demis Hassabis.
\newblock Mastering the game of {Go} with deep neural networks and tree search.
\newblock {\em Nature}, 529(7587):484, 2016.

\bibitem{silver2018general}
David Silver, Thomas Hubert, Julian Schrittwieser, Ioannis Antonoglou, Matthew
  Lai, Arthur Guez, Marc Lanctot, Laurent Sifre, Dharshan Kumaran, Thore
  Graepel, Timothy Lillicrap, Karen Simonyan, and Demis Hassabis.
\newblock A general reinforcement learning algorithm that masters chess, shogi,
  and {Go} through self-play.
\newblock {\em Science}, 362(6419):1140--1144, 2018.

\bibitem{tobin2017domain}
Josh Tobin, Rachel Fong, Alex Ray, Jonas Schneider, Wojciech Zaremba, and
  Pieter Abbeel.
\newblock Domain randomization for transferring deep neural networks from
  simulation to the real world.
\newblock In {\em IEEE/RSJ International Conference on Intelligent Robots and
  Systems}, pages 23--30. IEEE, 2017.

\bibitem{Tropp2012}
Joel~A. Tropp.
\newblock User-friendly tail bounds for sums of random matrices.
\newblock {\em Foundations of Computational Mathematics}, 12(4):389--434, Aug
  2012.

\bibitem{tu2017least}
Stephen Tu and Benjamin Recht.
\newblock Least-squares temporal difference learning for the linear quadratic
  regulator.
\newblock {\em arXiv preprint arXiv:1712.08642}, 2017.

\bibitem{Tyrtyshnikov2012}
Eugene~E Tyrtyshnikov.
\newblock {\em A brief introduction to numerical analysis}.
\newblock Springer Science \& Business Media, 2012.

\bibitem{umenberger2018learning}
Jack Umenberger and Thomas~B Sch{\"o}n.
\newblock Learning convex bounds for linear quadratic control policy synthesis.
\newblock In {\em Advances in Neural Information Processing Systems}, pages
  9561--9572, 2018.

\bibitem{venkataraman2018recovering}
Harish~K Venkataraman and Peter~J Seiler.
\newblock Recovering robustness in model-free reinforcement learning.
\newblock {\em arXiv preprint arXiv:1810.09337}, 2018.

\bibitem{willems1976feedback}
Jacques~L Willems and Jan~C Willems.
\newblock Feedback stabilizability for stochastic systems with state and
  control dependent noise.
\newblock {\em Automatica}, 12(3):277--283, 1976.

\bibitem{Wonham1967}
W~Murray Wonham.
\newblock Optimal stationary control of a linear system with state-dependent
  noise.
\newblock {\em SIAM Journal on Control}, 5(3):486--500, 1967.

\bibitem{gravell2020acc}
Yu~Xing, Ben Gravell, Xingkang He, Karl~Henrik Johansson, and Tyler Summers.
\newblock Linear system identification under multiplicative noise from multiple
  trajectory data.
\newblock In {\em IEEE 2020 American Control Conference (to appear), arXiv
  preprint arXiv:2002.06613}, 2020.

\end{thebibliography}

\end{document}